\documentclass[11pt]{article}
\usepackage[T1]{fontenc}
\usepackage{lmodern}
\usepackage{amsmath, amsthm, amssymb}
\usepackage[hidelinks]{hyperref}
\usepackage[margin=1in]{geometry}
\usepackage{mathrsfs}
\usepackage{enumitem}
\usepackage{yfonts}
\usepackage{dsfont}
\usepackage{graphicx} 
\usepackage{microtype}
\usepackage{authblk}
\usepackage{placeins}
\usepackage{tikz}
\usepackage{subfigure}
\usepackage{tikz-3dplot}

\usetikzlibrary{plotmarks}
\DeclareMathOperator{\spn}{span}
\usetikzlibrary{shapes.misc}

\tikzset{cross/.style={cross out, draw=black, minimum size=2*(#1-\pgflinewidth), inner sep=0pt, outer sep=0pt, line width=0.5mm},
cross/.default={4pt}}

\usepackage[ruled,linesnumbered,lined]{algorithm2e}
\SetKwInOut{Input}{input}
\SetKwInOut{Output}{output}

\usepackage{amsmath,amsfonts,bm}

\newcommand{\C}{{\mathbb C}}
\newcommand{\diag}[1]{\textbf{diag}\left(#1\right)}

\newcommand{\norm}[1]{\left\lVert#1\right\rVert}

\newcommand{\innerproduct}[2]{\langle #1, #2 \rangle}
\newcommand{\real}[1]{\text{Re}\left\{#1\right\}}
\newcommand{\imag}[1]{\text{Im}\left\{#1\right\}}
\newcommand\numberthis{\addtocounter{equation}{1}\tag{\theequation}}

\newcommand*{\rom}[1]{\expandafter\@slowromancap\romannumeral #1@}

\def\1{\bm{1}}

\def\ve{{\bm{e}}}

\def\vw{{\bm{w}}}

\def\vy{{\bm{y}}}

\DeclareMathAlphabet{\mathsfit}{\encodingdefault}{\sfdefault}{m}{sl}
\SetMathAlphabet{\mathsfit}{bold}{\encodingdefault}{\sfdefault}{bx}{n}

\newcommand{\E}{\mathbb{E}}

\newcommand{\R}{\mathbb{R}}

\DeclareMathOperator*{\argmax}{arg\,max}

\theoremstyle{plain}
\newtheorem{theorem}{Theorem}[section]
\newtheorem{proposition}[theorem]{Proposition}
\newtheorem{lemma}[theorem]{Lemma}

\theoremstyle{definition}

\theoremstyle{remark}
\newtheorem{remark}[theorem]{Remark}

\hypersetup{
  pdftitle={Model Selection and Parameter Estimation for Multidimensional Gaussian Mixture Models with a Common Covariance Matrix},
  pdfauthor={Xinyu Liu and Hai Zhang},
  pdfsubject={Model-order selection and parameter estimation for multidimensional Gaussian mixture models},
  pdfkeywords={Gaussian mixture models, model-order selection, minimax testing, Fourier methods, mean estimation}
}

\title{Model Selection and Parameter Estimation for Multidimensional Gaussian Mixture Models with a Common Covariance Matrix}
\author[1]{Xinyu Liu \thanks{Email: \texttt{xliuem@connect.ust.hk}}}
\author[1,2]{Hai Zhang \thanks{Email: \texttt{haizhang@ust.hk}}}

\affil[1]{Department of Mathematics, Hong Kong University of Science and Technology (HKUST)}
\affil[2]{HKUST-Shenzhen-Hong Kong Collaborative Innovation Research Institute}
\date{}
\begin{document}
\maketitle

\begin{abstract}
We develop a Fourier-domain framework for model-order selection and component-mean estimation in multidimensional Gaussian mixture models (GMMs) with a known common covariance matrix. The framework exploits the low-rank structure induced by empirical characteristic-function measurements.

We first establish a minimax lower bound for model-order
testing, showing that, for mixtures with common covariance $I_d$,
distinguishing a $k$-component GMM with a minimum separation $\Delta$ from the class of
$(k-1)$-component GMMs requires $\Omega(\Delta^{-(4k-4)})$ samples. We then
propose an oracle thresholding estimator that operates on the spectrum of an
empirical covariance matrix formed from random Fourier measurement vectors.
For fixed $k$, its dominant data-processing cost scales linearly in both the
sample size $n$ and the ambient dimension $d$. Its sufficient sample size
scales as $\Delta^{-(8k-8)}$, leaving a gap between the sufficient exponent
$8k-8$ and the minimax necessary exponent $4k-4$. We further develop a
practical singular-value-ratio estimator that reduces the dependence on
unknown model parameters at the cost of a stronger sample-size requirement.
  
Given the model order, we then introduce a gradient-based method for estimating the component means. Inspired by the MUSIC algorithm for line spectral estimation, we localize the component means by minimizing the orthogonal projection distance between parameterized Fourier steering vectors and the estimated signal subspace. To navigate the non-convex objective landscape, we employ a data-driven, score-based initialization strategy. Under an explicit sample-size condition, a qualifying sample initialization exists with high probability and lies in a certified attraction ball, after which the gradient iteration converges linearly. We prove that this method estimates the component means at the optimal parametric rate $\mathcal{O}_p(n^{-1/2})$, provided the separation between components is bounded below by a fixed constant.
To enhance the algorithm's efficiency in the regime where the ambient dimension exceeds the number of mixture components (i.e., \(d > k\)), we integrate principal component analysis (PCA) for dimension reduction. Numerical experiments demonstrate that our Fourier-based algorithmic framework consistently reduces computational time relative to conventional Expectation-Maximization (EM) methods, while achieving comparable or superior estimation accuracy across a range of scenarios.
\end{abstract}

\medskip
\noindent\textbf{Keywords:} Gaussian mixture models; model-order selection;
minimax testing; Fourier methods; mean estimation.

\section{Introduction}
The Gaussian Mixture Model (GMM) is a fundamental statistical framework utilized extensively across machine learning, pattern recognition, and robust statistics. It provides a powerful mechanism for modeling complex, heterogeneous data distributions originating from distinct sub-populations. The GMM represents a probability distribution as a weighted sum of Gaussian components, each characterized by its mean and covariance matrix. Formally, each independent observation $x \in \mathbb{R}^d$ from a $k$-component GMM is generated according to the density:
    \begin{equation}
        p(x) = \sum_{i=1}^{k} w_{i} \mathcal{N}(x; \mu_{i}, \Sigma_{i}),
    \end{equation}
where $w_i$ is the mixing weight such that $w_i > 0$ and $\sum_{i=1}^k w_i = 1$. The mean vector and the covariance matrix of the $i$-th component are denoted as $\mu_i$ and $\Sigma_i$, respectively. Given the i.i.d. samples drawn from the mixture distribution, the challenge is to learn the underlying model. 

The learning problem for GMMs is generally categorized into three highly interconnected tasks: (1) \textit{Clustering}, which infers the latent component assignment for each sample; (2) \textit{Parameter Estimation}, which recovers the precise values of the weights, means, and covariances up to a global permutation; and (3) \textit{Density Estimation}, which seeks to approximate the true probability density function under specific statistical metrics such as the Kullback-Leibler divergence or the Wasserstein distance.

In the context of clustering, objective functions such as $k$-means are generally NP-hard to optimize globally, even for $k=2$ clusters (see \cite{aloise2009np}). Traditional approximation methods, such as Lloyd's algorithm \cite{lloyd1982least}, are highly sensitive to initialization and dimension. Recent theoretical advancements have established that perfectly clustering high-dimensional Gaussian mixtures fundamentally depends on the separation distance $\Delta = \min_{i \neq j} \|\mu_i - \mu_j\|_2$. For instance, \cite{ndaoud2022sharp} proved that the critical threshold for the exact recovery of a two-component mixture with covariance $\sigma^2 I_d$ explicitly depends on the ambient dimension $d$:\begin{equation}\Delta^2 > \sigma^2 \left(1 + \sqrt{1 + \frac{2d}{n \log n}} \right) \log n.\end{equation}

For parameter estimation, the high-dimensional regime severely exacerbates the ``curse of dimensionality''. Traditional iterative methods like the Expectation-Maximization (EM) algorithm (\cite{dempster1977maximum}) often suffer from exponentially slow convergence rates or become trapped in spurious local optima when the dimension $d$ is large and the overlap between components is significant. To bypass the non-convexity of the likelihood landscape, method-of-moments and spectral approaches have gained prominence. \cite{vempala2002spectral} proposed a spectral algorithm that projects samples onto the subspace spanned by the top principal components, successfully learning mixtures of spherical Gaussians provided the component centers exhibit sufficient separation. \cite{hsu2013learning} later developed an efficient method-of-moments estimator using the spectral decomposition of low-order observable moments; this approach precludes the need for explicit minimum separation assumptions, but operates under the strict non-degeneracy condition that the component mean vectors are linearly independent (spanning a $k$-dimensional subspace).

The fundamental information-theoretic limits of parameter estimation for Gaussian mixtures were quantified by \cite{moitra2010settling}. They demonstrated that learning general GMMs necessitates a sample size and runtime that depend exponentially on the number of components $k$. Specifically, by constructing two distinct 1-dimensional Gaussian mixtures with $k$ components that are separated in parameter space by $\Omega(1/k)$ yet exhibit an exponentially small statistical distance bounded by $\mathcal{O}(e^{-\Omega(k)})$, they proved that distinguishing such models intrinsically requires a sample size exponential in $k$. Moving beyond the 1-dimensional worst-case scenario to high-dimensional spaces, \cite{regev2017learning} later established a sharp phase transition for the sample complexity of $k$ spherical GMMs. They proved that parameter recovery requires super-polynomial samples if the spatial separation $\Delta$ falls below $o(\sqrt{\log k})$, whereas a polynomial number of samples is strictly sufficient if $\Delta = \Omega(\sqrt{\log k})$. Furthermore, \cite{doss2023optimal} established the optimal minimax rate for estimating the mixing distribution of high-dimensional location mixtures under the Wasserstein distance, proving that the worst-case error is bounded by $\mathcal{O}\left( (d/n)^{1/4} + n^{-1/(4k-2)} \right)$, highlighting the severe degradation in statistical efficiency when the model components are severely overlapping.

A critical limitation in nearly all the aforementioned high-dimensional clustering and parameter estimation algorithms is the strict requirement that the model order $k$ (the true number of components) is known exactly a priori.
In practice, however, $k$ is rarely known, and its accurate estimation, formally termed as a model selection problem, is of great importance.
Selecting the optimal $k$ dictates the fundamental information-theoretic tradeoff between model expressivity and statistical efficiency. Underestimating $k$ forces disparate subpopulations to merge, violating the structural assumptions of the data-generating process and destroying the consistency of subsequent parameter estimates. Conversely, overestimating $k$ leads to over-parameterization, resulting in spurious clusters and loss of identifiability.

The likelihood ratio test (LRT) can be used to determine the model order. However, the asymptotic result may fail, see \cite{hartigan1985failure}. Some alternatives have been proposed, including the modified likelihood ratio test (MLRT) \cite{chen1998penalized, chen2001likelihood}, the EM test \cite{li2010testing} and the penalized method \cite{gassiat2012consistent}. Another widely used method is based on the information criteria such as BIC (\cite{schwarz1978estimating}) and AIC (\cite{akaike1998information}). The properties of AIC and BIC are investigated in \cite{leroux1992consistent}. The above methods all require separate fitting for each candidate model, which can be costly in high-dimensional scenarios. A computationally inexpensive alternative estimates $k$ from sample-covariance eigengaps. In the spherical case, or after whitening, it can recover $k$ only if the centered means span a $(k-1)$-dimensional subspace (equivalently, the component means are affinely independent) and the smallest covariance spike is detectable above sampling noise \cite{kritchman2009nonparametric,passemier2012determining}. The goal of this paper is to develop a Fourier-domain model-order selection method with finite-sample guarantees for multidimensional GMMs, together with a score-initialized method for mean estimation that admits a high-probability convergence guarantee under an explicit sample-size condition.

\subsection{Problem Setting}
Consider $n$ independent samples drawn from a $d$-dimensional Gaussian mixture distribution with a common covariance matrix:
    \begin{equation}
        \label{eqn:model setting}
        x_j \sim \sum_{i=1}^k w_i\mathcal{N}(\mu_i, \Sigma), \ j = 1,2,\cdots,n.
    \end{equation}
We assume that the covariance matrix $\Sigma \in \R^{d\times d}$ is symmetric and positive definite and is known a priori. This scenario is referred to as the Gaussian location mixture if $\Sigma = \sigma^2 I_d$. We define the separation distance $\Delta$ and the minimal weight $w_{\min}$ of model \eqref{eqn:model setting} as 
    \[
        \Delta = \min_{1\leq i < j \leq k} \norm{\mu_i - \mu_j}_2, \quad w_{\min} = \min_{1\leq i \leq k}w_i.
    \]
Throughout, we 
assume the following nondegeneracy condition:
\[
\Delta > 0,
\quad
w_{\min} > 0 ,
\]
so that the component means are distinct and every component carries positive
weight.

Model \eqref{eqn:model setting} can also be written in convolution form:
    \begin{equation}
        \label{eqn: model convolution}
        \sum_{i=1}^k w_i \delta_{\mu_i} * \mathcal{N}(0, \Sigma).
    \end{equation}
Here $\nu = \sum_{i=1}^k w_i\delta_{\mu_i}$ is known as the mixing distribution. In this paper, we study the problem of learning model \eqref{eqn:model setting} from $n$ i.i.d. samples. More specifically, we aim to determine the model order $k$ and estimate the component means from the samples.

\subsection{Our Contributions}
The main contribution of this work is a Fourier-domain methodology for
estimating both the number of components and the component means in
multidimensional Gaussian mixture models (GMMs) with a known common covariance
matrix. The proposed framework exploits the low-rank structure of Fourier
measurements derived from the empirical characteristic function of the observed
samples and yields a finite-sample model-order selection guarantee directly in
the multidimensional setting.

The multidimensional setting introduces two new challenges. First, since the Gaussian means need not be collinear, the low-rank structure in the Fourier domain can no longer be captured by a Hankel matrix formed from the Fourier measurements. To overcome this, we instead exploit the covariance matrix of the random Fourier measurement vectors, which requires a more intricate analysis. Second, the direct application of a MUSIC algorithm, which forms an imaging function evaluated over grid points, becomes computationally prohibitive in multiple dimensions. To address this, we introduce an iterative method with a data-driven, score-based initialization and establish a high-probability linear convergence guarantee under an explicit sample-size condition.

Our first theoretical contribution is a procedure-independent minimax lower
bound for model-order testing. By embedding a one-dimensional moment-matching
construction into $\mathbb{R}^d$, we show that, even within the
identity-covariance subclass, uniformly distinguishing a separated
$k$-component GMM from the class of $(k-1)$-component GMMs requires at least
$\Omega(\Delta^{-(4k-4)})$ samples. This result provides a minimax benchmark
for the Fourier-based estimators developed below.

To estimate the component number $k$, we evaluate the empirical characteristic
function at randomly sampled frequency vectors and use the resulting
measurements to construct an empirical Fourier covariance matrix. We introduce
an oracle thresholding estimator that determines the model order from the
spectrum of this matrix. When the numbers of Fourier measurement and
translation vectors are proportional to $k$, its time complexity is
$\mathcal{O}(ndk^2+k^3)$. Its sufficient sample size scales as
$\Delta^{-(8k-8)}$, leaving a gap relative to the minimax necessary scaling
$\Delta^{-(4k-4)}$. We further develop a practical singular-value-ratio
estimator with reduced dependence on unknown model parameters, at the cost of a
stronger sample-size requirement.

To estimate the component means, we minimize the projection error between parameterized Fourier 
steering vectors and the estimated signal subspace of the empirical Fourier 
covariance matrix. The resulting non-convex optimization problem is solved using 
gradient descent with a data-driven, score-based initialization strategy. Under
an explicit sample-size condition, a qualifying sample initialization exists
with high probability and lies in a certified attraction ball. Starting from
such an initialization, the subsequent iterates remain in a locally strongly
convex basin and converge linearly. For a fixed model with the component
separation bounded away from zero, we further establish that the resulting mean
estimates attain the optimal parametric rate $\mathcal{O}_p(n^{-1/2})$. In the setting where the ambient dimension exceeds the number of 
components, $d>k$, we incorporate principal component analysis (PCA) to reduce 
the computational burden. Numerical experiments across both low- and 
high-dimensional settings demonstrate that the proposed Fourier-based framework 
consistently reduces computational time relative to the 
Expectation-Maximization (EM) algorithm while achieving comparable or superior 
estimation accuracy.

Finally, we note that extending the methodology to GMMs with arbitrary component-specific covariance matrices presents significant challenges and requires new insights. For mixtures with unrestricted covariances, the Expectation-Maximization (EM) algorithm currently remains the standard empirical approach, despite its known issues, such as convergence to local optima and sensitivity to initialization. Developing computationally efficient algorithms with provable theoretical guarantees for this general setting remains an important avenue for future research.

\subsection{Paper Organization and Notations}
In Section \ref{sec: model selection}, we first establish a minimax lower bound for model-order testing by embedding a one-dimensional hard instance into the multidimensional setting. We then introduce a Fourier approach for model selection of GMMs. The proposed estimator is proven to be consistent, and its sample complexity order matches the minimax lower bound up to a doubling of the exponent in $\Delta$. In Section \ref{sec: mean estimation}, we propose a gradient-descent method for estimating the component means with a data-driven, score-based initialization. We establish a high-probability linear convergence guarantee under an explicit sample-size condition, derive the resulting estimation rate, and incorporate dimension reduction to improve efficiency for high-dimensional mixtures. In Section \ref{sec: numericals}, numerical studies are provided to illustrate the effectiveness of our algorithms. Proofs of the theoretical results are collected in the Appendix.

\section{Model Selection and Sample Complexity}
\label{sec: model selection}
In this section, we first establish a procedure-independent minimax lower bound
for distinguishing a $k$-component GMM from a $(k-1)$-component GMM. This lower
bound serves as the theoretical benchmark for our proposed Fourier-based
approach.

We then introduce a Fourier-based approach for model order selection. Using the exponential form of the covariance-compensated characteristic function of the GMM, we construct a latent Fourier covariance matrix from random Fourier measurement vectors. We characterize the algebraic and structural properties of the covariance matrix, explicitly quantifying its singular value spectrum as a function of the underlying model parameters. Leveraging this spectral characterization, we then develop two statistically consistent estimators for the true model order \(k\), accompanied by finite-sample complexity guarantees.

\subsection{A Minimax Lower Bound via One-Dimensional Embedding}
\label{sec:one-dimensional-lower-bound-embedding}
We now embed the hard one-dimensional construction of
\cite{liu2024fourier} into $\mathbb R^d$. Throughout this subsection, the
common covariance is $I_d$. Fix a unit vector $v\in\mathbb S^{d-1}$, set
\begin{equation}
    a_i=-1+\frac{2(i-1)}{k-1},
    \qquad
    \mu_i=Ra_i v,
    \qquad i=1,\ldots,k,
\end{equation}
and define
\begin{equation}
\label{eq:embedded-hard-P}
    P_R^{(d)}:=\sum_{i=1}^k w_i\mathcal N(Ra_i v,I_d).
\end{equation}
The means are equally spaced, with minimum separation
$\Delta=2R/(k-1)$. Let
$\sum_{i=1}^{k-1}\pi_i\delta_{b_i}$ be the $(k-1)$-point Gaussian
quadrature distribution that matches the first $2k-3$ moments of
$\sum_{i=1}^k w_i\delta_{a_i}$, and define
\begin{equation}
\label{eq:embedded-hard-Q}
    Q_R^{(d)}:=\sum_{i=1}^{k-1}\pi_i\mathcal N(Rb_i v,I_d).
\end{equation}

Choose an orthogonal matrix whose first column is $v$. In the resulting
coordinates, both $P_R^{(d)}$ and $Q_R^{(d)}$ factor into their respective
one-dimensional mixtures along $v$ and the same independent
$\mathcal N(0,I_{d-1})$ distribution on $v^\perp$. Therefore,
\begin{equation}
\label{eq:embedded-kl-identity}
    \log\frac{p_R^{(d)}(x)}{q_R^{(d)}(x)}
    =\log\frac{p_R^{(1)}(v^\mathrm Tx)}{q_R^{(1)}(v^\mathrm Tx)},
    \qquad
    D_{\mathrm{KL}}\!\left(P_R^{(d)}\middle\|Q_R^{(d)}\right)
    =D_{\mathrm{KL}}\!\left(P_R^{(1)}\middle\|Q_R^{(1)}\right).
\end{equation}
Thus the Kullback--Leibler divergence of the one-dimensional hard instance is
preserved under this embedding. To formulate the corresponding minimax testing
problem, let
\begin{align*}
    \mathcal P_{\Delta,d}^{(k)}
    &:=\left\{
       \sum_{i=1}^k\omega_i\mathcal N(\theta_i,I_d):
       \min_{i\ne j}\|\theta_i-\theta_j\|_2\ge\Delta,
       \ \omega_i\ge w_{\min},\ \sum_{i=1}^k\omega_i=1
    \right\},\\
    \mathcal Q_d^{(k-1)}
    &:=\left\{
       \sum_{i=1}^{k-1}\pi_i\mathcal N(\beta_i,I_d):
       \pi_i>0,\ \sum_{i=1}^{k-1}\pi_i=1
    \right\},
\end{align*}
where $0<w_{\min}\le\min_iw_i$. Define the minimax testing risk
\begin{equation}
\label{eq:embedded-minimax-risk}
    \mathcal R_{n,d}^{(k)}(\Delta)
    :=\inf_{\psi}\left[
       \sup_{Q\in\mathcal Q_d^{(k-1)}}
          \mathbb E_{Q^{\otimes n}}[\psi]
       +
       \sup_{P\in\mathcal P_{\Delta,d}^{(k)}}
          \mathbb E_{P^{\otimes n}}[1-\psi]
    \right],
\end{equation}
where $\psi=0$ denotes order $k-1$ and $\psi=1$ denotes order $k$.

\begin{theorem}[Embedded minimax lower bound]
\label{thm:embedded-minimax-lower-bound}
Fix $k\ge2$ and positive weights $w_1,\ldots,w_k$ summing to one. There
exists $\kappa_{k,w}>0$, depending only on $k$ and the weights, such that,
for every $\varepsilon\in(0,1)$, there exists $\Delta_0>0$ for which, whenever
$0<\Delta\le\Delta_0$,
\begin{equation}
\label{eq:embedded-minimax-lower-bound}
    \mathcal R_{n,d}^{(k)}(\Delta)
    \ge
    1-\sqrt{\frac{n}{2}
       D_{\mathrm{KL}}\!\left(P_R^{(d)}\middle\|Q_R^{(d)}\right)}
    \ge
    1-\sqrt{\frac{(1+\varepsilon)\kappa_{k,w}}{2}}\,
       \Delta^{2k-2}\sqrt n.
\end{equation}
Consequently, for every $\eta\in(0,1)$,
$\mathcal R_{n,d}^{(k)}(\Delta)\ge\eta$ whenever
\begin{equation}
    n\le
    \frac{2(1-\eta)^2}{(1+\varepsilon)\kappa_{k,w}}
    \frac{1}{\Delta^{4k-4}}.
\end{equation}
Thus no model-order estimator can uniformly distinguish order $k$ from
order $k-1$ below the scale $\Delta^{-(4k-4)}$ over these classes.
\end{theorem}

Theorem~\ref{thm:embedded-minimax-lower-bound} follows from Theorem~2.7 of
\cite{liu2024fourier}, Le Cam's two-point method, and the factorization in
\eqref{eq:embedded-kl-identity}.

\subsection{Fourier Measurements and Latent Fourier Covariance Matrix}
The Fourier transform of the density in \eqref{eqn:model setting} is given by
    \begin{equation}
        \label{eqn:cf}
        \phi(t) = \exp\left(- \frac{t^\mathrm{T} \Sigma t}{2}\right) \sum_{i=1}^k w_i \exp\left(\iota \innerproduct{\mu_i}{t}\right),
    \end{equation}
which is also known as the characteristic function of the GMM. It can be empirically estimated from i.i.d. samples $\{x_i\}_{i=1}^n$ by the empirical characteristic function:
    \begin{equation}
        \label{eqn:ecf}
        \hat{\phi}_n(t) = \frac{1}{n}\sum_{j=1}^n \exp(\iota \innerproduct{x_j}{t}).
    \end{equation}
To isolate the spectral contribution of the component means, we introduce the following covariance-compensated empirical characteristic function
    \begin{equation}
        \label{eqn:modulated Fourier}
        \hat{y}_n(t) = \exp\left(\frac{t^\mathrm{T} \Sigma t}{2}\right)\hat{\phi}_n(t) = \exp\left(\frac{t^\mathrm{T} \Sigma t}{2}\right) \frac{1}{n}\sum_{j=1}^n \exp(\iota \innerproduct{x_j}{t})
    \end{equation}
    and refer to it as the \textbf{Fourier measurement} henceforth. 
We denote 
$$
y(t) = \sum_{i=1}^k w_i e^{\iota \innerproduct{\mu_i}{t}}
$$
and call it the \textbf{latent spectral signal}. Consequently, the Fourier measurement can be decomposed as
\begin{equation}
    \hat{y}_n(t) = y(t) + e_n(t),
\end{equation}
where $e_n(t)$ can be regarded as the noise term due to the finite sample size $n$. The noise level $\norm{e_n(t)}_{\infty}$ can be quantified in a probability sense by the following proposition:

\begin{proposition}
\label{prop: fourier concentration}
     For any fixed $\epsilon > 0$, we have the concentration that
    \begin{equation*}
        \mathbb P\left(|e_n(t)| \geq \epsilon\right)  \leq 4\exp \left( -\frac{n\epsilon^2}{4e^{t^\mathrm{T}\Sigma t}} \right) 
        \leq 4\exp \left( -\frac{n\epsilon^2}{4e^{\norm{t}_2^2 \sigma_{\max}(\Sigma)}} \right),
    \end{equation*}
    where $\sigma_{\max}(\Sigma)$ denotes the maximum singular value of $\Sigma$. Then for any $\delta \in (0,1)$, if the sample size $n \geq \frac{4}{\epsilon^2}\ln\left(\frac{4}{\delta}\right){\exp\left(\norm{t}_2^2 \sigma_{\max}(\Sigma)\right)}$, with probability $1-\delta$, we have that $|e_n(t)| < \epsilon$.
\end{proposition}

Next, we use the latent spectral signal $y(t)$ to demonstrate the main idea behind our algorithm for model selection. We first randomly generate $L$ \textbf{frequency vectors} $\{t_i\}_{i=1}^L$ in the Fourier space that follow a prescribed distribution $\mathcal{D}$ (eg, a uniform distribution in a bounded ball, or a Gaussian distribution). Define the \textbf{latent spectral vector} at the generated points as 
    \[
        \vy_0 = \begin{bmatrix}
            y(t_1) & y(t_2) & \cdots & y(t_L)
        \end{bmatrix}^\mathrm{T}.
    \]
Given that $y(t) = \sum_{i=1}^k w_i e^{\iota \innerproduct{\mu_i}{t}}$, we can express $\vy_0$ as 
    \[
        \vy_0 = 
        \begin{bmatrix}
            e^{\iota \innerproduct{\mu_1}{t_1}} & e^{\iota \innerproduct{\mu_2}{t_1}} & \cdots & e^{\iota \innerproduct{\mu_k}{t_1}} \\
            e^{\iota \innerproduct{\mu_1}{t_2}} & e^{\iota \innerproduct{\mu_2}{t_2}} & \cdots & e^{\iota \innerproduct{\mu_k}{t_2}} \\
            \vdots & \vdots & \ddots & \vdots \\
            e^{\iota \innerproduct{\mu_1}{t_L}} & e^{\iota \innerproduct{\mu_2}{t_L}} & \cdots & e^{\iota \innerproduct{\mu_k}{t_L}}
        \end{bmatrix}
        \begin{bmatrix}
            w_1 \\ w_2 \\ \vdots \\ w_k
        \end{bmatrix} =: \mathcal A \vw.
    \]
We also introduce a \textbf{Fourier steering vector} for each mean vector $\mu$:
$$\mathbf a_L(\mu) = [e^{\iota \innerproduct{\mu}{t_1}}, e^{\iota \innerproduct{\mu}{t_2}}, \cdots, e^{\iota \innerproduct{\mu}{t_L}}]^\mathrm{T}.$$ 
The matrix $\mathcal A$ can then be written as 
$$\mathcal A = [\mathbf a_L(\mu_1), \mathbf a_L(\mu_2), \cdots, \mathbf a_L(\mu_k)].$$
To extract the Fourier steering vectors of the means $\mu_1,\dots,\mu_k$, we
form a \textbf{latent Fourier covariance matrix} of the random latent spectral
vectors, which are obtained by translating the frequency vectors
$t_1,\dots,t_L$.  More precisely, let $\{v_m\}_{m=1}^{M}$ be a set of translation vectors and denote $v_0 = 0 \in \R^d$, we define
    \[
        \vy_m = \begin{bmatrix}
                y(t_1+v_m) & y(t_2 + v_m) & \cdots & y(t_L+v_m)
            \end{bmatrix}^\mathrm{T} \in \C^L,
    \]
    Then 
    \[
        \vy_m = \mathcal A \vw_m, \ m = 0,\cdots, M,
    \]
    where 
    \[
        \vw_m = \begin{bmatrix}
            w_1 e^{\iota \innerproduct{\mu_1}{v_m}} &
            w_2 e^{\iota \innerproduct{\mu_2}{v_m}} &
            \cdots &
            w_k e^{\iota \innerproduct{\mu_k}{v_m}}
        \end{bmatrix}^\mathrm{T} \in \C^k.
    \]
The latent Fourier covariance matrix of the Fourier vectors $\{\vy_m\}_{m=0}^M$ is then defined by 
    \begin{equation} \label{eq-C}
              C := \frac{1}{M+1}\sum_{m=0}^M \vy_m \vy_m^* = \frac{1}{M+1} \sum_{m=0}^M \mathcal A 
        \vw_m \vw_m^* \mathcal A^*
        = \mathcal A W \mathcal A^* \in \C^{L\times L},
    \end{equation}
where $$W = \frac{1}{M+1}\sum_{m=0}^M \vw_m \vw_m^* \in \C^{k\times k}.$$ The crux of the algorithm is: In the noiseless case with $L \geq k$, the image space of the latent Fourier covariance matrix $C$ coincides with the image space of $\mathcal A$, which is spanned by $k$ Fourier steering vectors $\mathbf a_L(\mu_1), \mathbf a_L(\mu_2), \cdots, \mathbf a_L(\mu_k)$. This image space can be decomposed by the orthogonal basis from the singular value decomposition with the following property:
    \begin{proposition}
        \label{thm: svd}
        Suppose that $M+1, L \geq k$ and \normalfont{rank}($\mathcal A$) $=$ \normalfont{rank}($W$) $=k$, then the latent Fourier covariance matrix $C$ has the following singular value decomposition (SVD), or eigendecomposition (as $C$ is Hermitian): 
        \[
            C = [U_1 \quad U_2] \diag{\sigma_1, \sigma_2, \cdots, \sigma_k, 0, \cdots, 0}[U_1 \quad U_2]^*,
        \]
        where $U_1 \in \C^{L\times k}, U_2 \in \C^{L\times(L-k)}$ and $\sigma_1 \geq \sigma_2 \geq \cdots \geq \sigma_k > 0$.
    \end{proposition}

    \begin{remark}
        The condition that \normalfont{rank}($\mathcal A$) $=$ \normalfont{rank}($W$) $=k$ is equivalent to that
        \begin{itemize}
            \item $\mathbf a_L(\mu_1),\cdots, \mathbf a_L(\mu_k)$ are linearly independent;
            \item $\vw_0, \vw_1, \cdots, \vw_M$ span the space $\mathbb{C}^k$ (or equivalently, the matrix $[\vw_0, \dots, \vw_M]$ has rank $k$).
        \end{itemize}
        The above two conditions are easily satisfied by choosing different $t_l$'s and $v_m$'s if all $\mu_i$'s are distinct.
    \end{remark}
The column space of $U_1$ spans the image space. In practice, the latent spectral signal $y(t)$ is approximated by $\hat y_n(t)$ and the latent Fourier covariance matrix $C$ by $\hat C$ with SVD (or eigendecompostion) as:
    \begin{equation}\label{eq-hatC}
        \hat C := \frac{1}{M+1}\sum_{m=0}^M \hat \vy_m \hat \vy_m^* = [\hat U_1 \quad \hat U_2] \diag{\hat \sigma_1, \hat \sigma_2, \cdots, \hat \sigma_k, \hat \sigma_{k+1}, \cdots, \hat \sigma_L}[\hat U_1 \quad \hat U_2]^*,
    \end{equation}
where $\hat \vy_m = \begin{bmatrix}
                \hat y_n(t_1+v_m) & \hat y_n(t_2 + v_m) & \cdots & \hat y_n(t_L+v_m)
\end{bmatrix}^\mathrm{T}$ denotes the \textbf{random Fourier measurements vector} and we refer to $\hat C$ as the \textbf{empirical Fourier covariance matrix}. 
Due to finite sample size, the empirical Fourier measurements deviate from the latent spectral signal. Consequently, the spectrum of the empirical Fourier covariance matrix \(C\) is perturbed, thereby obscuring its true underlying rank \(k\). We have the following estimate for the spectral norm of the error matrix
\[
\hat E= \hat C - C.
\]

\begin{proposition}[Spectral perturbation of the empirical Fourier covariance]
\label{prop:covariance-perturbation}
Let $N := L(M+1)$ be the total number of Fourier evaluation points
$\{t_i+v_m\}_{i=1,\dots,L}^{m=0,\dots,M}$, let
$$T_{\max} := \max_{i,m}\|t_i+v_m\|_2,$$ and set
\[
  {\kappa_\Sigma} := \exp\!\big(T_{\max}^2\,\sigma_{\max}(\Sigma)\big).
\]
Then for any $\delta\in(0,1)$, with probability at least $1-\delta$,
\[
  \|\hat E\|_2
  \;\le\; 4L \sqrt{\kappa_\Sigma}\sqrt{\frac{\ln(4N/\delta)}{n}}
       \;+\; 4L{\kappa_\Sigma}\,\frac{\ln(4N/\delta)}{n}
  \;=\; O\!\left(L\,\sqrt{\kappa_\Sigma}\sqrt{\frac{\ln\!\big(L(M+1)/\delta\big)}{n}}\right).
\]

\end{proposition}

\subsection{Signal Strength Analysis}
The ability to detect the correct model order $k$ depends fundamentally on the spectral properties of the latent Fourier covariance matrix $C$. Specifically, the magnitude of the smallest non-zero singular value $\sigma_k(C)$. If $\sigma_k(C)$ is too small relative to the sampling noise, the $k$-th component becomes indistinguishable from the noise floor.

In this subsection, we derive a lower bound for the singular values of $C$ in terms of model parameters (the minimum separation distance $\Delta$ and the mixing weights $w_{\min}$) and the sampling strategy for the frequency vectors and translation vectors. First, we establish that the high-dimensional separation of the means is preserved when projected onto a one-dimensional subspace. This allows us to reduce the analysis of the singular values to a univariate problem.

\begin{lemma}
    \label{lem:projected-separation}
    Let $\{\mu_{1}, \dots, \mu_{k}\} \subset \mathbb{R}^{d}$ be a set of distinct points with minimum separation distance $\Delta = \min_{i \ne j} \|\mu_{i} - \mu_{j}\|_2$. Let $\{t_1, \dots, t_J\}$ be a set of $J$ independent random vectors drawn uniformly from the unit sphere $\mathbb{S}^{d-1}$.
    For any confidence parameter $\delta \in (0, 1)$, if the number of sampled directions satisfies
    \begin{equation}
        J \ge \lceil \log_2(1/\delta) \rceil,
    \end{equation}
        then with probability at least $1-\delta$, there exists at least one index $j_* \in \{1, \dots, J\}$ such that the projected separation satisfies:
    \begin{equation}
        \min_{i \ne j} |\langle \mu_{i} - \mu_{j}, t_{j_*} \rangle| \ge \frac{\Delta}{k^2 \sqrt{d}}.
    \end{equation}
\end{lemma}

Using this geometric property, we can estimate the lower bound of the $k$-th singular value $\sigma_k(C)$ for a particular sampling strategy of generating frequency vectors and translation vectors as in the Proposition below.

\begin{proposition}[Lower Bound of $\sigma_k(C)$]
\label{prop:lower-bound-sampling}
Let $\{\mu_1,\dots,\mu_k\}\subset\mathbb{R}^d$ have minimum separation
$\Delta=\min_{i\neq j}\|\mu_i-\mu_j\|_2$, diameter
$R=\max_{i,j}\|\mu_i-\mu_j\|_2$, and let $w_{\min}=\min_i w_i$. Let $S$ and $M'$ be two integers such that 
\[
S+1\ge k, \quad M'+1 \ge k.
\]
Construct the frequency vectors and their translation vectors as follows:
\begin{itemize}
  \item \emph{Frequency vectors.} Draw direction vectors $t_1,\dots,t_J\sim
  \mathcal{U}\left({\mathbb{S}}^{d-1}\right)$ i.i.d. For each direction $t_j$, form $S+1$ {collinear}
  frequency vectors $t_{j,s}=s\tau t_j$, $s=0, 1,\dots,S$, with
  step size $\tau\le \pi/R$. This yields $L=J(S+1)$ frequency vectors.
  \item \emph{Translation vectors.} Draw direction vectors $v_1,\dots,v_J\sim
  \mathcal{U}({\mathbb{S}}^{d-1})$ i.i.d., and independently of $\{t_j\}$, and form, for each direction $v_j$, $M'+1$ {collinear} translation
  vectors
  \[
    v_{j,m} = m\,\rho\,v_j,\qquad m=0,1,\dots,M',
  \]
  with step size $\rho\le\pi/R$.
\end{itemize}
Then for any confidence level $\delta\in(0,1)$, if the number of measurement
directions satisfies
\[
  J \ge \big\lceil \log_2(2/\delta) \big\rceil,
\]
with probability at least $1-\delta$ the $k$-th singular {value} of $\mathcal A$, $W$,
and the latent Fourier covariance matrix $C$ admit the following lower bounds:
\begin{align}
  \sigma_k(\mathcal A) &\ge \frac{1}{\sqrt{k}}
    \left(\frac{\tau\Delta}{\pi k^2\sqrt{d}}\right)^{k-1},
    \label{eq:sigmak-Phi}\\[2pt]
  \sigma_k(W) & \ge \frac{w_{\min}^2}{kJ(M'+1)}
    \left( \frac{\rho\Delta}{\pi k^2 \sqrt{d}} \right)^{2k-2},
    \label{eq:sigmak-W}\\[2pt]
  \sigma_k(C) &\ge \frac{w_{\min}^2}{k^2J(M'+1)}
    \left(\frac{\tau\Delta}{\pi k^2\sqrt{d}}\right)^{2k-2}
    \left( \frac{\rho\Delta}{\pi k^2 \sqrt{d}} \right)^{2k-2}.
    \label{eq:sigmak-C}
\end{align}
In particular, if the two step sizes are taken equal, $\tau = \rho$, then
\[
\sigma_k(C) \ge \frac{w_{\min}^2}{k^2J(M'+1)}
  \left( \frac{\tau\Delta}{\pi k^2 \sqrt{d}} \right)^{4k-4}.
\]
\end{proposition}

\begin{remark}[Other sampling strategies and the price in the separation]
\label{rem:other-sampling}
The collinear-grid construction of Proposition~\ref{prop:lower-bound-sampling}
is only one admissible sampling design. A lower bound on $\sigma_k(C)$ can be obtained
for essentially any sampling strategy for which the resulting exponential
feature vectors $\{e^{\iota\langle t,\mu_i\rangle}\}_{i=1}^k$ remain linearly
independent with high probability; in particular it holds for the
\emph{fully random} case in which the measurement points $t_1,\dots,t_L$ and
translations $v_1,\dots,v_{M}$ are drawn i.i.d.\ (e.g.\ uniformly on a ball
of radius $\pi/R$, or as i.i.d.\ Gaussians), \emph{without} imposing the collinear structure. 

The difference lies in the \emph{sharpness} of the resulting bound as a
function of the minimum separation $\Delta$. The collinear design aligns the
measurement grid so that the $k$ nodes are resolved along a common
one-dimensional frequency axis; this coherent alignment reduces the problem to
a scalar Vandermonde system and yields the dependence
$\big(\tau\Delta/(\pi k^2\sqrt d)\big)^{k-1}$ in
\eqref{eq:sigmak-Phi}--\eqref{eq:sigmak-C}. For unstructured sampling, no such alignment is available, and we are not able to establish a lower bound on
$\sigma_k(C)$ with a sharper dependence on $\Delta$ than in the collinear case.
In our numerical experiments, however, we adopt the unstructured random
sampling strategy, which yields appealing results despite no theoretical guarantee. Closing this gap between theory and practice is left as
an open problem for future investigation.
\end{remark}

\subsection{Algorithms and Theoretical Results}

\subsubsection{An oracle thresholding algorithm}
In practice, we compute the empirical Fourier covariance matrix $\hat{C}$ from noisy measurements. Let $\hat{\sigma}_1 \ge \dots \ge \hat{\sigma}_L$ be its singular values.
Given the singular value relation, the model order can be determined through a thresholding scheme for a chosen threshold $\varepsilon > 0$:
    \[
        \hat k = \max\{1\leq l\leq L-1:\hat\sigma_l \ge \varepsilon\}.
    \]

\begin{theorem}
\label{thm:thresholding}
    Adopt the setup and assumptions of Proposition \ref{prop:lower-bound-sampling},
    and suppose in addition that all Fourier evaluation points satisfy
    $\|t_{j,s}+v_{j,m}\|_2\le T_{\max}$ for some $T_{\max}>0$, and that the two step sizes coincide,
    $\tau=\rho$. Denote ${\kappa_\Sigma}= \exp\!\big(T_{\max}^2\,\sigma_{\max}(\Sigma)\big)$. Fix a
    confidence level $\delta\in(0,1)$ and a threshold parameter $0<\varepsilon<1$.
    If the sample size $n$ satisfies
        \begin{equation}
        \label{eqn: threshold condition 1}
            n > \; \frac{64\,L^2\,\kappa_\Sigma\,\ln(4N/\delta)}{\varepsilon^2},
        \end{equation}
    then with probability at least $1-\delta$
        \[
            \hat{\sigma}_l < \varepsilon, \qquad l = k+1,\dots, L .
        \]
    Moreover, if the threshold parameter $\varepsilon$ additionally satisfies
        \begin{equation}
        \label{eqn: threshold condition 2}
            \varepsilon < \frac{w_{\min}^2}{2k^2 J(M'+1)}
                \left(\frac{\tau\Delta}{\pi k^2 \sqrt{d}}\right)^{4k-4},
        \end{equation}
    then
        \[
            \hat{\sigma}_k > \varepsilon .
        \]
    Both conclusions hold simultaneously with probability at least $1-2\delta$, and consequently, the thresholding estimator recovers the true order exactly,
\[
   \hat k=\max\{\,1\le l\le L-1:\hat\sigma_l\ge\varepsilon\,\}=k .
\]
\end{theorem}

\begin{remark}
\label{remark: sampling complexity}
    Suppose the threshold is chosen admissibly, i.e. $\varepsilon$ satisfies
    \eqref{eqn: threshold condition 2}. To make
    \eqref{eqn: threshold condition 1} hold with the largest such threshold,
    take $\varepsilon$ of the order of its upper bound in
    \eqref{eqn: threshold condition 2},
    \[
      \varepsilon \;\asymp\; \frac{w_{\min}^2}{k^2 J(M'+1)}
        \left(\frac{\tau\Delta}{\pi k^2\sqrt d}\right)^{4k-4}
      \;\asymp\; w_{\min}^2\,\Delta^{4k-4}.
    \]
    Since \eqref{eqn: threshold condition 1} requires $n\gtrsim
    \varepsilon^{-2}\log(1/\delta)$, consistent model-order selection using our
Fourier-based thresholding approach is achievable
    once
    \[
        n \;\gtrsim\; \frac{1}{w_{\min}^{4}\,\Delta^{\,8k-8}}\,\log(1/\delta).
    \]

Comparing this sufficient condition to the procedure-independent lower bound
in Theorem~\ref{thm:embedded-minimax-lower-bound}, which shows that reliable
model-order testing requires $\Omega\!\big(\Delta^{-(4k-4)}\big)$ samples in
the embedded worst-case family, we see the sufficient rate carries
$\Delta^{-(8k-8)}$ against the
necessary $\Delta^{-(4k-4)}$. Closing the gap
remains open. This gap originates in
the threshold condition~\eqref{eqn: threshold condition 2}, where the admissible threshold is tied to the signal singular value
$\sigma_k(C)\sim\Delta^{4k-4}$; since $\varepsilon$ then enters the sample-size
requirement quadratically, the exponent is doubled. We expect that better sampling strategies for the frequency vectors and their translation
vectors could yield a sharper lower estimate of $\sigma_k(C)$, thereby, relaxing the threshold condition and narrowing, or even closing, the gap. 
\end{remark}

\subsubsection{A practical spectral-gap algorithm}
We note that implementing the oracle thresholding algorithm requires precise prior knowledge of $\sigma_k(C)$, which depends on the minimum separation distance $\Delta$ and weight $w_{\min}$, to set a valid threshold $\varepsilon$. In most practical applications, these parameters are unknown, rendering this theoretical algorithm infeasible. To mitigate the dependency on unknown model parameters, we propose an algorithm based on the ratio of singular values. Given finite samples, the algorithm returns the model order as the index with the largest ratio between adjacent singular values:
    \[
        \hat{k} = \argmax_{1\leq l\leq L-1}
        \frac{\hat \sigma_l}{\hat \sigma_{l+1}}.
    \]

To handle this instability, we regularize the denominator by a
screening threshold set at a fraction of the perturbation level. Let
\begin{equation}
\label{eq:perturbation-level-ratio}
    \eta_n
    :=
    4L\sqrt{\kappa_\Sigma}
    \sqrt{\frac{\ln(4N/\delta)}{n}}
    +
    4L\kappa_\Sigma
    \frac{\ln(4N/\delta)}{n},
    \qquad
    \kappa_\Sigma
    :=
    \exp\!\bigl(T_{\max}^2\sigma_{\max}(\Sigma)\bigr),
\end{equation}
so that $\|\widehat E\|_2\le\eta_n$ holds with high probability (cf.\
Proposition~\ref{prop:covariance-perturbation}). Fix a constant $c\in(0,1)$
and set
\begin{equation}
\label{eq:conservative-ratio-threshold}
    \varepsilon_*:=c\,\eta_n.
\end{equation}
We stress that $\varepsilon_*$ is \emph{not} intended to screen out every
noise singular value: since $\varepsilon_*<\eta_n$, some trailing
$\widehat\sigma_l$ ($l>k$) may exceed $\varepsilon_*$ and survive. Its sole
purpose is to place a floor under the denominator, so that no noise--noise
ratio can exceed $\eta_n/\varepsilon_*=1/c$.

Concretely, define the screened index set
\[
    \widehat{\mathcal I}_{\varepsilon_*}
    :=
    \bigl\{
        1\le l\le L-1:
        \widehat{\sigma}_l>\varepsilon_*
    \bigr\},
\]
and, for each $l\in\widehat{\mathcal I}_{\varepsilon_*}$, the
threshold-regularized ratio
\begin{equation}
\label{eq:regularized-ratio}
    \widehat R_l(\varepsilon_*)
    :=
    \frac{\widehat{\sigma}_l}
         {\widehat{\sigma}_{l+1}\vee\varepsilon_*},
    \qquad
    a\vee b:=\max\{a,b\}.
\end{equation}
The model-order estimator is then
\begin{equation}
\label{eq:screened-ratio-estimator-revised}
    \widehat k
    :=
    \underset{l\in\widehat{\mathcal I}_{\varepsilon_*}}
              {\arg\max}\,
    \widehat R_l(\varepsilon_*).
\end{equation}
The consistency of this estimator is established below.

\begin{theorem}
\label{thm:model-selection}
Suppose that
\[
    \|t_l+v_m\|_2\le T_{\max}
\]
for every Fourier evaluation point and that $L\ge k+1$. Let
\[
    \sigma_1\ge\cdots\ge\sigma_k>0
    =\sigma_{k+1}=\cdots=\sigma_L
\]
be the singular values of the latent Fourier covariance matrix $C$.
Fix $\delta\in(0,\tfrac12)$ and $c\in(0,1)$, and define $\eta_n$ and
$\varepsilon_*$ by \eqref{eq:perturbation-level-ratio} and
\eqref{eq:conservative-ratio-threshold}, respectively.

Set
\begin{equation}
\label{eq:ratio-detectability-level}
    b_c
    :=
    \min\left\{
        \frac{\sigma_k^2}{\sigma_1+2\sigma_k},
        \frac{c\sigma_k}{1+c}
    \right\}.
\end{equation}
If
\begin{equation}
\label{eq:modelselection-n-revised}
    n>
    \max\left\{
        \frac{64L^2\kappa_\Sigma}{b_c^2},
        \frac{8L\kappa_\Sigma}{b_c}
    \right\}
    \ln\!\left(\frac{4N}{\delta}\right),
\end{equation}
then, with probability at least $1-\delta$,
\[
    \widehat R_k(\varepsilon_*)
    >
    \max_{1\le l<k}\widehat R_l(\varepsilon_*)
\]
and
\[
    \widehat R_k(\varepsilon_*)
    >
    \max_{k<l\le L-1}
    \widehat R_l(\varepsilon_*).
\]
Consequently, the estimator
\eqref{eq:screened-ratio-estimator-revised} returns the correct model
order $\widehat k=k$ with probability at least $1-\delta$.
\end{theorem}

    \begin{remark}
            The bound \eqref{eq:modelselection-n-revised} is stated in terms of the fixed
    singular values $\sigma_1,\sigma_k$ of the latent Fourier covariance matrix $C$. If, in addition, 
    $C$ arises from the construction of
    Proposition~\ref{prop:lower-bound-sampling} with equal step sizes
    $\tau=\rho$, then the lower bound \eqref{eq:sigmak-C} gives
$\sigma_k=\Omega\!\big(w_{\min}^2\,\Delta^{4k-4}\big)$ (with $k,d,\tau,J,M'$
    held fixed) while $\sigma_1=O(1)$. Then the sample complexity for the correct model
    selection scales as
    \[
        n=\Omega\!\left(\frac{\sigma_1^2}{\sigma_k^4}\log(1/\delta)\right)
         =\Omega\!\left(\frac{1}{w_{\min}^{8}\,\Delta^{16k-16}}\log(1/\delta)\right).
    \]
    \end{remark}

    \begin{algorithm}[H]
        \caption{Model Selection}
        \label{algo: model selection}
        \Input{samples $\{x_j\}_{j=1}^n$, measurement points $\{t_i\}_{i=1}^L$, translation vectors $\{v_m\}_{m=0}^M$, common covariance matrix $\Sigma$, screening threshold $\varepsilon_*$.}
        \For{$m = 0,1, \cdots, M$}{
            \For{$l = 1, \cdots, L$}{
                $[\hat{\vy}_m]_l \gets
                \exp\!\left(\frac{(t_l+v_m)^\mathrm{T}\Sigma(t_l+v_m)}{2}\right)
                \frac{1}{n}\sum_{j=1}^n
                \exp\!\left(\iota\innerproduct{x_j}{t_l+v_m}\right)$\;
            }
        }
        $\hat{C} \gets \frac{1}{M+1} \sum_{m=0}^M \hat{\vy}_m \hat{\vy}_m^*$ \;
        $\hat{\sigma}_1 \geq \hat{\sigma}_2 \geq \cdots \geq \hat{\sigma}_L \gets$ singular values of $\hat C$\;
        $\widehat{\mathcal I}_{\varepsilon_*}
        \gets \{1\le l\le L-1:\hat\sigma_l>\varepsilon_*\}$\;
        \Output{model order $\displaystyle
        \hat k \gets \argmax_{l\in\widehat{\mathcal I}_{\varepsilon_*}}
        \frac{\hat\sigma_l}{\hat\sigma_{l+1}\vee\varepsilon_*}$.}
    \end{algorithm}

\paragraph{Computational complexity.}
Computing the $L(M+1)$ empirical Fourier measurements in Algorithm
\ref{algo: model selection} requires
\[
    \mathcal O\bigl(ndL(M+1)\bigr)
\]
arithmetic operations once the evaluation points and covariance-compensation
factors have been precomputed. Forming $\hat C$ and computing its
eigendecomposition require, respectively,
\[
    \mathcal O\bigl(L^2(M+1)\bigr)
    \qquad\text{and}\qquad
    \mathcal O(L^3)
\]
additional operations. Thus, the overall time complexity after precomputation
is
\[
    \mathcal O\bigl(ndL(M+1)+L^2(M+1)+L^3\bigr).
\]
When $L=\mathcal O(k)$ and $M+1=\mathcal O(k)$, this becomes
\[
    \mathcal O(ndk^2+k^3).
\]
For a general dense covariance matrix $\Sigma$, precomputing the quadratic
forms $(t_l+v_m)^\mathrm T\Sigma(t_l+v_m)$ costs
$\mathcal O\bigl(d^2L(M+1)\bigr)$; this reduces to
$\mathcal O\bigl(dL(M+1)\bigr)$ when $\Sigma$ is diagonal or spherical. This
precomputation is independent of the sample size $n$.

\section{Mean Estimation}
\label{sec: mean estimation}

In this section, we estimate the component means of the mixture model \eqref{eqn:model setting}. We assume the model order $k$ is known. Our proposed methodology is inspired by the MUltiple SIgnal Classification (MUSIC) algorithm in array signal processing and line spectral estimation, which relies on the geometric properties of spectral subspaces. 
Specifically, by exploiting the algebraic structure of the latent Fourier covariance matrix \(C\) (see \eqref{eq-C}), we first estimate the \textbf{spectral subspace} spanned by the Fourier steering vectors associated with the true component means of the GMM. Subsequently, we localize these component means by minimizing the orthogonal projection distance between parameterized Fourier steering vectors and the estimated spectral subspace. 

The resulting continuous optimization landscape is highly non-convex. We therefore introduce a data-driven, score-based initialization scheme, which selects candidate initializations directly from the observed samples according to the squared $\ell_2$-norm of their Fourier steering vectors projected onto the estimated spectral subspace. Under an explicit sample-size condition, a qualifying candidate exists with high probability and lies in a certified attraction basin, after which gradient descent converges linearly. In the case $k<d$, we first apply uncentered principal component analysis (PCA) to reduce the effective dimension.

\subsection{Case 1: Learning GMMs When $k \ge d$}
In this section, we present our algorithm for the case $k \ge d$. Recall that in the noiseless case, the Fourier steering vectors $\mathbf a_L(\mu_1), \mathbf a_L(\mu_2), \cdots, \mathbf a_L(\mu_k)$ span the \textbf{spectral subspace} $U_1$. Therefore, 
    \[
       \mathcal{P}_{U_1} (\mathbf a_L(\mu)) = \mathbf a_L(\mu) \,\,\mbox{for}\,\, \mu=\mu_1, \mu_2, \cdots, \mu_k. 
    \]
Here $\mathcal{P}_{U_1}(\cdot)$ is the projection operator which projects the parameterized Fourier steering vector onto the spectral subspace $U_1$. In the presence of noise due to the limited sample size, the spectral subspace $U_1$ is estimated by $\hat{U}_1$ (see \eqref{eq-hatC}), which is obtained from the covariance-compensated Fourier measurement defined in \eqref{eqn:modulated Fourier}. 
Unlike the classical MUSIC algorithm, which relies on an exhaustive grid search to identify steering vectors that are orthogonal to the noise subspace (the orthogonal complement of the spectral subspace), we adopt a continuous optimization strategy. This approach circumvents the prohibitive computational complexity associated with multi-dimensional grid scanning when \(d > 1\). Accordingly, the component means are estimated by identifying the local minimizers for the objective function defined via the following \textbf{orthogonal complement subspace projector}:
    \[
        \hat J(\mu) = \norm{\mathcal P_{\hat U_1} \mathbf a_L(\mu) - \mathbf a_L(\mu)}_2, \ \text{where}\ \mu \in \R^d.
    \]
We denote it population counterpart
 \[
        J(\mu) = \norm{\mathcal P_{ U_1} \mathbf a_L(\mu) - \mathbf a_L(\mu)}_2.
    \]
It is obvious that $J(\mu^*) := \norm{\mathcal P_{U_1} \mathbf a_L(\mu^*) - \mathbf a_L(\mu^*)}_2 = 0$ for any $\mu^* \in \{\mu_1,\cdots,\mu_k\}$. We have the following estimate on the discrepancy between the empirical objective $\hat J$ and its population counterpart $J$. 

\begin{proposition}
\label{prop:imaging-stability}
Adopt the setting of Proposition~\ref{thm: svd}: assume $M+1,L\ge k$ and
$\mathrm{rank}(\mathcal A)=\mathrm{rank}(W)=k$, so that the latent Fourier
covariance matrix $C=\mathcal A W\mathcal A^*\in\mathbb C^{L\times L}$ (see \eqref{eq-C}) is of rank $k$, with smallest nonzero singular value
$\sigma_k := \sigma_k(C) > 0$.
Then, on the event $\{\|\hat E\|_2 < \sigma_k\}$,
\begin{equation}
\label{eq:deterministic-imaging}
  \sup_{\mu\in\mathbb R^d} \bigl|\hat J(\mu) - J(\mu)\bigr|
  \;\le\;  \sqrt{L} \,\|\mathcal P_{ U_2} - \mathcal P_{ \hat{U}_2}\|_2 \;\le \; \frac{2\sqrt L\,\|\hat E\|_2}{\sigma_k - \|\hat E\|_2}.
\end{equation}

\end{proposition}

Next, we establish the local coercivity of the population objective $J$ under
suitable conditions.

\begin{proposition}
\label{prop:local coercivity condition}
Suppose $L\ge N:=d+k$. Then there exists a closed set $Z\subset\mathbb R^{dN}$,
of Lebesgue measure zero and with empty interior, with the following property:
whenever the frequency-vector set $\{t_1,\dots,t_L\}$ contains $N$ vectors
$t_{i_1},\dots,t_{i_N}$ with $(t_{i_1},\dots,t_{i_N})\notin Z$, then for each
$\mu^\ast\in\{\mu_1,\dots,\mu_k\}$ there exist a constant $\kappa>0$ and a
neighborhood $\mathcal N$ of $\mu^\ast$, both depending on the true means
$\{\mu_j:1\le j\le k\}$ and the frequency vectors $\{t_j:1\le j\le L\}$, such
that
\[
    J(\mu)\;\ge\;\kappa\,\norm{\mu-\mu^\ast}_2
    \qquad\text{for all }\mu\in\mathcal N .
\]
\end{proposition}

We are ready to state our main result on the local convergence of the estimator
obtained from a local minimizer of the empirical objective $\hat J$.

\begin{theorem}\label{cor:param_error}
Adopt the setting of Proposition~\ref{prop:imaging-stability}. Assume in
addition that $L\ge d+k$ and that the sampled frequency vectors $t_j$, $1\leq j\leq L$, satisfy the
condition of Proposition~\ref{prop:local coercivity condition}. Fix a true
component mean $\mu^*\in\{\mu_1,\dots,\mu_k\}$, and let $\kappa>0$ and
$\mathcal N$ be such that
\[
    J(\mu)\ge\kappa\|\mu-\mu^*\|_2,
    \qquad
    \mu\in\mathcal N.
\]
Choose $r_0>0$ such that
$\overline{B}(\mu^*,r_0)\subset\mathcal N$, and set
\[
    Q:=L(M+1),
    \qquad
    \kappa_\Sigma
    :=
    \exp\!\bigl(T_{\max}^2\sigma_{\max}(\Sigma)\bigr).
\]

Then the empirical objective $\hat J$ admits a local minimizer
$\hat\mu\in B(\mu^*,r_0)$ such that 
\begin{equation}\label{eq:param-rate}
    \|\hat\mu-\mu^*\|_2=O_p(n^{-1/2}).
\end{equation}
More precisely, for every $\delta\in(0,1)$, there exists $n_0<\infty$ such
that, for every $n\ge n_0$, with probability at least $1-\delta$,
\begin{equation}
    \|\hat\mu-\mu^*\|_2
    \le
    \frac{32L^{3/2}\sqrt{\kappa_\Sigma}}
         {\kappa\sigma_k(C)}
    \sqrt{\frac{\ln(4Q/\delta)}{n}}
    \nonumber
    +
    \frac{32L^{3/2}\kappa_\Sigma}
         {\kappa\sigma_k(C)}
    \frac{\ln(4Q/\delta)}{n}.
    \label{eq:param-highprob}
\end{equation}

One sufficient choice of $n_0$ is obtained by defining
\begin{equation}\label{eq:param-tau0}
    \tau_0
    :=
    \min\left\{
        \frac{\sigma_k(C)}{2},
        \frac{\kappa\sigma_k(C)r_0}{16\sqrt L}
    \right\}, 
\end{equation}
and then setting
\begin{equation}\label{eq:n0-explicit}
    n_0
    :=
    \max\left\{
        \frac{64L^2\kappa_\Sigma\ln(4Q/\delta)}{\tau_0^2},
        \frac{8L\kappa_\Sigma\ln(4Q/\delta)}{\tau_0}
    \right\}.
\end{equation}

\end{theorem}

\begin{remark}
Previous work in \cite{doss2023optimal} establishes that the optimal global minimax rate for estimating the mixing distribution of GMMs is $O(n^{-1/4})$ (or slower depending on $d$ and $k$). This slow rate arises from the "worst-case" regime where component means merge ($\Delta \to 0$), causing the Fisher information matrix to become singular. In contrast, our result in Theorem \ref{cor:param_error} establishes a convergence rate of $O_p(n^{-1/2})$. This is consistent with the standard theory of strongly identifiable parametric models (see \cite{chen1995optimal}). Our bound depends on the separation distance $\Delta$ through $\sigma_k(C)$, which appears in the denominator. Under the sampling design of Proposition~\ref{prop:lower-bound-sampling}, the lower bound on $\sigma_k(C)$ makes this dependence explicit; in particular, when $\tau=\rho$, we have $\sigma_k(C)\gtrsim\Delta^{4k-4}$. Thus, the constant in the estimation error bound deteriorates as $\Delta\to0$, reflecting the degeneracy that occurs when component means merge.
\end{remark}


Note that the convergence result in Theorem~\ref{cor:param_error} is local:
in practice, a good initialization is needed to guarantee convergence for the
nonconvex objective $\hat J(\mu)$. At the same time, evaluating $\hat J(\mu)$ on
a fine grid becomes costly in dimensions $d\ge 2$.

In our algorithm, we use gradient descent to locate a local minimizer. For
computational convenience, we minimize the \textbf{squared empirical objective}
\[
    f(\mu):=\hat J^{2}(\mu)
\]
instead of $\hat J(\mu)$ itself, via the iteration
\begin{equation}
    \label{eqn:gradient descent}
    \mu^{(i+1)} = \mu^{(i)} - \gamma\,\nabla_\mu f\bigl(\mu^{(i)}\bigr),
    \qquad i=0,1,\dots,
\end{equation}
where $\gamma>0$ is the learning rate. The gradient above has the following explicit form.  

    \begin{lemma}
        \label{prop:gradient}
        Let $r_i$ denote the $i$-th row vector of $\hat{U}_1$ for $i = 1,\ldots,L$. Then 
        \[
            \nabla_\mu f(\mu) = -\iota \sum_{m=1}^L \sum_{l=1}^L (r_m r_l^* e^{\iota\innerproduct{\mu}{t_l - t_m}} (t_l - t_m).
        \]
     \end{lemma}

\subsection{The score function for initialization and convergence guarantee}

The successful convergence of the iteration~\eqref{eqn:gradient descent} to one
of the $k$ true Gaussian means depends on the choice of initialization
$\mu^{(0)}$. We draw candidate initializations directly from the observed
samples. The probability of obtaining a valid initialization is governed by the
Gaussian mass of a certified basin of attraction around a component mean; this
small-ball probability can deteriorate with the effective dimension, as
quantified in the discussion following Lemma~\ref{lem:score_hit}.

Specifically, we score each sample \(x_i\) via the \textbf{score 
function}
\[
    s(x_i) = \norm{\mathcal{P}_{\hat{U}_1} \mathbf a_L(x_i)}_2^2 
    = \mathbf a_L(x_i)^* \hat{U}_1 \hat{U}_1^* \mathbf a_L(x_i),
\]
where \(\mathcal{P}_{\hat{U}_1}\) denotes the orthogonal projection onto the estimated spectral 
subspace spanned by \(\hat{U}_1\). Under the normalization $\norm{\mathbf a_L(x)}_2 = \sqrt{L}$,
\[
    s(x_i) = L - \hat J(x_i)^2=L-f(x_i), 
    \qquad \hat J(x_i) := \norm{\mathcal{P}_{\hat U_2}\mathbf a_L(x_i)}_2.
\]
Thus, maximizing the score is exactly equivalent to minimizing the residual $\hat J(x_i)$.

In what follows, we first show that any candidate sample in $\Theta$ satisfying
the stated score threshold lies in a certified attraction basin. We then derive
an explicit sample-size condition under which such a sample exists with high
probability and prove that gradient descent initialized at this sample converges
linearly to a local minimizer of $f$ that is $O_p(n^{-1/2})$-close to a true
mean.

Throughout this subsection, we work on the event of
Proposition~\ref{prop:imaging-stability} and
Proposition~\ref{prop:local coercivity condition}. Following Step~1 in the proof
of Proposition~\ref{prop:local coercivity condition}, let $\mu_j$ be the true
center nearest to the initialization $x_i$, i.e.
$j=\arg\min_{1\le l\le k}\norm{x_i-\mu_l}_2$, and let $$T_j=\mathcal{P}_{U_2}D \mathbf a_L(\mu_j)$$
where $D\mathbf a_L(\mu_j)$ denote the Jacobi matrix of $\mathbf a_L$ at $\mu_j$. We 
set
\[
    m_j := 2\sigma_{\min}(T_j)^2, \qquad M_j := 2\sigma_{\max}(T_j)^2, \qquad
    \kappa_j := \frac{M_j}{m_j}
             = \Bigl(\tfrac{\sigma_{\max}(T_j)}{\sigma_{\min}(T_j)}\Bigr)^2 .
\]

We have the following regularity for the empirical objective function $f$.

\begin{lemma}[Local strong convexity and smoothness]\label{lem:local_regularity}
Fix a true mean $\mu_j$, and let $\mathcal{N}_j$ be its coercivity
neighborhood from Proposition~\ref{prop:local coercivity condition}. There
exists $\rho_{0,j}>0$, depending only on the true means
$\mu_1,\dots,\mu_k$, and the frequency vectors $t_l$, $1\leq l\leq L$, and their translations $v_m$, $1\leq m \leq M$,
such that $\overline{B}(\mu_j,\rho_{0,j})\subset\mathcal{N}_j$. Moreover, for every
$\delta\in(0,1)$, there exists $n_0>0$, depending additionally on $\delta$ and
the covariance matrix $\Sigma$, such that, whenever $n\ge n_0$, with
probability at least $1-\delta$,
\[
    \frac{1}{2}m_jI_d
    \;\preceq\; \nabla^2 f(\mu) \;\preceq\;
    2M_jI_d,
    \qquad \mu\in B(\mu_j,\rho_{0,j}).
\]
On the same event, $f$ admits a unique minimizer
$\hat\mu_j\in B(\mu_j,\rho_{0,j})$. Moreover,
\[
    \norm{\hat\mu_j-\mu_j}_2=O_p(n^{-1/2}).
\]
\end{lemma}
Recall the signal subspace
$U_1=\operatorname{span}\{\mathbf a_L(\mu_1),\dots,\mathbf a_L(\mu_k)\}$, and the population objective functional
\[
  J(\mu)=\bigl\|\mathbf a_L(\mu)-\mathcal P_{U_1}\mathbf a_L(\mu)\bigr\|_2
        =\bigl\|\mathcal P_{U_1^{\perp}}\mathbf a_L(\mu)\bigr\|_2 .
\]
We have the following result on the global coercivity of $J$. 
\begin{lemma}
\label{lem:global_coercivity}
Suppose the true means $\mu_1,\dots,\mu_k$ and the parameter search lie in a
known closed ball $\Theta:=\bar B(\bar\mu, R/2)$ centered at some
$\bar\mu\in\mathbb R^d$, with $R>0$, and that $L\ge k+1$ with the frequency
vectors $t_l$, $1\le l\le L$, chosen so that the steering vectors at any $k+1$
distinct parameters in $\Theta$ are linearly independent; in particular,
$\mathbf a_L(\mu)\in U_1$ implies $\mu\in\{\mu_1,\dots,\mu_k\}$. Then, with
componentwise radii $\rho_{0,l}$ as in
Lemma~\ref{lem:local_regularity}, set
\[
    \rho_0:=\min_{1\le l\le k}\rho_{0,l},
    \qquad
    \tilde{\mathcal N}_l:=B(\mu_l,\rho_0),
\]
and define
\[
    K:=\Theta\setminus\bigcup_{l=1}^{k}\tilde{\mathcal N}_l .
\]
Define
\[
  J_\star:=\tfrac12\min_{\mu\in K}J(\mu).
\]
Then $J_\star$ is strictly
positive and
\[
  \{\,\mu\in\Theta : J(\mu)\le J_\star\,\}
  \ \subseteq\
  \bigcup_{l=1}^{k}\tilde{\mathcal N}_l .
\]
\end{lemma}

We now provide a concrete instance satisfying the conditions of the lemma
above. Let $\Theta=\bar B(\bar\mu,R/2)\subset\mathbb R^d$ be the closed ball of
radius $R/2$ centered at $\bar\mu$, and suppose $\mu_1,\dots,\mu_k\in\Theta$. Let
the frequency vectors $t_l$, $1\le l\le L$, be constructed as in
Proposition~\ref{prop:lower-bound-sampling}. Fix any $\mu\in\Theta$ with
$\mu\ne\mu_j$ for all $1\le j\le k$, and apply
Proposition~\ref{prop:lower-bound-sampling} to the collection
$\{\mu,\mu_1,\dots,\mu_k\}$: the $(k+1)$-th singular value of
\[
    \bigl(\mathbf a_L(\mu),\,\mathbf a_L(\mu_1),\,\dots,\,\mathbf a_L(\mu_k)\bigr)
\]
then obeys a lower bound analogous to~\eqref{eq:sigmak-Phi}. In particular this
singular value is strictly positive, so $\mathbf a_L(\mu),\mathbf a_L(\mu_1),\dots,
\mathbf a_L(\mu_k)$ are linearly independent and the required full-rank condition
holds.

With these preparations, we are ready to state our main result on the
localization property of the score-based initialization.

\begin{theorem}[Score-based localization]\label{thm:basin}
Suppose that the conditions of Lemmas~\ref{lem:local_regularity}
and~\ref{lem:global_coercivity} hold. Write
$\sigma_{\min}:=\min_{1\le j\le k}\sigma_{\min}(T_j)$, and let
$\Delta:=\min_{l\ne l'}\norm{\mu_l-\mu_{l'}}_2$. Define
\[
   \rho_1\;:=\;\min\Bigl\{\,\tfrac{\rho_0}{2},\
             \tfrac{\Delta}{3},\ \tfrac{2J_\star}{\sigma_{\min}}\,\Bigr\},
   \qquad
   \mathcal B_l\;:=\;\overline B(\mu_l,\rho_1)
       \ \subset\ \tilde{\mathcal N}_l .
\]
Fix $\tau\in(0,\tfrac12\sigma_{\min}\rho_1)$ and $\delta\in(0,1)$. Let $n_{\rm reg}(\delta/2)$ be the regularity threshold from Lemma~\ref{lem:local_regularity}, applied with $\delta$ replaced by $\delta/2$. Then there exists
\[
   n_0
   =
   n_{\rm reg}(\delta/2)\ \vee\
   O\!\left(
      \frac{L^3\kappa_\Sigma
             \ln\!\bigl(L(M+1)/\delta\bigr)}
           {\sigma_k(C)^2
             \bigl(\tfrac12\sigma_{\min}\rho_1-\tau\bigr)^2}
   \right),
\]
such that, for all $n\ge n_0$, with
probability at least $1-\delta$,
\begin{equation}\label{eqn:score_threshold}
   x_i\in\Theta,\quad s(x_i)\ge L-\tau^2
\end{equation}
implies that there exists one and only one $j$ such that $\ x_i\in\mathcal B_j\subset\tilde{\mathcal N}_j$.
In particular, $\|x_i-\mu_j\|_2\le\rho_1$, and $f=\hat J^2$ is
$\tfrac12m_j$-strongly convex and $2M_j$-smooth on $\tilde{\mathcal N}_j$.
\end{theorem}

\medskip
\noindent\emph{Practical interpretation.}
The threshold in Theorem~\ref{thm:basin} is a population-level certificate:
it depends on latent quantities such as $\sigma_{\min}$, $J_\star$, and
$\Delta$. In applications one may rank the observed scores without knowing
these quantities, but the theorem applies only to candidates whose scores
cross the stated threshold and that lie in the prescribed search set
$\Theta$.

The next lemma makes explicit the sample requirement for obtaining a sample that passes the score threshold.

\begin{lemma}[Existence of a sample passing the score threshold]\label{lem:score_hit}
Let $x_1,\dots,x_n$ be i.i.d.\ from the mixture model \eqref{eqn:model setting}. Let 
$\mu^*$ be the target mean with weight $w^*$, and let $\tau>0$ be a constant. Define
\[
    L_a:=\left(\sum_{l=1}^L\|t_l\|_2^2\right)^{1/2}>0,
    \qquad
    r_{\rm cert}:=\frac{\tau}{2L_a}, 
\]
where $t_l$, $1\leq l \leq L$, are the frequency vectors. On the event
\[
    \|\hat J-J\|_\infty\le\frac{\tau}{2},
\]
every $x\in B(\mu^*,r_{\rm cert})$ satisfies the score threshold
\eqref{eqn:score_threshold}. Moreover, for any 
$\delta\in(0,1)$, whenever
\begin{equation}\label{eqn:hit_n}
    n \;\ge\; \frac{\log(1/\delta)}{w^*\,p_{r_{\rm cert}}},
\end{equation}
where $p_{r_{\rm cert}}:=\Pr_{x\sim N(\mu^*,\Sigma)}\!\bigl(\norm{x-\mu^*}_2\le r_{\rm cert}\bigr)$, 
some $x_i$ lies in $B(\mu^*,r_{\rm cert})$ with probability at least $1-\delta$.
On the intersection with the event $\|\hat J-J\|_\infty\le\tau/2$, this sample satisfies the condition \eqref{eqn:score_threshold}.
\end{lemma}

\medskip
\noindent\emph{Dimension dependence of sample-based initialization.}
For a spherical covariance $\Sigma=\sigma^2I_q$,
\[
    p_{r_{\rm cert}}
    =\Pr\!\left(\chi_q^2\le\frac{r_{\rm cert}^2}{\sigma^2}\right),
\]
where $q=d$ in the case $k\ge d$ and $q=k$ after the PCA reduction in the case
$k<d$. If
\[
    a:=\frac{r_{\rm cert}^2}{q\sigma^2}\in(0,1),
\]
then the lower-tail bound for a chi-square random variable gives
\[
    p_{r_{\rm cert}}
    \le
    \exp\!\left\{-\frac q2\bigl(a-1-\log a\bigr)\right\}.
\]
Consequently, the sample size in \eqref{eqn:hit_n} may grow exponentially with
the effective dimension; for fixed $r_{\rm cert}/\sigma$, its reciprocal grows
as $\exp\{\Theta(q\log q)\}$. This dimension dependence concerns the cost of
finding a certified initialization and does not affect the conditional linear
convergence rate once such an initialization is available.

To state the end-to-end requirement, set
\[
    \varepsilon_{\rm cert}
    :=\min\left\{\frac{\tau}{2},
       \frac12\sigma_{\min}\rho_1-\tau\right\}>0,
\]
and let
\[
    n_{\rm cert}(\delta)
    :=n_{\rm reg}(\delta/2)\ \vee\
    O\!\left(
       \frac{L^3\kappa_\Sigma
              \ln\!\bigl(L(M+1)/\delta\bigr)}
            {\sigma_k(C)^2\varepsilon_{\rm cert}^2}
    \right).
\]
By the covariance-concentration and imaging-stability bounds used in
Theorem~\ref{thm:basin}, $n\ge n_{\rm cert}(\delta)$ ensures, with probability
at least $1-\delta$, both the simultaneous regularity event and
$\|\hat J-J\|_\infty\le\varepsilon_{\rm cert}$.
Combining Theorem~\ref{thm:basin} with the preceding lemma and allocating
failure probability $\delta/2$ to each event, the sufficient end-to-end sample
condition is
\begin{equation}\label{eqn:end-to-end-initialization}
    n\ge
    n_{\rm cert}(\delta/2)
    \ \vee\ 
    \frac{\log(2/\delta)}{w^*p_{r_{\rm cert}}}.
\end{equation}
Under \eqref{eqn:end-to-end-initialization}, with probability at least
$1-\delta$, a qualifying sample initialization exists, lies in a certified
attraction basin, and satisfies the hypotheses of Theorem~\ref{thm:gd_convergence}.

\medskip

Let $x_i$ be a sample from the mixture model \eqref{eqn:model setting}. For a score-selected initialization $\mu^{(0)}=x_i$, consider the iteration
\[
    \mu^{(t+1)}=\mu^{(t)}-\gamma\nabla f(\mu^{(t)}),
    \qquad f=\hat J^2.
\]
We have the following main result on the convergence of score-based initialization. 

\begin{theorem}[Linear convergence from a score-selected initialization]
\label{thm:gd_convergence}
Let $\Theta=\bar B(\bar\mu, R/2)$ be the search domain satisfying the condition
in Lemma~\ref{lem:global_coercivity}, and let $\tau$ be as defined in Theorem \ref{thm:basin}. Under the setting of
Theorem~\ref{thm:basin}, let $x_i\in\Theta$ satisfy $s(x_i)\ge L-\tau^2$, and
let $j$ be the unique index identified by~\eqref{eqn:score_threshold}. Let $\hat\mu_j$ be the minimizer
from Lemma~\ref{lem:local_regularity}, and further assume that
\begin{equation}\label{eqn:gd_interior}
    \rho_1+2\norm{\hat\mu_j-\mu_j}_2<\rho_0 .
\end{equation}
If $0<\gamma\le 1/(2M_j)$, then every iterate remains in $\tilde{\mathcal N}_j$ (defined as in Lemma \ref{lem:global_coercivity}) and
\begin{equation}\label{eqn:linear_rate}
    \norm{\mu^{(t)}-\hat\mu_j}_2
    \le
    \Bigl(1-\tfrac12\gamma m_j\Bigr)^t
    \norm{\mu^{(0)}-\hat\mu_j}_2 .
\end{equation}
\end{theorem}

The procedure for estimating the means is detailed in Algorithm \ref{algo: mean estimation 1}. The algorithm runs gradient descent from several starting points, which may initially appear costly. However, the starting points are selected so that only a few iterations are needed for convergence. To illustrate this, Figure \ref{fig:illu algo} presents a two-dimensional example. We draw $500$ samples from the mixture model and run gradient descent from the $25$ samples with the largest scores. After only $5$ iterations with learning rate $0.5$, these points converge to neighborhoods of the three component centers.
\begin{figure}[!htbp]
    \centering
    \includegraphics[width=1.0\linewidth]{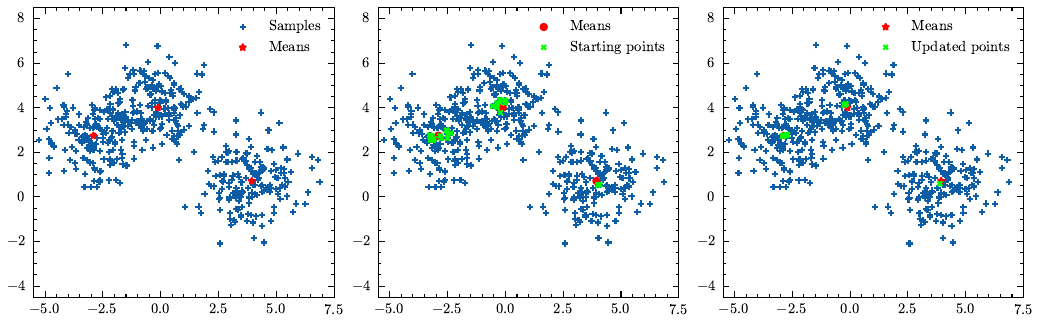}
    \caption{Illustration of Algorithm \ref{algo: mean estimation 1}. Left: $500$ samples (blue) drawn from a three-component mixture model with centers $(3.94,0.72)$, $(-0.12,4.00)$, and $(-2.91,2.75)$ (red). Middle: the $25$ starting points (green) with the largest scores $s(x)$. Right: the updated points after $5$ iterations of \eqref{eqn:gradient descent} with $\gamma=0.5$ (green).}
    \label{fig:illu algo}
\end{figure}
\FloatBarrier

Although the primary focus of this section is mean estimation, the full mixing distribution can be estimated once the component means are available. Given the mean estimates $\{\hat\mu_i\}_{i=1}^k$, we estimate the mixing weights by solving a constrained quadratic program that minimizes the squared discrepancy between the covariance-compensated empirical characteristic function \eqref{eqn:modulated Fourier} and $\sum_{i=1}^k w_i \exp(\iota \innerproduct{\hat \mu_i}{t})$. The quadratic program is as follows:
    \begin{align} \label{eqn:weight solver}
            &\text{minimize } \sum_{l=1}^L \sum_{m=0}^M \left|\sum_{i=1}^k w_i e^{\iota\innerproduct{\hat \mu_i}{t_l + v_m}} - \hat y_n(t_l + v_m)\right|^2, \nonumber \\
            &\text{subject to } w_i \geq 0, \quad \sum_{i=1}^{k} w_i = 1. 
    \end{align}

Here $r_{\mathrm{dup}}\geq0$ denotes a user-specified de-duplication radius: a newly obtained local minimizer is retained only if it lies more than $r_{\mathrm{dup}}$ from every previously retained estimate.

    \begin{algorithm}
    \label{algo: mean estimation 1}
        \caption{Mean estimation for GMMs $(k\geq d)$}
        \Input{samples $\{x_j\}_{j=1}^n$, model order $k$, measurement points $\{t_i\}_{i=1}^L$, translation vectors $\{v_m\}_{m=0}^M$, covariance matrix $\Sigma$, de-duplication radius $r_{\mathrm{dup}}$.}
        \For{$m = 0,1, \cdots, M$}{
            \For{$l = 1, \cdots, L$}{
                $[\hat{\vy}_m]_l \gets
                \exp\!\left(\frac{(t_l+v_m)^\mathrm{T}\Sigma(t_l+v_m)}{2}\right)
                \frac{1}{n}\sum_{j=1}^n
                \exp\!\left(\iota\innerproduct{x_j}{t_l+v_m}\right)$\;
            }
        }
        $\hat{C} \gets \frac{1}{M+1} \sum_{m=0}^M \hat{\vy}_m \hat{\vy}_m^*$ \;
        $\hat{U}_1 \gets $ matrix formed by the first $k$ left singular vectors of $\hat C$\;
        $\{x_{s_j}\}_{j=1}^n \gets$ samples sorted in decreasing order of $s(x_j)$\;
        $j \gets 0,\  \mathcal{S}\gets \{\}$\;
        \While{$\lvert\mathcal S\rvert<k$}
            {
            $j \gets j+1$\;
            
            $\hat{\mu} \gets $ scheme \eqref{eqn:gradient descent} with $\mu^{(0)} = x_{s_j}$\;
            \If{$\norm{\hat{\mu} - \tilde\mu}_2 > r_{\mathrm{dup}}$ \normalfont{for all} $\tilde\mu \in \mathcal{S}$}
                {
                $\mathcal{S} \gets \mathcal{S} \cup \{\hat{\mu}\}$
                }
            }
        \Output{A set of estimated centers $\mathcal{S}$.}
    \end{algorithm}

\clearpage
\subsection{Case 2: Learning GMMs When $k < d$}
Throughout this subsection, we consider a common spherical covariance
$\Sigma=\sigma^2I_d$, where $\sigma>0$. For a known general common covariance
$\Sigma_0\succ0$, we first whiten the observations according to
\[
    z_j=\Sigma_0^{-1/2}x_j,
    \qquad
    z_j\sim\sum_{i=1}^k w_i
    \mathcal N\!\left(\Sigma_0^{-1/2}\mu_i,I_d\right).
\]
The procedure below is then applied to $\{z_j\}_{j=1}^n$, and an estimated
mean $\hat\mu_i^{\mathrm w}$ in the whitened coordinates is mapped back to the
original coordinates by $\hat\mu_i=\Sigma_0^{1/2}\hat\mu_i^{\mathrm w}$. Thus, the spherical
assumption in this subsection concerns the PCA reduction and does not restrict
the general common-covariance formulation considered elsewhere in the paper.

When $k < d$, the centers $\mu_1, \mu_2, \cdots, \mu_k$ span a subspace $V$ in $\R^d$ of dimension at most $k$. We can then estimate the $\{\mu_i\}_{i=1}^k$ by first estimating their projection onto this subspace. By projecting the samples onto the estimated subspace $\hat V$, we can estimate the projection of centers by Algorithm \ref{algo: mean estimation 1}. Throughout this subsection, PCA refers to uncentered PCA: no sample mean is subtracted, and the SVD is applied directly to the data matrix
    \[
        X = \begin{bmatrix}
            x_1 & x_2 & \cdots & x_n
        \end{bmatrix}^\mathrm{T} \in \R^{n\times d}.
    \]
The estimation is based on the following proposition for the uncentered second-moment matrix:
\begin{proposition}
\label{prop: expectation cov}
    If $\Sigma = \sigma^2I_d$, then
    $$\E \left[\frac{1}{n}X^\mathrm{T}X\right] = \sum_{i=1}^k w_i \mu_i \mu_i^T + \sigma^2 I_d.$$
\end{proposition}
We can estimate $V$ by $\hat V = \begin{bmatrix}
    v_1 & v_2 & \cdots & v_k
\end{bmatrix}$ where $v_1,\cdots,v_k$ are the first $k$ right singular vectors of $X$. The procedure of the estimation is detailed in Algorithm \ref{algo: mean estimation 2}.

    \begin{algorithm}
    \label{algo: mean estimation 2}
        \caption{Mean Estimation for GMMs $(k < d)$}
        \Input{samples $\{x_i\}_{i=1}^n$, model order $k$.}
        $v_1, \cdots, v_k \gets $ the first $k$ right singular vectors of $X = \begin{bmatrix}
            x_1 & x_2 & \cdots & x_n
        \end{bmatrix}^\mathrm{T}$\;
        $\hat V \gets \begin{bmatrix}
            v_1 & v_2 &\cdots & v_k
        \end{bmatrix}$ and $\bar{x}_j \gets \hat V^\mathrm{T} x_j$ for $j=1,\cdots, n$\;

        $\mathcal{S}\gets$ run Algorithm \ref{algo: mean estimation 1} with projected samples $\{\bar{x}_j\}_{j=1}^n$ and other proper inputs\;

        \Output{A set of estimated centers $\{\hat V\mu: \mu \in \mathcal{S}\}$}
    \end{algorithm}
After projection onto $\hat{V}$, the samples $\{\bar{x}_j\}_{j=1}^n$ follow a GMM with distribution $$\sum_{i=1}^{{k}} w_i \mathcal{N}\left(\hat V^\mathrm{T}\mu_i, \sigma^2 I_k\right).$$
We estimate the means of this $k$-dimensional GMM by Algorithm \ref{algo: mean estimation 1} and denote them by $\hat{\eta}_i$, $i=1,\ldots,k$. The corresponding $d$-dimensional estimates are $\hat V\hat{\eta}_i$. To quantify their accuracy, we first bound the error between $\mu_i$ and $\hat V\hat V^\mathrm{T}\mu_i$ introduced by dimension reduction:
\begin{proposition}
    \label{prop: pca error}
Suppose that $\Sigma=\sigma^2I_d$ with $\sigma>0$, and set$R:=\max_{1\le i\le k}\norm{\mu_i}_2.$
Let $\hat V\in\R^{d\times k}$ contain the first $k$ right singular vectors
of the uncentered data matrix $X$. Then there exists a universal constant
$C>0$ such that, for every $\delta\in(0,1)$, with probability at least
$1-\delta$,
\begin{align}
    \sum_{i=1}^k w_i
    \norm{\mu_i-\hat V\hat V^\mathrm T\mu_i}_2^2
    \le{}& Ck\frac{(R+\sigma)^2(R^2+\sigma^2)}{\sigma^2}
    \left(
        \sqrt{\frac{d+\ln(1/\delta)}{n}}
        +\frac{d+\ln(1/\delta)}{n}
    \right).
    \label{eqn:pca-error-bound}
\end{align}
\end{proposition}
Proposition \ref{prop: pca error} shows that dimension reduction incurs little loss of accuracy when the sample size is sufficiently large. The estimation error satisfies
    \[
        \norm{\mu_i - \hat V {\hat{\eta}_i}}_2 \leq \norm{\mu_i - \hat V \hat V^\mathrm{T}\mu_i}_2 + \norm{\hat{V}^\mathrm{T} \mu_i - {\hat{\eta}_i}}_2,
    \]
where the second term is the estimation error from Algorithm \ref{algo: mean estimation 1}.
\section{Numerical Experiments}
\label{sec: numericals}

\subsection{Model Selection under Random-Data and Dirichlet-Weight EM Initialization}
\label{sec: dirichlet init}

We compare Algorithm \ref{algo: model selection} against several baselines for the same model selection task, with every baseline that requires fitting a Gaussian mixture model by EM initialized by a random-data-point-and-Dirichlet-weight scheme, described below.

\paragraph{Experimental settings.} We test equally weighted $k$-component GMMs for $k = 2, 3, 4$ in $\R^{10}$. The means form an antipodal pair of separation $\Delta$ for $k = 2$, a regular triangle of side length $\Delta$ for $k = 3$, and a regular tetrahedron of edge length $\Delta$ for $k = 4$. Every component has covariance matrix $I_{10}$. For each $(\log(n), \Delta)$ pair, we perform $24$ repeated trials and record the empirical success rate and running time of every method. The value of $\log(n)$ ranges over $16$ equally spaced points from $3.00$ to $5.00$ (increment $0.1333$), and $\Delta$ ranges over $16$ equally spaced points from $2.0$ to $7.0$ (increment $0.3333$). In Algorithm \ref{algo: model selection}, we generate $ck$ measurement points uniformly from the ball of radius $0.5$ ($c = 5$ for $k = 2$ and $c = 3$ for $k = 3, 4$), take the translation vectors to be $\pm e_1, \ldots, \pm e_{10}$, and determine the model order using $\varepsilon_* = 10^{-3}$.

\paragraph{Methods and initialization.} We compare Algorithm \ref{algo: model selection} (\emph{Proposed}) against five baselines: the covariance eigengap estimator (\emph{Cov.\ eigengap}), the penalized-likelihood criteria AIC, BIC, and ICL, and the sequential likelihood ratio test with $\chi^2$ calibration (LRT). Algorithm \ref{algo: model selection} and the covariance eigengap estimator are spectral methods and require no mixture fitting. The remaining four baselines each fit a tied-covariance Gaussian mixture model by EM, independently, for every candidate order $k = 1, \ldots, k_{\max}$, and select $k$ by their respective criterion.

Every EM fit is seeded as follows. For each of $n_{\text{init}} = 5$ independent restarts:
\begin{itemize}
    \item the $k$ initial component means are $k$ data points drawn uniformly at random, without replacement, from the sample;
    \item the initial mixing weights $(\pi_1, \ldots, \pi_k)$ are drawn from a symmetric Dirichlet distribution, $(\pi_1, \ldots, \pi_k) \sim \mathrm{Dir}(1, \ldots, 1)$, i.e.\ uniformly over the probability simplex;
    \item the initial (shared, tied) covariance matrix is set to the empirical covariance of the whole sample.
\end{itemize}
Among the $5$ restarts, the fit with the highest converged log-likelihood is kept. Each restart is run to the same EM stopping rule: iterations stop once the per-sample log-likelihood improves by less than $10^{-5}$, or after $5000$ iterations, whichever occurs first.

\paragraph{Results.} Figure \ref{fig: dirichlet success rate} shows the empirical success rate of every method over the full $(\log(n), \Delta)$ grid, and Figure \ref{fig: dirichlet success mean} the same success rate averaged over $\Delta \in [2, 7]$ and plotted against $\log(n)$ alone, which makes the ordering of the methods easier to read off directly. Algorithm \ref{algo: model selection} attains a high success rate over most of the grid, comparable to or exceeding the penalized-likelihood and likelihood ratio baselines, and its $\Delta$-averaged success rate in Figure \ref{fig: dirichlet success mean} increases steadily with $\log(n)$ for every $k$. The covariance eigengap estimator tracks Algorithm \ref{algo: model selection} closely in both figures.

Figure \ref{fig: dirichlet runtime} shows the corresponding mean running time per trial versus $\log(n)$, averaged over $\Delta$. The running time of Algorithm \ref{algo: model selection} is essentially flat in $n$ and consistently smaller than that of AIC, BIC, ICL, and the LRT, whose running time is dominated by the cost of $n_{\text{init}} = 5$ independent EM fits for each of the $k_{\max}$ candidate orders and grows with $n$ as EM requires more iterations to converge. The covariance eigengap estimator is the fastest method overall, though, unlike Algorithm \ref{algo: model selection}, it relies on the population covariance matrix being of the spiked form induced by affinely independent component means, a property specific to the center configurations tested here.

\begin{figure}[!htbp]
    \centering
    \subfigure[$k=2$]{
    \resizebox{0.31\textwidth}{!}{
        \includegraphics[]{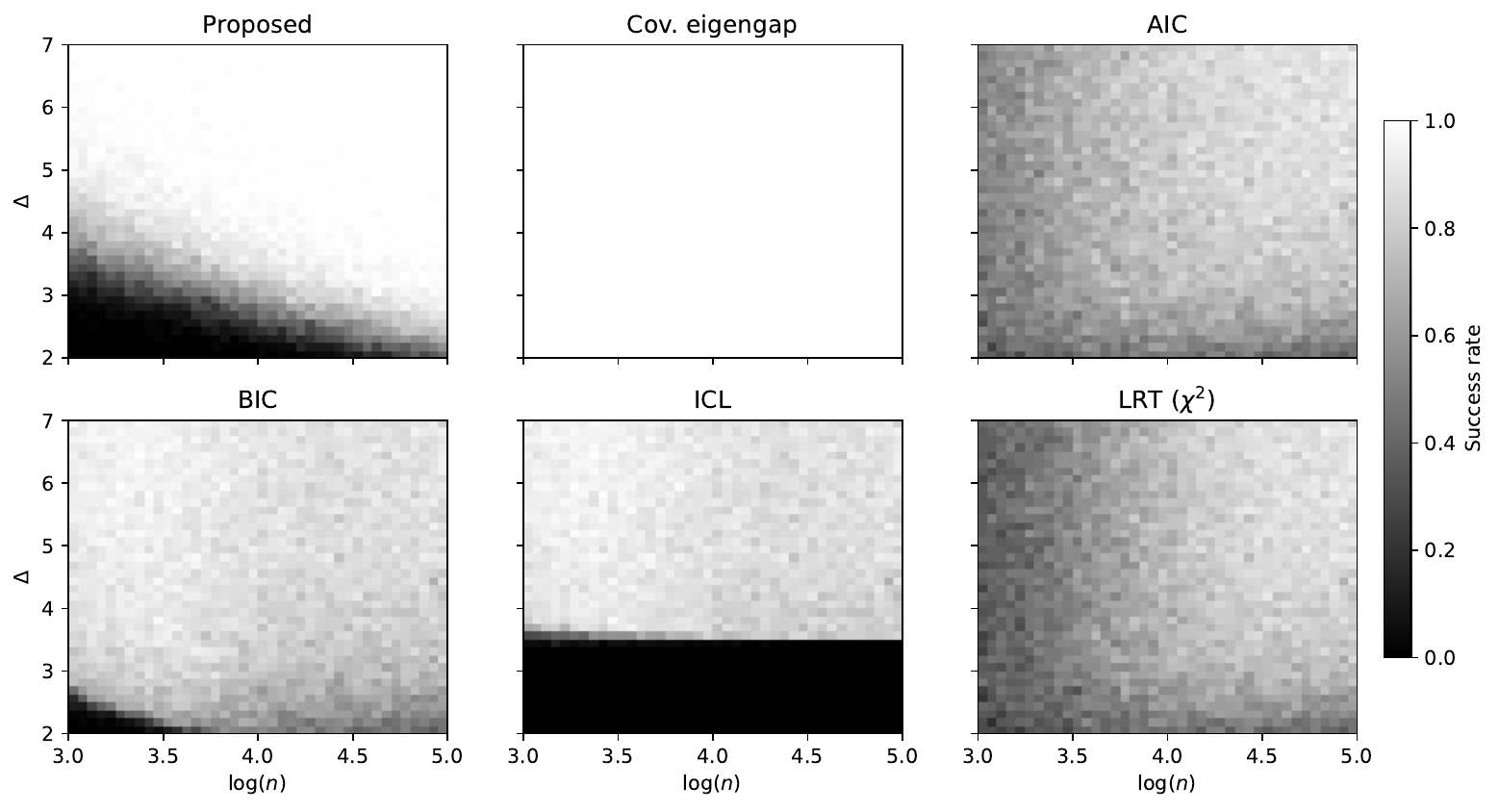}
    }
}
    \subfigure[$k=3$]{
    \resizebox{0.31\textwidth}{!}{
        \includegraphics[]{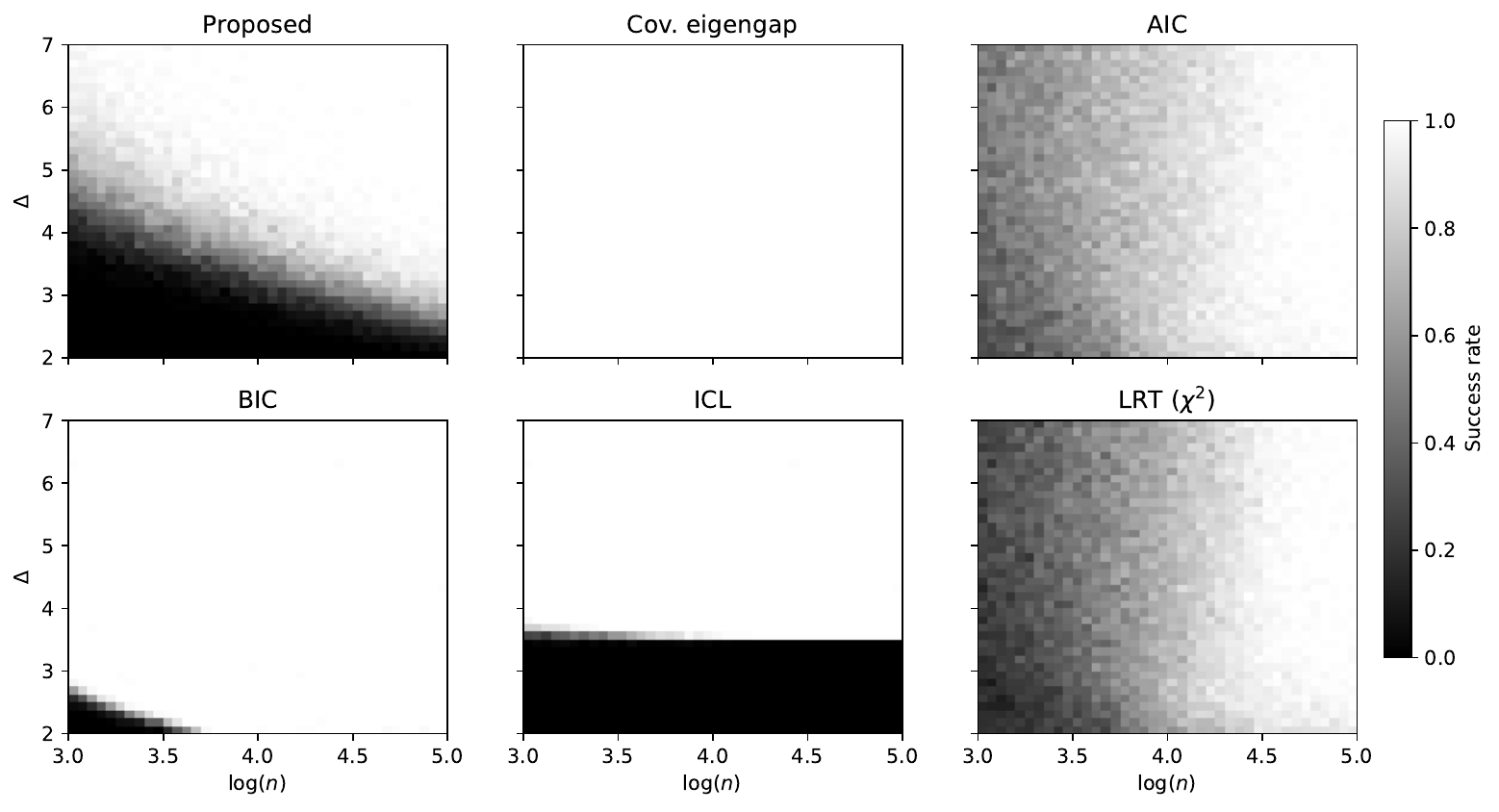}
    }
}
    \subfigure[$k=4$]{
    \resizebox{0.31\textwidth}{!}{
        \includegraphics[]{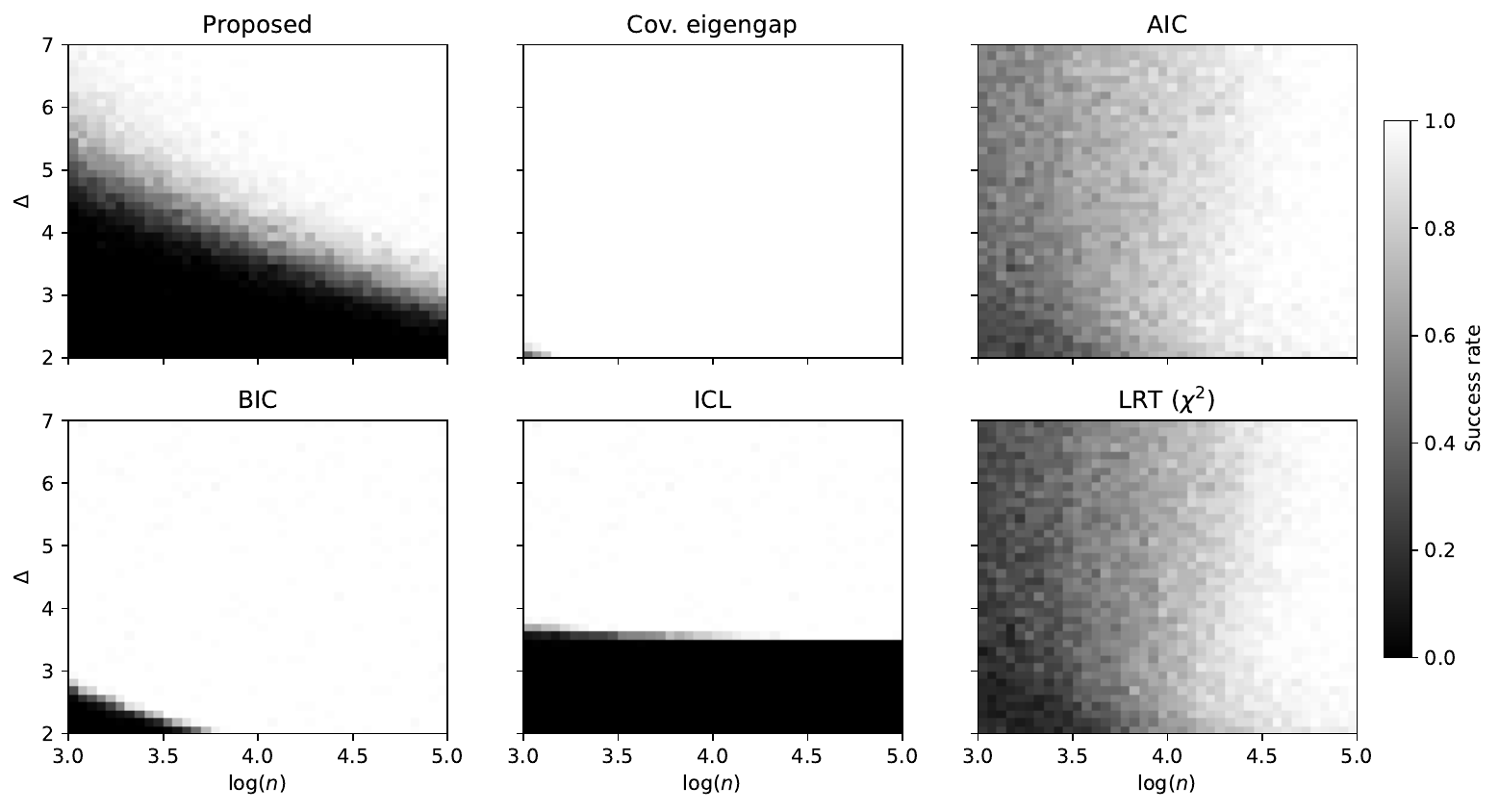}
    }
}
\caption{Empirical success rate of $24$ trials under each $(\log(n), \Delta)$ setting, for every method, under the random-data-point-and-Dirichlet-weight EM initialization.}
\label{fig: dirichlet success rate}
\end{figure}

\begin{figure}[!htbp]
    \centering
    \subfigure[$k=2$]{
    \resizebox{0.31\textwidth}{!}{
        \includegraphics[]{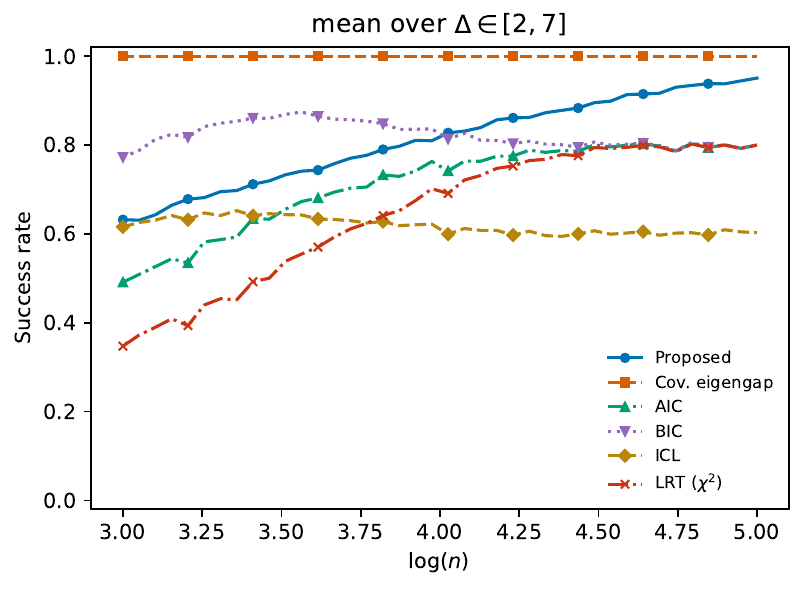}
    }
}
    \subfigure[$k=3$]{
    \resizebox{0.31\textwidth}{!}{
        \includegraphics[]{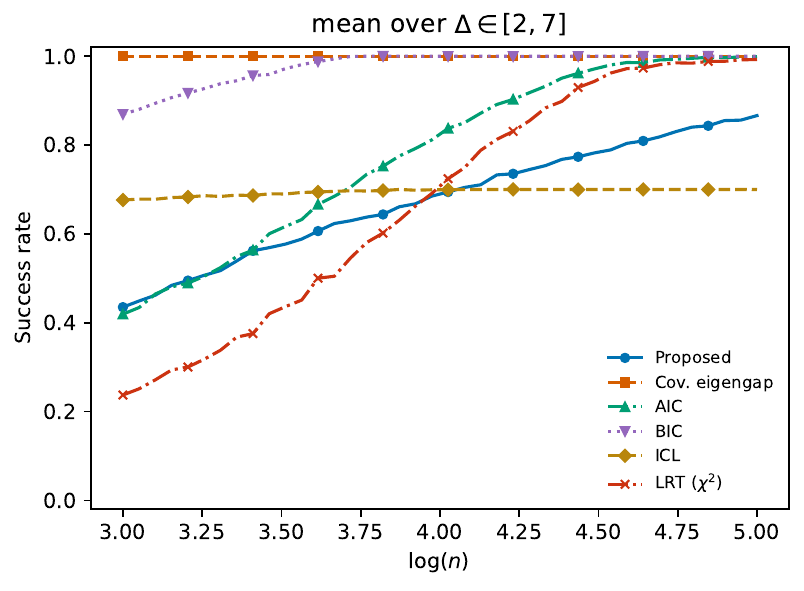}
    }
}
    \subfigure[$k=4$]{
    \resizebox{0.31\textwidth}{!}{
        \includegraphics[]{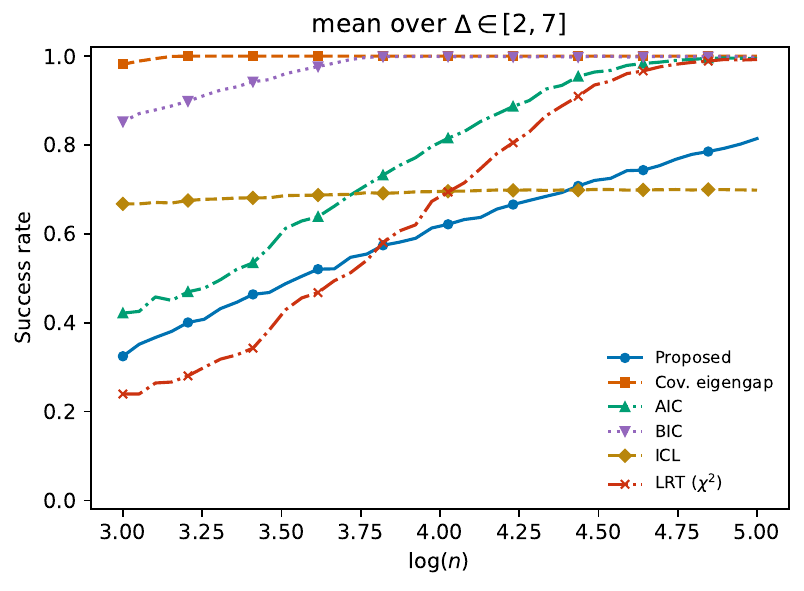}
    }
}
\caption{Empirical success rate averaged over $\Delta \in [2, 7]$, plotted against $\log(n)$, for every method, under the random-data-point-and-Dirichlet-weight EM initialization.}
\label{fig: dirichlet success mean}
\end{figure}

\begin{figure}[!htbp]
    \centering
    \subfigure[$k=2$]{
    \resizebox{0.31\textwidth}{!}{
        \includegraphics[]{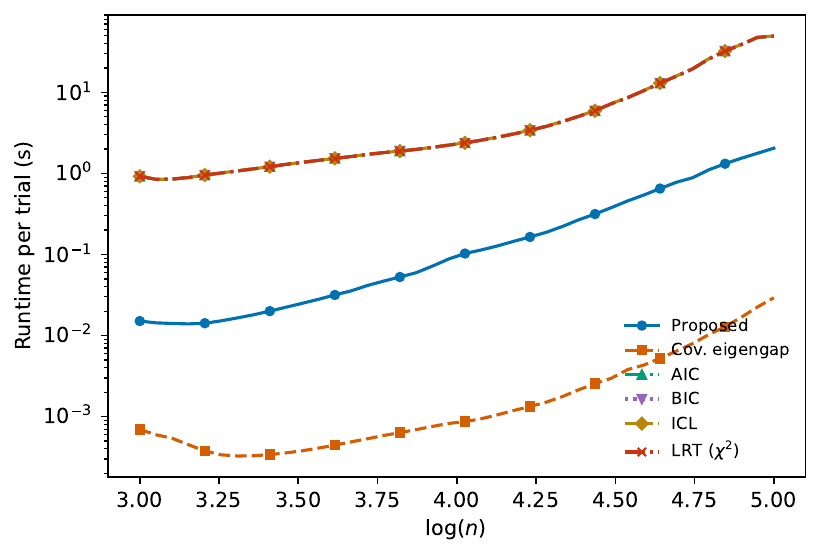}
    }
}
    \subfigure[$k=3$]{
    \resizebox{0.31\textwidth}{!}{
        \includegraphics[]{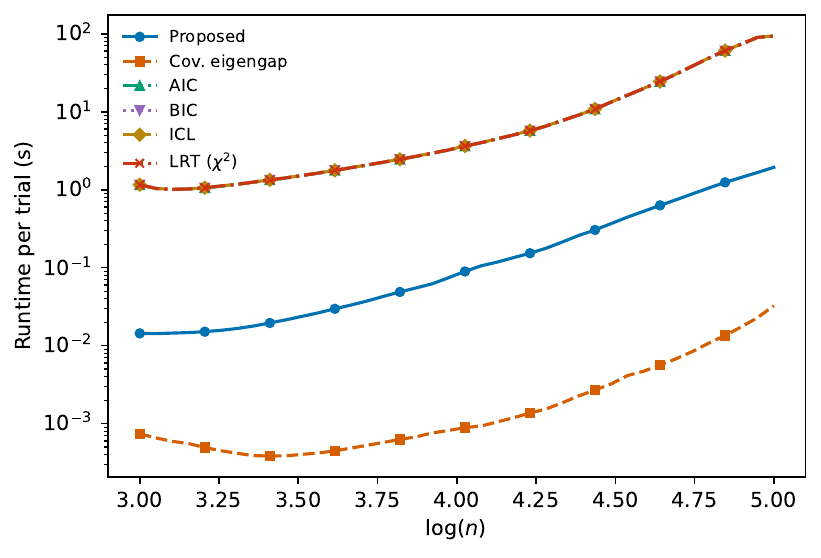}
    }
}
    \subfigure[$k=4$]{
    \resizebox{0.31\textwidth}{!}{
        \includegraphics[]{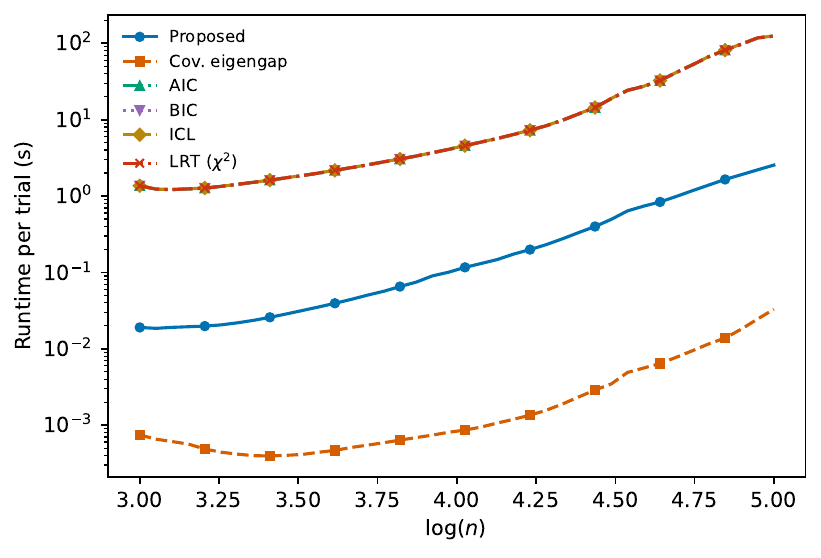}
    }
}
\caption{Mean wall-clock running time per trial versus $\log(n)$, averaged over $\Delta$, under the random-data-point-and-Dirichlet-weight EM initialization.}
\label{fig: dirichlet runtime}
\end{figure}
\FloatBarrier

\subsection{Mixing Distribution Estimation}
In this section, we provide several numerical experiments comparing the proposed estimator with the EM algorithm. The estimation accuracy for the mixing distribution is described by the 1-Wasserstein distance: 
    \begin{equation*}
        W_1(\nu, \hat{\nu}) = \inf \mathbb E \norm{X-Y}_2,
    \end{equation*}
where the infimum is taken for all joint distributions of random vectors $(X,Y)$ with marginals $\nu, \hat{\nu}$ and this 1-Wasserstein distance can be numerically computed through optimal transport \footnote{In our experiments, we use \texttt{wasserstein\_distance\_nd} in the Python package \texttt{scipy}.}.

In all experiments, we run $100$ repeated trials for each sample size $n$. We plot the mean error with its standard deviation and the average running time. We initialize EM with randomly selected samples. The EM algorithm terminates after $1,000$ iterations or when the increase in log-likelihood is less than $10^{-6}$.

In Figure \ref{fig:numerical algo 1}, we compare performance on equally weighted $k$-component GMMs in $\R^k$ for $k=3,4,5$ and on five-component GMMs in $\R^5$ with weights drawn from a Dirichlet distribution with parameter $(1,1,1,1,1)$. The samples are drawn from $\sum_{i=1}^k \mathcal{N}(\mu_i,I_k)$, where the component means are drawn uniformly from the sphere of radius $4$ centered at the origin. The sample size $n$ ranges from $5,000$ to $50,000$ in increments of $5,000$. In Algorithm \ref{algo: mean estimation 1}, we generate $5k$ measurement points uniformly from the ball of radius $0.5$, take the translation vectors to be the standard orthonormal basis of $\R^k$, and set $r_{\mathrm{dup}}=0.01$. The component weights are obtained by solving \eqref{eqn:weight solver} with the CVXOPT toolbox. Our algorithm outperforms EM in both accuracy and computational efficiency. This efficiency gain arises because EM accesses every sample at each iteration, whereas our algorithm uses the full sample only to compute the Fourier data $y_n(t_l+v_m)$ in the first three steps of Algorithm \ref{algo: mean estimation 1}. The variability in EM running time is due to its sensitivity to initialization. If the initial centers are close to $\{\mu_i\}_{i=1}^k$, EM converges in a few iterations; otherwise, it can require many more iterations to reach a local maximum of the likelihood.

\begin{figure}[!htbp]
    \centering
    \subfigure[$k=3$, equal weighted]{
    \resizebox{0.45\textwidth}{!}{
        \includegraphics[]{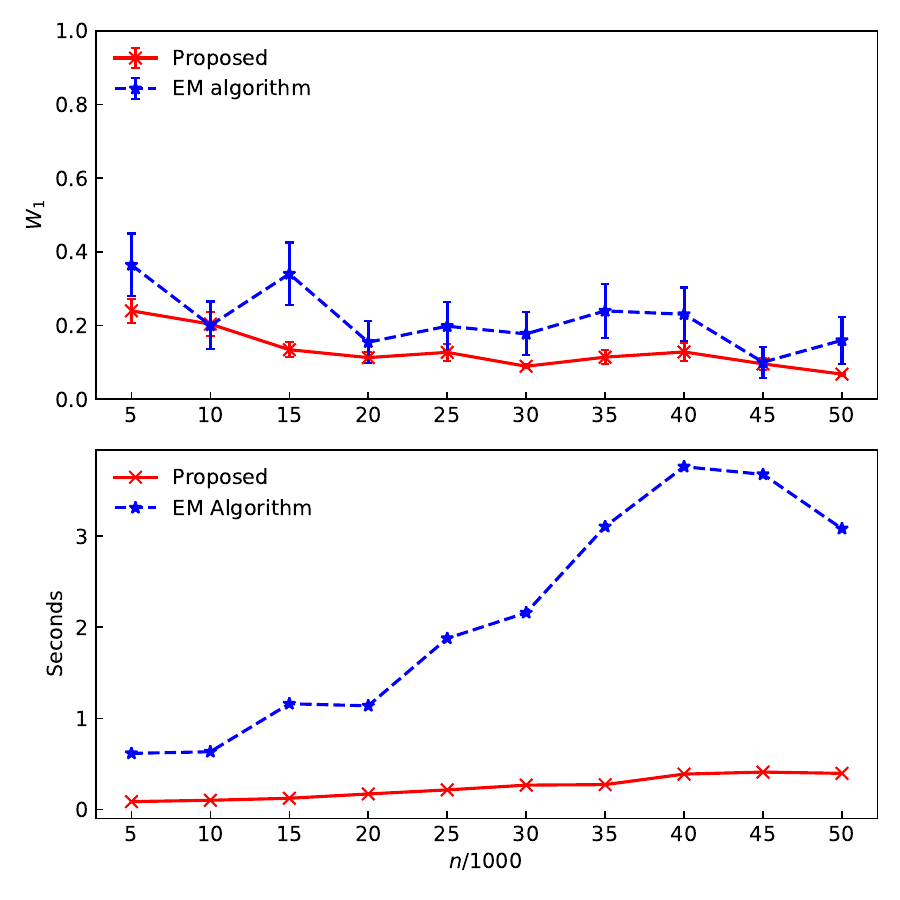}
        }
}
    \subfigure[$k=4$, equal weights]{
    \resizebox{0.45\textwidth}{!}{
        \includegraphics[]{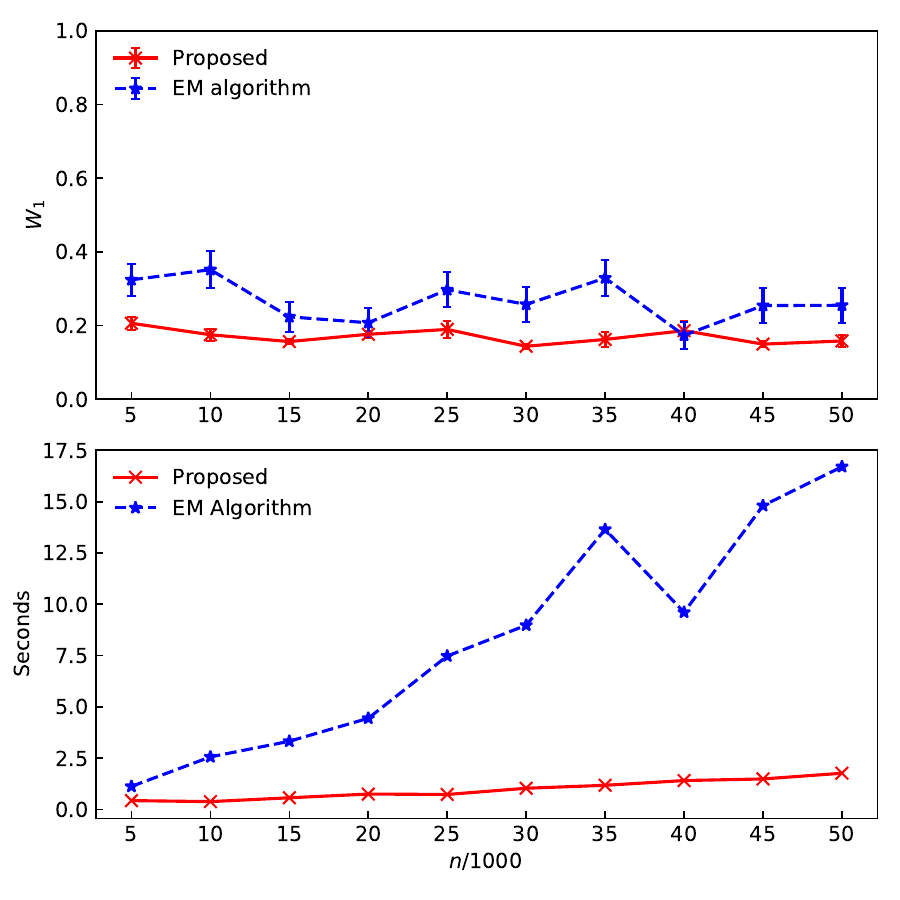}
        }
}
    \subfigure[$k=5$, equal weights]{
    \resizebox{0.45\textwidth}{!}{
        \includegraphics[]{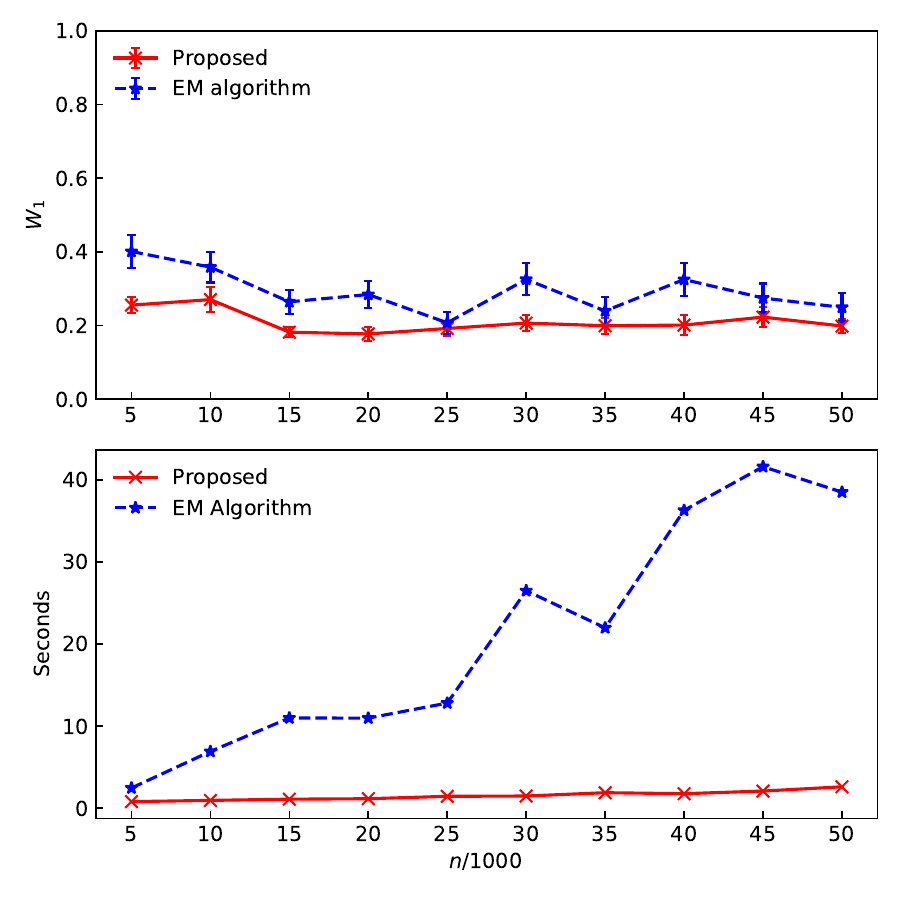}
        }
    }
    \subfigure[$k=5$, unequal weights]{
    \resizebox{0.45\textwidth}{!}{
        \includegraphics[]{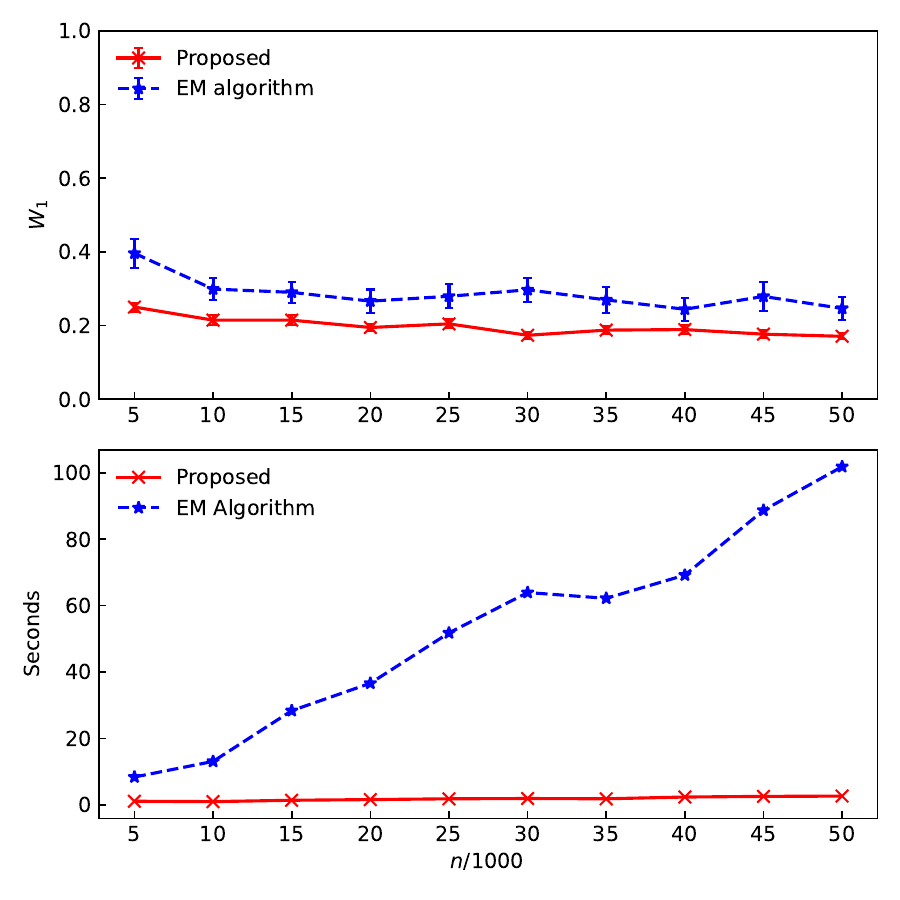}
        }
    }
\caption{$\sum_{i=1}^k w_i \mathcal{N}(\mu_i, I_k)$ with $\mu_i$'s drawn uniformly from the sphere with radius $4$. For each figure, the upper plot shows the accuracy of the mixing distribution estimation under the 1-Wasserstein distance, and the lower plot shows the average running time of each trial.}
    \label{fig:numerical algo 1}
\end{figure}
\FloatBarrier

In Figure \ref{fig:overlap}, we investigate the performance of Algorithm \ref{algo: mean estimation 1} in the cases that the components overlap. In Figure \ref{fig:overlap}(a), the $k = 3$ means are drawn uniformly from of a sphere with radius $0.1$, and in Figure \ref{fig:overlap}(b) we have $\Delta=0$, that is, the components completely overlap. The algorithmic settings are the same as in Figure \ref{fig:numerical algo 1}, with $r_{\mathrm{dup}}=0.01$ in both cases.

\begin{figure}[!htbp]
    \centering
    \subfigure[$k=3$, extremely close]{
    \resizebox{0.45\textwidth}{!}{
        \includegraphics[]{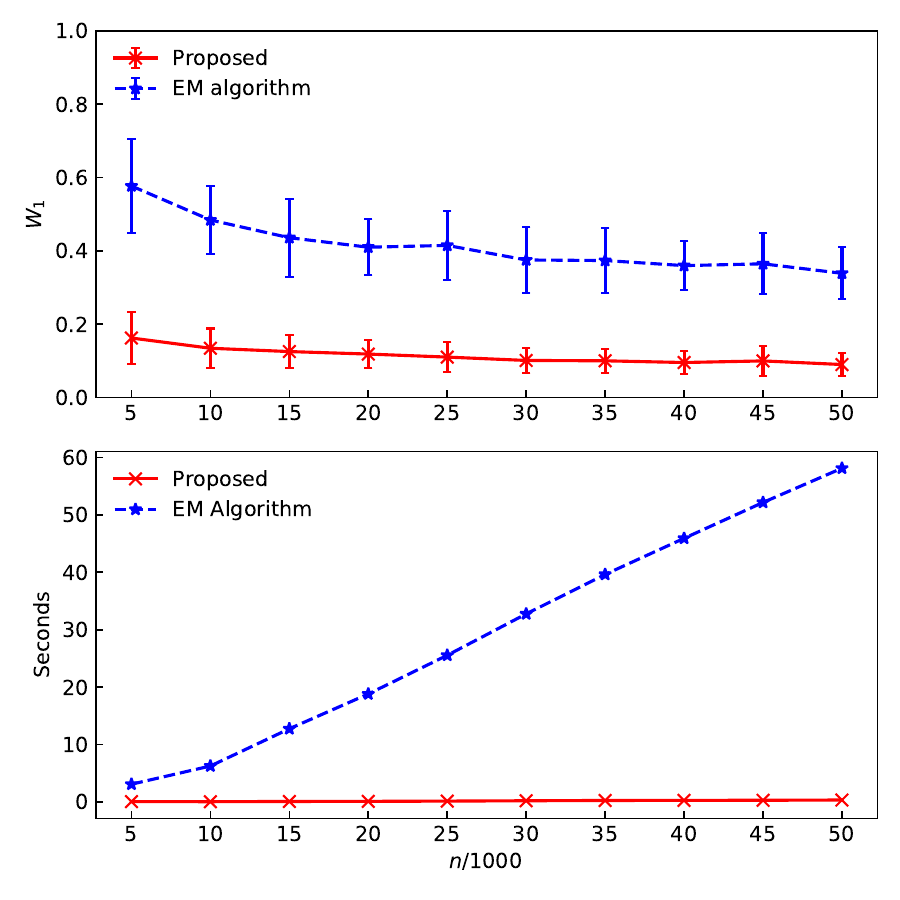}
        }
}
    \subfigure[$k=3$, completely overlap]{
    \resizebox{0.45\textwidth}{!}{
        \includegraphics[]{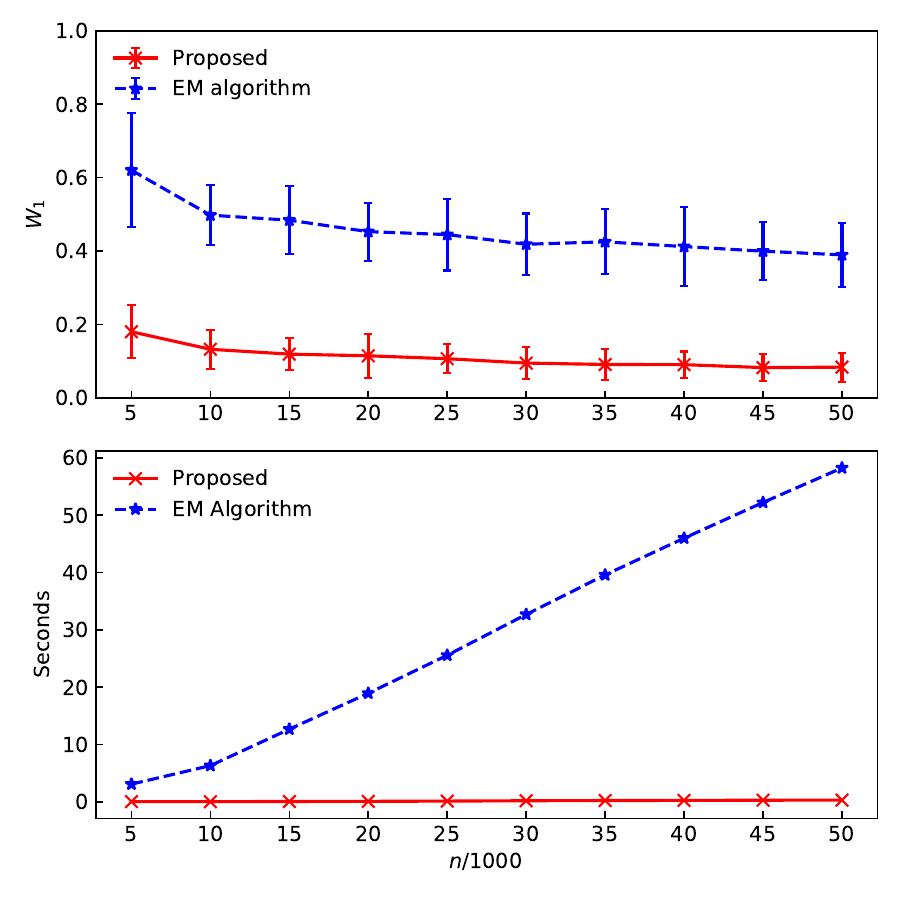}
        }
}
\caption{(a) Means drawn from sphere with radius $0.1$. (b) Means completely overlap i.e. $\Delta = 0$.}
    \label{fig:overlap}
\end{figure}
\FloatBarrier

In Figure \ref{fig: numericals mean high}, we compare performance on equally weighted $k$-component GMMs in $\R^{100}$ for $k=5,6$ and on six-component GMMs with weights drawn from a Dirichlet distribution with parameter $(1,1,1,1,1,1)$. We set $\Sigma=I_{100}$ and draw the component means uniformly from the sphere of radius $4$ centered at the origin. The sample size $n$ ranges from $10,000$ to $100,000$ in increments of $10,000$. In Algorithm \ref{algo: mean estimation 2}, we apply uncentered PCA: we compute the first $k$ right singular vectors of the original data matrix $X$, without subtracting the sample mean, and project the original samples onto their span. We generate $15k$ measurement points uniformly from the ball of radius $0.5$, take the translation vectors to be the standard orthonormal basis of $\R^k$, and set $r_{\mathrm{dup}}=0.01$. Our method achieves better or comparable accuracy for equally weighted mixtures. For unequally weighted GMMs, EM is slightly more accurate than Algorithm \ref{algo: mean estimation 2}, but our algorithm is substantially faster.

\begin{figure}[!htbp]

    \centering
    \subfigure[$k=5$, equal weighted]{
    \resizebox{0.31\textwidth}{!}{
        \includegraphics[]{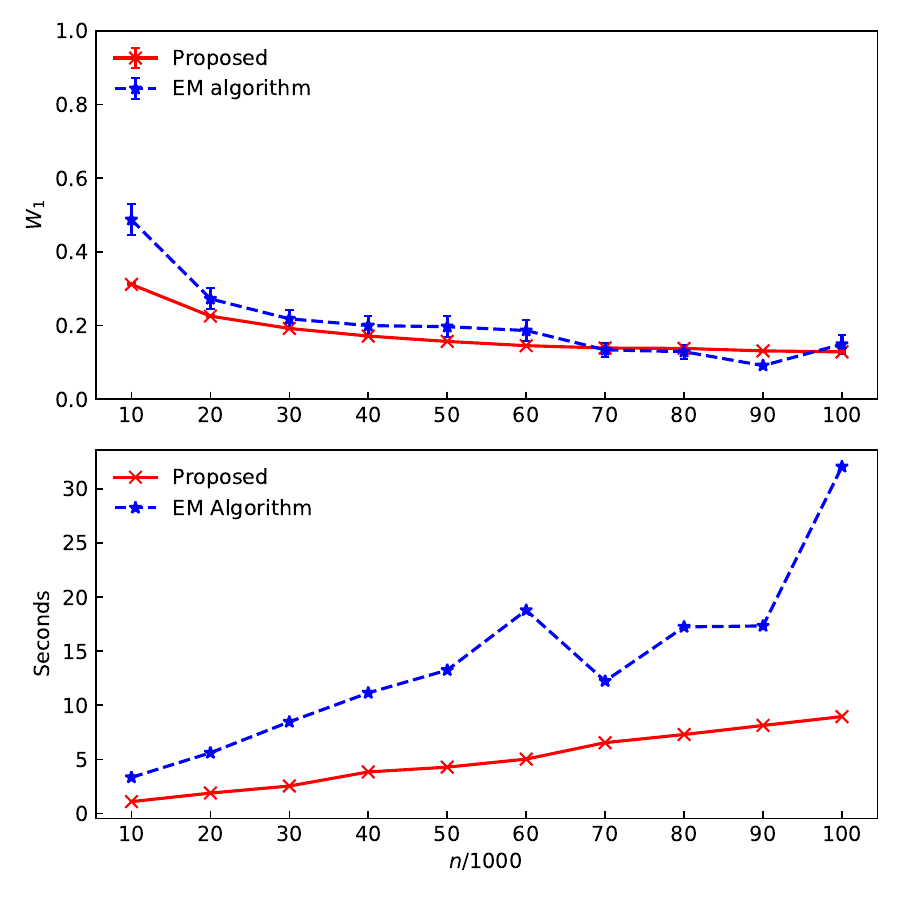}
        }
}
    \subfigure[$k=5$, equal weights]{
    \resizebox{0.31\textwidth}{!}{
        \includegraphics[]{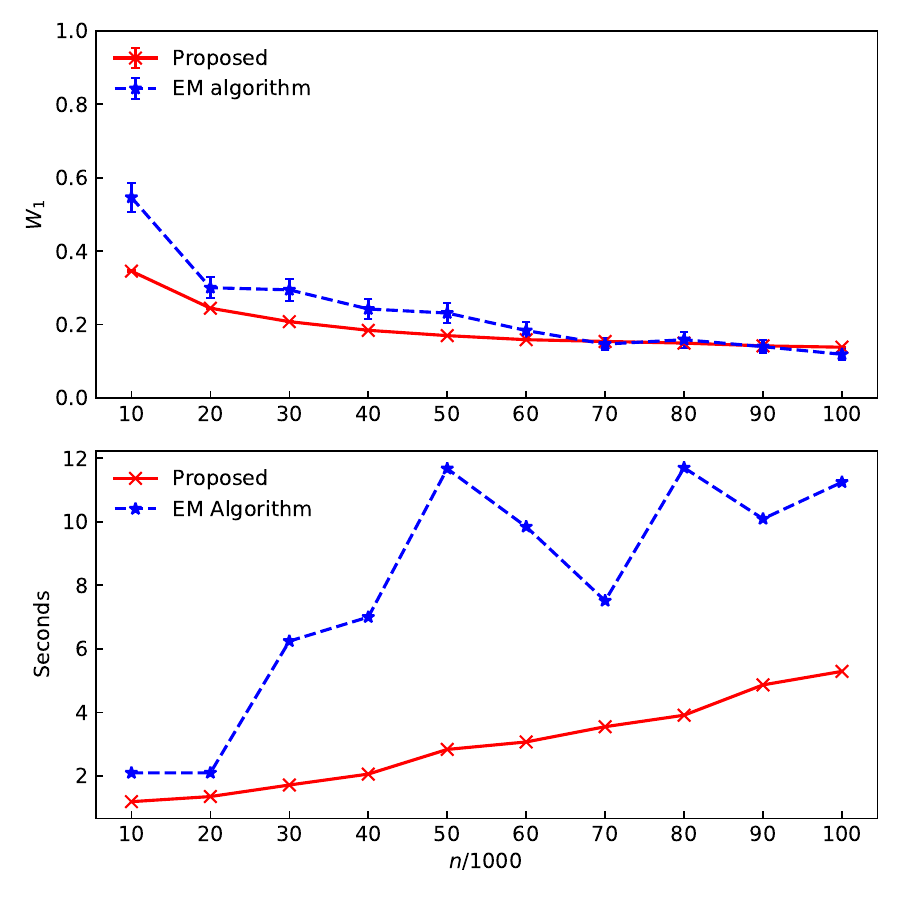}
        }
}
    \subfigure[$k=6$, unequal weights]{
    \resizebox{0.31\textwidth}{!}{
        \includegraphics[]{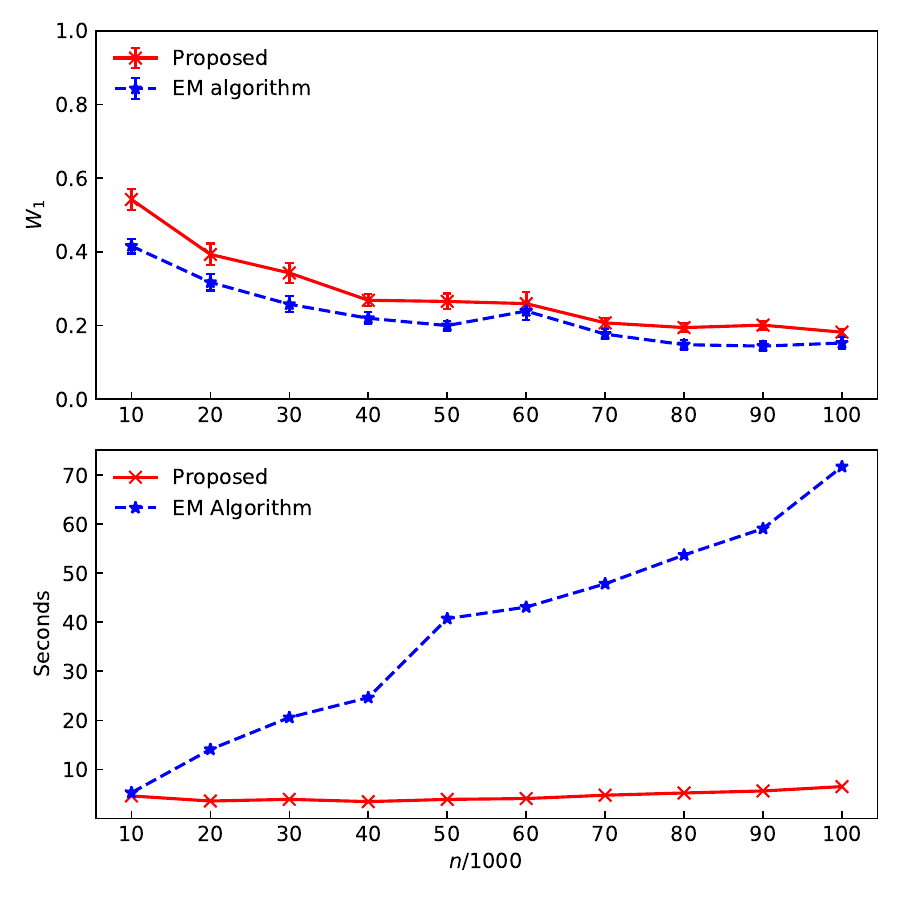}
        }
    }
\caption{$\sum_{i=1}^k w_i \mathcal{N}(\mu_i, I_{100})$ with $\mu_i$'s drawn uniformly from the sphere with radius $4$ in $\R^{100}$. The proposed method uses uncentered PCA for dimension reduction. For each figure, the upper plot shows the accuracy of the mixing distribution estimation under the 1-Wasserstein distance and the lower plot shows average running time of each trial.}
\label{fig: numericals mean high}
\end{figure}
\FloatBarrier

\newpage
\bibliography{reference}
\bibliographystyle{abbrv}

\appendix
\section{Proofs of Propositions and Theorems}
\subsection{Proof of Theorem \ref{thm:embedded-minimax-lower-bound}}

\begin{proof}
Restricting the two suprema in \eqref{eq:embedded-minimax-risk} to
$Q_R^{(d)}$ and $P_R^{(d)}$, respectively, gives
\begin{align*}
\mathcal R_{n,d}^{(k)}(\Delta)
&\ge
\inf_{\psi}\left[
\mathbb E_{(Q_R^{(d)})^{\otimes n}}[\psi]
+\mathbb E_{(P_R^{(d)})^{\otimes n}}[1-\psi]
\right]\\
&=1-
\left\|(P_R^{(d)})^{\otimes n}-(Q_R^{(d)})^{\otimes n}\right\|_{\mathrm{TV}}\\
&\ge
1-\sqrt{\frac12
D_{\mathrm{KL}}\!\left(
(P_R^{(d)})^{\otimes n}\middle\|(Q_R^{(d)})^{\otimes n}
\right)}\\
&=1-\sqrt{\frac n2
D_{\mathrm{KL}}\!\left(P_R^{(d)}\middle\|Q_R^{(d)}\right)}.
\end{align*}
By \eqref{eq:embedded-kl-identity} and Theorem~2.7 of
\cite{liu2024fourier}, for every $\varepsilon\in(0,1)$ and all sufficiently
small $\Delta$,
\[
D_{\mathrm{KL}}\!\left(P_R^{(d)}\middle\|Q_R^{(d)}\right)
=D_{\mathrm{KL}}\!\left(P_R^{(1)}\middle\|Q_R^{(1)}\right)
\le (1+\varepsilon)\kappa_{k,w}\Delta^{4k-4}.
\]
Substitution proves \eqref{eq:embedded-minimax-lower-bound}. Finally,
\[
1-\sqrt{\frac{(1+\varepsilon)\kappa_{k,w}}{2}}\,
\Delta^{2k-2}\sqrt n\ge\eta
\]
is equivalent to
\[
n\le
\frac{2(1-\eta)^2}{(1+\varepsilon)\kappa_{k,w}}
\frac{1}{\Delta^{4k-4}},
\]
which proves the remaining claims.
\end{proof}

\subsection{Proof of Proposition \ref{prop: fourier concentration}}
 \begin{proof}
Note that
        \[
            \hat y_n(t) = \frac{1}{n}\sum_{j=1}^n e^{t^\mathrm{T}\Sigma t/2} \cos\innerproduct{x_j}{t} + \iota \frac{1}{n}\sum_{j=1}^n e^{t^\mathrm{T}\Sigma t/2} \sin\innerproduct{x_j}{t}.
        \]
    Since $\bigl|e^{t^\mathrm{T}\Sigma t/2}\cos\innerproduct{x_j}{t}\bigr| \leq e^{t^\mathrm{T}\Sigma t/2}$ and $\bigl|e^{t^\mathrm{T}\Sigma t/2}\sin\innerproduct{x_j}{t}\bigr| \leq e^{t^\mathrm{T}\Sigma t/2}$, each summand lies in an interval of length $2e^{t^\mathrm{T}\Sigma t/2}$. Applying Hoeffding's inequality to the real and imaginary parts, we have
        \begin{align*}
            &\mathbb{P}\left(\left|\frac{1}{n}\sum_{j=1}^n e^{t^\mathrm{T}\Sigma t/2} \cos\innerproduct{x_j}{t} - e^{t^\mathrm{T}\Sigma t/2}\real{\phi(t)}\right|> \epsilon\right) \leq 2 \exp\left(-\frac{n\epsilon^2}{2e^{t^\mathrm{T}\Sigma t}}\right), \\
            &\mathbb{P}\left(\left|\frac{1}{n}\sum_{j=1}^n e^{t^\mathrm{T}\Sigma t/2} \sin\innerproduct{x_j}{t} - e^{t^\mathrm{T}\Sigma t/2}\imag{\phi(t)}\right|> \epsilon\right) \leq 2 \exp\left(-\frac{n\epsilon^2}{2e^{t^\mathrm{T}\Sigma t}}\right).
        \end{align*}
    Hence,
        \begin{align*}
            &\mathbb{P} \left(\left|e^{t^\mathrm{T}\Sigma t/2}\hat \phi_n(t) - \sum_{i=1}^k w_i\exp\left(\iota \innerproduct{\mu_i}{t}\right)\right| > \epsilon\right) 
            = \mathbb{P} \left(\left|e^{t^\mathrm{T}\Sigma t/2}[\hat \phi_n(t) - \phi(t)]\right| > \epsilon\right) \\
            &\leq \mathbb{P}\left(\left|\real{e^{t^\mathrm{T}\Sigma t/2}[\hat \phi_n(t) - \phi(t)]}\right| > \frac{\epsilon}{\sqrt{2}}\right) + \mathbb{P}\left(\left|\imag{e^{t^\mathrm{T}\Sigma t/2}[\hat \phi_n(t) - \phi(t)]}\right| > \frac{\epsilon}{\sqrt{2}}\right) \\
            &= \mathbb{P}\left(\left|\frac{1}{n}\sum_{j=1}^n e^{t^\mathrm{T}\Sigma t/2} \cos\innerproduct{x_j}{t} - e^{t^\mathrm{T}\Sigma t/2}\real{\phi(t)}\right|> \frac{\epsilon}{\sqrt{2}}\right)\\
            &\quad + \mathbb{P}\left(\left|\frac{1}{n}\sum_{j=1}^n e^{t^\mathrm{T}\Sigma t/2} \sin\innerproduct{x_j}{t} - e^{t^\mathrm{T}\Sigma t/2}\imag{\phi(t)}\right|> \frac{\epsilon}{\sqrt{2}}\right) \\
            &\leq 4 \exp\left(-\frac{n\epsilon^2}{4e^{t^\mathrm{T}\Sigma t}}\right) \leq 4\exp\left(-\frac{n\epsilon^2}{4e^{\norm{t}_2^2 \sigma_{\max}(\Sigma)}}\right), \numberthis \label{eqn: concentration}
        \end{align*}
    where the last inequality uses $t^\mathrm{T}\Sigma t \leq \norm{t}_2^2 \sigma_{\max}(\Sigma)$.
    If we choose $ n \geq \frac{4}{\epsilon^2}\ln\left(\frac{4}{\delta}\right){e^{\norm{t}_2^2 \sigma_{\max}(\Sigma)}}$, then \eqref{eqn: concentration} implies that 
    \[
        \mathbb{P} \left(\left|e^{t^\mathrm{T}\Sigma t/2}\hat \phi_n(t) - \sum_{i=1}^k w_i\exp\left(\iota \innerproduct{\mu_i}{t}\right)\right| > \epsilon\right) < \delta,
    \]
    which completes the proof.
\end{proof}

\subsection{Proof of Lemma \ref{lem:projected-separation}}

\begin{proof}
Let $v\sim\mathcal U(\mathbb S^{d-1})$ and, for each $i\ne j$, set
\[
    u_{ij}:=\frac{\mu_i-\mu_j}{\|\mu_i-\mu_j\|_2},
    \qquad
    Z_{ij}:=\langle u_{ij},v\rangle,
    \qquad
    \epsilon:=\frac{1}{k^2\sqrt d}.
\]
We first show that
\begin{equation}
    \mathbb P\bigl(|Z_{ij}|\le\epsilon\bigr)\le\frac{1}{k^2}.
    \label{eqn:pairwise-projection-failure}
\end{equation}

If $d=1$, then $v\in\{-1,1\}$ and $|Z_{ij}|=1$. Since $k\ge2$,
\[
    \mathbb P\bigl(|Z_{ij}|\le\epsilon\bigr)=0.
\]

If $d=2$, then $Z_{ij}$ has density
\[
    f_Z(t)=\frac{1}{\pi\sqrt{1-t^2}},
    \qquad -1<t<1.
\]
Consequently, using $\arcsin(x)\le(\pi/2)x$ for $x\in[0,1]$,
\[
    \mathbb P\bigl(|Z_{ij}|\le\epsilon\bigr)
    =\frac{2}{\pi}\arcsin(\epsilon)
    \le\epsilon
    =\frac{1}{k^2\sqrt2}
    \le\frac{1}{k^2}.
\]

If $d\ge3$, then $Z_{ij}$ has density
\[
    f_Z(t)=C_d(1-t^2)^{(d-3)/2},
    \qquad
    C_d:=\frac{1}{\sqrt\pi}
    \frac{\Gamma(d/2)}{\Gamma((d-1)/2)}.
\]
Since $(d-3)/2\ge0$, $f_Z(t)\le C_d$. By Gautschi's inequality,
\[
    C_d
    <\sqrt{\frac{d-1}{2\pi}}
    <\frac{\sqrt d}{2}.
\]
Hence
\[
    \mathbb P\bigl(|Z_{ij}|\le\epsilon\bigr)
    \le 2\epsilon C_d
    \le\epsilon\sqrt d
    =\frac{1}{k^2}.
\]
Thus, \eqref{eqn:pairwise-projection-failure} holds for every $d\ge1$.

For a single random direction, let
\[
    E_{\mathrm{fail}}
    :=\bigcup_{1\le i<j\le k}
    \bigl\{|\langle u_{ij},v\rangle|<\epsilon\bigr\}.
\]
By the union bound,
\[
    \mathbb P(E_{\mathrm{fail}})
    \le\binom{k}{2}\frac{1}{k^2}
    =\frac{k-1}{2k}
    <\frac12.
\]
Since $\|\mu_i-\mu_j\|_2\ge\Delta$, on $E_{\mathrm{fail}}^c$,
\[
    \min_{i\ne j}|\langle\mu_i-\mu_j,v\rangle|
    \ge\frac{\Delta}{k^2\sqrt d}.
\]
For $J$ independent directions $t_1,\ldots,t_J$,
\[
    \mathbb P(\text{all $J$ directions fail})
    \le 2^{-J}\le\delta
\]
whenever $J\ge\lceil\log_2(1/\delta)\rceil$.
\end{proof}

\subsection{Proof of Proposition \ref{prop:covariance-perturbation}}

\begin{proof}
\textbf{Step 1 (Uniform concentration of the noise).}
For every evaluation point $t=t_i+v_m$ we have
$\|t\|_2\le T_{\max}$, so
$e^{\|t\|_2^2\sigma_{\max}(\Sigma)}\le{\kappa_\Sigma}$. Proposition~\ref{prop: fourier
concentration} then gives, for each of the $N$ points,
\[
  \mathbb P\big(|e_n(t_i+v_m)|\ge\epsilon\big)
  \le 4\exp\!\left(-\frac{n\epsilon^2}{4{\kappa_\Sigma}}\right).
\]
By the union bound over all $N=L(M+1)$ points,
\[
  \mathbb P\Big(\max_{i,m}|e_n(t_i+v_m)|\ge\epsilon\Big)
  \le 4N\exp\!\left(-\frac{n\epsilon^2}{4{\kappa_\Sigma}}\right).
\]
Equating the right-hand side to $\delta$ and solving for $\epsilon$ yields
$\epsilon_n = 2\sqrt{\kappa_\Sigma}\sqrt{\ln(4N/\delta)/n}$. Hence, with probability at least
$1-\delta$,
\begin{equation}
  \max_{i,m}|e_n(t_i+v_m)| < \epsilon_n .
  \label{eq:uniform-noise}
\end{equation}
We work on this event for the remainder of the proof.

\textbf{Step 2 (Entrywise to vector norm).}
Since $\ve_m = \hat\vy_m-\vy_m$ has entries $e_n(t_i+v_m)$, inequality
\eqref{eq:uniform-noise} gives
\[
  \|\ve_m\|_2 = \Big(\sum_{i=1}^L |e_n(t_i+v_m)|^2\Big)^{1/2}
  < \sqrt{L}\,\epsilon_n,
  \qquad m=0,1,\dots,M .
\]
Together with $\|\vy_m\|_2\le\sqrt L$ (each $|y|\le\sum_i w_i=1$), this controls
every term in your telescoping bound.

\textbf{Step 3 (Assembling the spectral bound).}
Starting from
\[
  \|\hat E\|_2
  \le \frac{1}{M+1}\sum_{m=0}^M
      \big(2\|\ve_m\|_2\|\vy_m\|_2 + \|\ve_m\|_2^2\big),
\]
substitute $\|\vy_m\|_2\le\sqrt L$ and $\|\ve_m\|_2<\sqrt L\,\epsilon_n$:
\[
  \|\hat E\|_2
  < \frac{1}{M+1}\sum_{m=0}^M
      \big(2\sqrt L\cdot\sqrt L\,\epsilon_n + L\,\epsilon_n^2\big)
  = 2L\,\epsilon_n + L\,\epsilon_n^2 .
\]
Finally, inserting $\epsilon_n = 2\sqrt{\kappa_\Sigma}\sqrt{\ln(4N/\delta)/n}$ gives
\[
  \|\hat E\|_2
  \le 4L\sqrt{\kappa_\Sigma}\sqrt{\frac{\ln(4N/\delta)}{n}}
    + 4L{\kappa_\Sigma}\,\frac{\ln(4N/\delta)}{n},
\]
this proves the claim.
\end{proof}

\subsection{Proof of Proposition \ref{prop:lower-bound-sampling}}

\begin{proof}
\textbf{Step 0 (The overall setting).}
For the measurement directions $t_1,\dots,t_J\sim\mathcal{U}(\mathbb{S}^{d-1})$,
Lemma~\ref{lem:projected-separation} with confidence $\delta/2$ gives:
provided $J\ge\lceil\log_2(2/\delta)\rceil$, with probability at least
$1-\delta/2$ there exists an index, relabeled $t_a$, such that
\begin{equation}
  \min_{i\neq j}|\langle\mu_i-\mu_j,t_a\rangle| \ge \frac{\Delta}{k^2\sqrt{d}}.
  \tag{P1}\label{eq:P1}
\end{equation}
For the independently drawn translation vectors
$v_1,\dots,v_J\sim\mathcal{U}(\mathbb{S}^{d-1})$, the same lemma with confidence
$\delta/2$ gives, with probability at least $1-\delta/2$, an index, relabeled
$v_b$, such that
\begin{equation}
  \min_{i\neq j}|\langle\mu_i-\mu_j,v_b\rangle| \ge \frac{\Delta}{k^2\sqrt{d}}.
  \tag{P2}\label{eq:P2}
\end{equation}
By the union bound, \eqref{eq:P1} and \eqref{eq:P2} hold simultaneously with
probability at least $1-\delta$. We work on this event; all subsequent
statements are deterministic.

Note that for any unit vector $u$ and any pair $i\neq j$, Cauchy--Schwarz gives
$|\langle\mu_i-\mu_j,u\rangle|\le\|\mu_i-\mu_j\|_2\le R$. Hence, with
$\tau\le\pi/R$ and $\rho\le\pi/R$, all phase differences appearing below are
bounded in modulus by $\pi$, which is exactly what the chord--arc inequality
requires.

\medskip
\medskip
\textbf{Step 1 (Vandermonde submatrix of $\mathcal A$).}
Recall $C=\mathcal A W\mathcal A^*$. Fix the direction $t_a$ satisfying \eqref{eq:P1}. In the
block of $\mathcal A$ corresponding to $t_a$, the row indexed by step
$s\in\{0,1,\dots,S\}$ has entries
\[
  \mathcal A_{(a,s),i}=e^{\iota s\tau\langle\mu_i,t_a\rangle}=\zeta_i^{\,s},
  \qquad i=1,\dots,k,
\]
where we set the nodes
\[
  \zeta_i:=e^{\iota\tau\langle\mu_i,t_a\rangle}.
\]
Since $S+1\ge k$, the rows corresponding to $s=0,1,\dots,k-1$ are all present, and
they form the square Vandermonde matrix
\[
  V=\begin{pmatrix}
    1 & \cdots & 1\\
    \zeta_1 & \cdots & \zeta_k\\
    \zeta_1^{2} & \cdots & \zeta_k^{2}\\
    \vdots & & \vdots\\
    \zeta_1^{\,k-1} & \cdots & \zeta_k^{\,k-1}
  \end{pmatrix}\in\mathbb{C}^{k\times k}.
\]
Thus $V$ is exactly a row-submatrix of $\mathcal A$ (no phase correction is needed).
Collecting the remaining rows of $\mathcal A$ into a matrix $R'$, we have, up to a row
permutation that leaves singular values unchanged,
\[
  \mathcal A=\begin{bmatrix} V\\ R'\end{bmatrix},
  \qquad\text{hence}\qquad
  \mathcal A^*\mathcal A=V^*V+R'^*R'\;\succeq\;V^*V.
\]
Because the nodes $\zeta_i$ are distinct by \eqref{eq:P1} (the phases
$\tau\langle\mu_i,t_a\rangle$ are pairwise distinct modulo $2\pi$ since their
differences are nonzero and bounded by $\pi$ in modulus), $V$ is invertible.

\medskip
\textbf{Step 2 (Lower bound for $\sigma_k(\mathcal A)$).}
Since $\mathcal A^*\mathcal A=V^*V+R'^*R'\succeq V^*V\succeq0$, Weyl's monotonicity
inequality gives $\lambda_k(\mathcal A^*\mathcal A)\ge\lambda_k(V^*V)$, and therefore
\begin{equation}
  \sigma_k(\mathcal A)=\sqrt{\lambda_k(\mathcal A^*\mathcal A)}\ge\sqrt{\lambda_k(V^*V)}
  =\sigma_k(V).
  \label{eq:phi-ge-V}
\end{equation}
For the square, invertible Vandermonde $V$ we have
$\sigma_k(V)=1/\|V^{-1}\|_2$, and since $\|V^{-1}\|_2\le\|V^{-1}\|_F
\le\sqrt{k}\,\|V^{-1}\|_\infty$,
\begin{equation}
  \sigma_k(V)=\frac{1}{\|V^{-1}\|_2}\ge\frac{1}{\sqrt{k}\,\|V^{-1}\|_\infty}.
  \label{eq:sigmaV-lower}
\end{equation}
By Lemma~\ref{lemma:van}, the row sums of $|V^{-1}|$ obey
\begin{equation}
  \|V^{-1}\|_\infty \le \max_{1\le j\le k}\prod_{m\neq j}
    \frac{2}{|\zeta_j-\zeta_m|}.
  \label{eq:Vinv-bound}
\end{equation}
By Step~0 the phase difference satisfies
$|\langle\mu_j-\mu_m,\tau t_a\rangle|=\tau|\langle\mu_j-\mu_m,t_a\rangle|
\le\tau R\le\pi$, so the chord--arc inequality
$|e^{\iota\theta}-1|=2|\sin(\theta/2)|\ge(2/\pi)|\theta|$ for $|\theta|\le\pi$,
together with \eqref{eq:P1}, yields
\begin{equation}
  |\zeta_j-\zeta_m|
  =\bigl|e^{\iota\tau\langle\mu_j-\mu_m,t_a\rangle}-1\bigr|
  \ge\frac{2}{\pi}\,\tau|\langle\mu_j-\mu_m,t_a\rangle|
  \ge\frac{2\tau\Delta}{\pi k^2\sqrt{d}}.
  \label{eq:node-gap-V}
\end{equation}
Each product in \eqref{eq:Vinv-bound} has $k-1$ factors, so substituting
\eqref{eq:node-gap-V} gives
\[
  \|V^{-1}\|_\infty\le\left(\frac{\pi k^2\sqrt{d}}{\tau\Delta}\right)^{k-1}.
\]
Combining this with \eqref{eq:phi-ge-V} and \eqref{eq:sigmaV-lower},
\begin{equation}
  \sigma_k(\mathcal A)\ge\sigma_k(V)\ge\frac{1}{\sqrt{k}}
    \left(\frac{\tau\Delta}{\pi k^2\sqrt{d}}\right)^{k-1},
  \label{eq:phi-final}
\end{equation}
which proves \eqref{eq:sigmak-Phi}.

\medskip
\textbf{Step 3 (Lower bound for $\sigma_k(W)$).}
For each $j\in\{1,\dots,J\}$ define the phase vector of the $m$-th translation
step along $v_j$ as
\[
\boldsymbol{\zeta}_{j,m}
:=\bigl(e^{\iota m\rho\langle\mu_1,v_j\rangle},\dots,
e^{\iota m\rho\langle\mu_k,v_j\rangle}\bigr)^{\top}.
\]
Let $D:=\mathrm{diag}(w_1,\dots,w_k)$, and define the local and averaged Gram
matrices
\[
G_j:=\frac{1}{M'+1}\sum_{m=0}^{M'}
  \boldsymbol{\zeta}_{j,m}\boldsymbol{\zeta}_{j,m}^{*},
\qquad
G:=\frac{1}{J}\sum_{j=1}^{J}G_j ,
\]
so that $W=DGD$. Each $G_j\succeq0$, hence $G\succeq\frac{1}{J}G_b$ and
\[
\lambda_{\min}(G)\ge\frac{1}{J}\lambda_{\min}(G_b).
\]
Let $Z_b':=[\boldsymbol{\zeta}_{b,0},\dots,\boldsymbol{\zeta}_{b,k-1}]
\in\mathbb{C}^{k\times k}$ be the square Vandermonde matrix formed by the first
$k$ translation steps (using the assumption that $M'+1\ge k$), with nodes
$\eta_i=e^{\iota\rho\langle\mu_i,v_b\rangle}$. Dropping the remaining
nonnegative rank-one terms,
\[
G_b=\frac{1}{M'+1}\sum_{m=0}^{M'}
  \boldsymbol{\zeta}_{b,m}\boldsymbol{\zeta}_{b,m}^{*}
\;\succeq\;\frac{1}{M'+1}\,Z_b'(Z_b')^{*},
\]
so that
\[
\lambda_{\min}(G_b)\ge\frac{1}{M'+1}\,\sigma_k(Z_b')^2 .
\]
For the square Vandermonde $Z_b'$,
$\sigma_k(Z_b')=1/\|(Z_b')^{-1}\|_2\ge 1/(\sqrt{k}\,\|(Z_b')^{-1}\|_\infty)$,
and by Lemma~\ref{lemma:van},
$\|(Z_b')^{-1}\|_\infty\le\max_i\prod_{l\neq i}2/|\eta_i-\eta_l|$. By Step~0,
$|\langle\mu_i-\mu_l,\rho v_b\rangle|\le\rho R\le\pi$, so chord--arc together
with \eqref{eq:P2} gives
\[
|\eta_i-\eta_l|\ge\frac{2}{\pi}\rho|\langle\mu_i-\mu_l,v_b\rangle|
  \ge\frac{2\rho\Delta}{\pi k^2\sqrt{d}}.
\]
Hence $\|(Z_b')^{-1}\|_\infty\le(\pi k^2\sqrt{d}/(\rho\Delta))^{k-1}$ and
\[
\sigma_k(Z_b')\ge\frac{1}{\sqrt{k}}
  \left(\frac{\rho\Delta}{\pi k^2\sqrt{d}}\right)^{k-1},
\qquad\Longrightarrow\qquad
\lambda_{\min}(G_b)\ge\frac{1}{k(M'+1)}
  \left(\frac{\rho\Delta}{\pi k^2\sqrt{d}}\right)^{2k-2}.
\]
Finally, since $W=DGD$ with $D$ real diagonal and entries $\ge w_{\min}$,
\[
\sigma_k(W)=\lambda_{\min}(W)\ge w_{\min}^2\lambda_{\min}(G)
\ge\frac{w_{\min}^2}{J}\lambda_{\min}(G_b)
\ge\frac{w_{\min}^2}{kJ(M'+1)}
  \left(\frac{\rho\Delta}{\pi k^2\sqrt{d}}\right)^{2k-2},
\]
which proves \eqref{eq:sigmak-W}.

\medskip
\textbf{Step 4 (Combine for $\sigma_k(C)$).}
Since $C=\mathcal A W\mathcal A^*$,
\[
\sigma_k(C)\ge \sigma_k(\mathcal A)^2\,\sigma_k(W).
\]
Combining \eqref{eq:phi-final} and \eqref{eq:sigmak-W},
\[
\sigma_k(C)
\ge\frac{1}{k}\left(\frac{\tau\Delta}{\pi k^2\sqrt{d}}\right)^{2k-2}
\cdot\frac{w_{\min}^2}{kJ(M'+1)}
  \left(\frac{\rho\Delta}{\pi k^2\sqrt{d}}\right)^{2k-2}
=\frac{w_{\min}^2}{k^2J(M'+1)}
  \left(\frac{\tau\Delta}{\pi k^2\sqrt{d}}\right)^{2k-2}
  \left(\frac{\rho\Delta}{\pi k^2\sqrt{d}}\right)^{2k-2},
\]
which proves \eqref{eq:sigmak-C}. Setting $\tau=\rho$ gives the stated special case. This
completes the proof.
\end{proof}

\subsection{Proof of Theorem \ref{thm:thresholding}}

\begin{proof}
Write the empirical Fourier covariance as $\hat C=C+\hat E$, where
$C=\mathcal A W\mathcal A^*$ is the population Fourier covariance and $\hat E$ is the
sampling error. Since $C$ is generated by the $k$ mixture components,
$\operatorname{rank}(C)\le k$, so
\begin{equation}
  \sigma_l(C)=0,\qquad l=k+1,\dots,L .
  \label{eq:rankC}
\end{equation}
By Weyl's inequality for singular values, for every $l$,
\begin{equation}
  \bigl|\hat\sigma_l-\sigma_l(C)\bigr|\le\|\hat E\|_2 .
  \label{eq:weyl}
\end{equation}

\medskip
\textbf{Step 1 (Control of $\|\hat E\|_2$).}
Apply Proposition~\ref{prop:covariance-perturbation}. 

For $n$ satisfying the condition \eqref{eqn: threshold condition 1}, we can verify that 
\[
4L \sqrt{\kappa_\Sigma}\sqrt{\frac{\ln(4N/\delta)}{n}} < \frac{\varepsilon}{2}, \;\; 4L{\kappa_\Sigma}\,\frac{\ln(4N/\delta)}{n} < \frac{\varepsilon}{2}.
\]
Therefore, 
\begin{equation}
  \|\hat E\|_2<\frac{\varepsilon}{2}+\frac{\varepsilon}{2} =\varepsilon .
  \label{eq:Efinal}
\end{equation}

\medskip
\textbf{Step 2 (Noise singular values, $l>k$).}
Fix $l\in\{k+1,\dots,L\}$. Combining \eqref{eq:weyl}, \eqref{eq:rankC} and
\eqref{eq:Efinal},
\[
  \hat\sigma_l\le\sigma_l(C)+\|\hat E\|_2=\|\hat E\|_2<\varepsilon .
\]
This holds simultaneously for all such $l$ on the event
\eqref{eq:Efinal}, i.e. with probability at least $1-\delta$, which proves the
first assertion.

\medskip
\textbf{Step 3 (Signal singular value, $l=k$).}
By \eqref{eq:weyl} with $l=k$ and \eqref{eq:Efinal},
\begin{equation}
  \hat\sigma_k\ge\sigma_k(C)-\|\hat E\|_2>\sigma_k(C)-\varepsilon.
  \label{eq:sigk-lb}
\end{equation}
On the event of Proposition~\ref{prop:lower-bound-sampling}, taking
$\tau=\rho$, the population covariance satisfies
\[
  \sigma_k(C)\ge\frac{w_{\min}^2}{k^2J(M'+1)}
    \left(\frac{\tau\Delta}{\pi k^2\sqrt d}\right)^{4k-4}.
\]
Hence condition \eqref{eqn: threshold condition 2} is precisely
$\varepsilon<\tfrac12\sigma_k(C)$, i.e. $\sigma_k(C)>2\varepsilon$. Substituting into
\eqref{eq:sigk-lb} proves the second assertion.

\medskip
\textbf{Probability accounting.}
The bound \eqref{eq:Efinal} holds with probability at least $1-\delta$ over the
$n$ samples, and the lower bound on $\sigma_k(C)$ holds with probability at
least $1-\delta$ over the random sampling of frequency vectors and their translations $\{t_j\},\{v_j\}$
(Proposition~\ref{prop:lower-bound-sampling}, valid since
$J\ge\lceil\log_2(2/\delta)\rceil$). The two events are independent, so by a
union bound both conclusions hold simultaneously with probability at least
$1-2\delta$.
\end{proof}

\subsection{Proof of Theorem \ref{thm:model-selection}}
\begin{proof}
Let
$\mathcal E_n
    =
    \left\{
        \|\widehat E\|_2\le\eta_n
    \right\}$.
By Proposition~\ref{prop:covariance-perturbation},
$\mathbb P(\mathcal E_n)\ge 1-\delta$.
We prove the result on the event $\mathcal E_n$.

First, condition~\eqref{eq:modelselection-n-revised} ensures that both
summands in \eqref{eq:perturbation-level-ratio} are smaller than $b_c/2$.
Indeed,
\[
    n>
    \frac{64L^2\kappa_\Sigma}{b_c^2}
    \ln\!\left(\frac{4N}{\delta}\right)  \Rightarrow 
    4L\sqrt{\kappa_\Sigma}
    \sqrt{\frac{\ln(4N/\delta)}{n}}
    <\frac{b_c}{2},
\]
while
\[
    n>
    \frac{8L\kappa_\Sigma}{b_c}
    \ln\!\left(\frac{4N}{\delta}\right)
\Rightarrow
    4L\kappa_\Sigma
    \frac{\ln(4N/\delta)}{n}
    <\frac{b_c}{2}.
\]
Consequently,
\begin{equation}
\label{eq:eta-bc-bound}
    \eta_n<b_c.
\end{equation}

By Weyl's inequality, on $\mathcal E_n$,
\begin{equation}
\label{eq:weyl-regularized-ratio}
    \bigl|\widehat{\sigma}_l-\sigma_l\bigr|
    \le\eta_n,
    \qquad l=1,\ldots,L.
\end{equation}

Since $b_c\le c\sigma_k/(1+c)$, \eqref{eq:eta-bc-bound} gives
$ (1+c)\eta_n<c\sigma_k$.
Therefore,
\[
    \sigma_k-\eta_n>c\eta_n=\varepsilon_*.
\]
It follows from \eqref{eq:weyl-regularized-ratio} that
\[
    \widehat{\sigma}_l
    \ge\sigma_k-\eta_n
    >\varepsilon_*,
    \qquad 1\le l\le k.
\]
Thus all signal indices, including $k$, belong to
$\widehat{\mathcal I}_{\varepsilon_*}$.

For the ratio at the true order, since
$\widehat{\sigma}_{k+1}\le\eta_n$ and
$\varepsilon_*=c\eta_n\le\eta_n$, we obtain
\begin{equation}
\label{eq:true-regularized-ratio-lower}
    \widehat R_k(\varepsilon_*)
    =
    \frac{\widehat{\sigma}_k}
         {\widehat{\sigma}_{k+1}\vee\varepsilon_*}
    \ge
    \frac{\sigma_k-\eta_n}{\eta_n}.
\end{equation}

For every $1\le l<k$, we have
\[
    \widehat{\sigma}_l\le\sigma_1+\eta_n
\]
and
\[
    \widehat{\sigma}_{l+1}
    \ge\sigma_k-\eta_n
    >\varepsilon_*.
\]
Hence the denominator is not altered by the threshold, and
\begin{equation}
\label{eq:interior-regularized-ratio-upper}
    \widehat R_l(\varepsilon_*)
    =
    \frac{\widehat{\sigma}_l}{\widehat{\sigma}_{l+1}}
    \le
    \frac{\sigma_1+\eta_n}{\sigma_k-\eta_n}.
\end{equation}
The right-hand side of
\eqref{eq:true-regularized-ratio-lower} is strictly larger than that of
\eqref{eq:interior-regularized-ratio-upper} whenever
\[
    \frac{\sigma_k-\eta_n}{\eta_n}
    >
    \frac{\sigma_1+\eta_n}{\sigma_k-\eta_n}.
\]
This is equivalent to
\[
    \sigma_k^2>(\sigma_1+2\sigma_k)\eta_n,
\]
which follows from
\[
    \eta_n<b_c
    \le\frac{\sigma_k^2}{\sigma_1+2\sigma_k}.
\]
Therefore,
\begin{equation}
\label{eq:true-ratio-dominates-signal}
    \widehat R_k(\varepsilon_*)
    >
    \max_{1\le l<k}\widehat R_l(\varepsilon_*).
\end{equation}

It remains to control the ratios whose numerators correspond to noise
singular values. For every $l>k$, Weyl's inequality and monotonicity give
\[
    \widehat{\sigma}_l
    \le\widehat{\sigma}_{k+1}
    \le\eta_n.
\]
Therefore,
\begin{equation}
\label{eq:noise-ratio-bound}
    \widehat R_l(\varepsilon_*)
    =
    \frac{\widehat{\sigma}_l}
         {\widehat{\sigma}_{l+1}\vee\varepsilon_*}
    \le
    \frac{\widehat{\sigma}_{k+1}}{\varepsilon_*}
    \le
    \frac{\eta_n}{c\eta_n}
    =
    \frac1c.
\end{equation}
On the other hand, because
$\eta_n<c\sigma_k/(1+c)$,
\[
    \frac{\sigma_k-\eta_n}{\eta_n}
    >
    \frac1c.
\]
Together with \eqref{eq:true-regularized-ratio-lower} and
\eqref{eq:noise-ratio-bound}, this yields
\begin{equation}
\label{eq:true-ratio-dominates-noise}
    \widehat R_k(\varepsilon_*)
    >
    \max_{k<l\le L-1}\widehat R_l(\varepsilon_*).
\end{equation}

Together, \eqref{eq:true-ratio-dominates-signal} and
\eqref{eq:true-ratio-dominates-noise} show that the unique maximizer is
$l=k$. Hence $\widehat k=k$ on $\mathcal E_n$. Since
$\mathbb P(\mathcal E_n)\ge1-\delta$, the result follows.
\end{proof}

\subsection{Proof of Proposition \ref{prop:imaging-stability}}

\begin{proof}
Note that $C$ and $\hat C$ are both Hermitian and that $\sigma_k:=\sigma_k(C)>0$. We write the singular value decomposition of $C$ and $\hat C$ as:
    \begin{equation}
    \label{eqn: block matrix}
        C = \begin{bmatrix}
            U_1 & U_2
        \end{bmatrix}
        \begin{bmatrix}
            \Sigma_1 & 0\\
            0 & 0
        \end{bmatrix}
        \begin{bmatrix}
            U_1^* \\
            U_2^*
        \end{bmatrix}, \quad
        \hat C = \begin{bmatrix}
            \hat U_1 & \hat U_2
        \end{bmatrix}
        \begin{bmatrix}
            \hat \Sigma_1 & 0\\
            0 & \hat \Sigma_2
        \end{bmatrix}
        \begin{bmatrix}
            \hat U_1^* \\
            \hat U_2^*
        \end{bmatrix},
    \end{equation}
where $\Sigma_1$ is an invertible $k\times k$ diagonal matrix.

\medskip
\textbf{Step 1 (Reduction to a projector difference).}
For any $\mu$, the reverse triangle inequality gives
\[
  \bigl|\hat J(\mu)-J(\mu)\bigr|
  = \bigl|\,\|\mathcal P_{\hat U_2}\mathbf a_L(\mu)\|_2
          -\|\mathcal P_{U_2}\mathbf a_L(\mu)\|_2\,\bigr|
  \le \|(\mathcal P_{\hat U_2}-\mathcal P_{U_2})\mathbf a_L(\mu)\|_2
  \le \|\mathcal P_{\hat U_2}-\mathcal P_{U_2}\|_2\,\|\mathbf a_L(\mu)\|_2.
\]
Since
$\|\mathbf a_L(\mu)\|_2=\sqrt L$ for every $\mu$,
\begin{equation}
\label{eq:pf-step1}
  \sup_{\mu}\bigl|\hat J(\mu)-J(\mu)\bigr|
  \le \sqrt L\,\|\mathcal P_{\hat U_2}-\mathcal P_{U_2}\|_2 .
\end{equation}

\medskip
\textbf{Step 2 (Splitting the projector difference).}
Using the resolution of identity $I=\mathcal P_{U_1}+\mathcal P_{U_2}$ and
$I=\mathcal P_{\hat U_1}+\mathcal P_{\hat U_2}$,
\[
  \mathcal P_{\hat U_2}-\mathcal P_{U_2}
  = (\mathcal P_{U_1}+\mathcal P_{U_2})\mathcal P_{\hat U_2}-\mathcal P_{U_2}
  = \mathcal P_{U_1}\mathcal P_{\hat U_2}
    - \mathcal P_{U_2}(I-\mathcal P_{\hat U_2})
  = \mathcal P_{U_1}\mathcal P_{\hat U_2}
    - \mathcal P_{U_2}\mathcal P_{\hat U_1}.
\]
By the fact that multiplying by a matrix with
orthonormal columns (from the left) or orthonormal rows (from the right) leaves
the spectral norm unchanged,
\[
  \|\mathcal P_{U_1}\mathcal P_{\hat U_2}\|_2
   = \|U_1U_1^*\hat U_2\hat U_2^*\|_2 = \|U_1^*\hat U_2\|_2,
  \qquad
  \|\mathcal P_{U_2}\mathcal P_{\hat U_1}\|_2
   = \|U_2^*\hat U_1\|_2,
\]
so that
\begin{equation}
\label{eq:pf-step2}
  \|\mathcal P_{\hat U_2}-\mathcal P_{U_2}\|_2
  \le \|U_1^*\hat U_2\|_2 + \|U_2^*\hat U_1\|_2 .
\end{equation}

\textbf{Step 3 (Estimate of $\|U_1^*\hat U_2\|_2$ and $\|U_2^*\hat U_1\|_2)$.}
Recall that $\hat C\hat U_1=\hat U_1\hat\Sigma_1$ and $\hat C\hat U_2=\hat U_2\hat\Sigma_2$. 

Since $\mathrm{rank}(C)=k$, $U_2^*C=0$.  Therefore, 
$$
U_2^*\hat U_1 \hat\Sigma_1 = U_2^*\hat C\hat U_1= U_2^*(C+\hat E)\hat U_1= U_2^*\hat E\hat U_1. 
$$ 
When $\hat\sigma_k>0$,
$\hat\Sigma_1$ is invertible, so $U_2^*\hat U_1 = U_2^*\hat E\hat U_1\hat\Sigma_1^{-1}$. 
Using $\|U_2^*\|_2=\|\hat U_1\|_2=1$ and $\|\hat\Sigma_1^{-1}\|_2=\hat\sigma_k^{-1}$, we obtain
\begin{equation}
\label{eq:pf-bound1}
  \|U_2^*\hat U_1\|_2 \le \frac{\|\hat E\|_2}{\hat\sigma_k}.
\end{equation}

On the other hand, since $U_1^*C=\Sigma_1U_1^*$,
\[
  \Sigma_1 U_1^*\hat U_2 + U_1^*\hat E\hat U_2 = U_1^*( C+\hat E) \hat U_2 = U_1^* \hat C \hat U_2 = U_1^*\hat U_2\hat\Sigma_2
  \quad\Longrightarrow\quad
  \Sigma_1 U_1^*\hat U_2 = U_1^*\hat U_2\hat\Sigma_2 - U_1^*\hat E\hat U_2 .
\]
Fix a unit vector $v\in\mathbb C^{L-k}$. Because the smallest singular value of
$\Sigma_1$ is $\sigma_k$ and the largest singular value of $\hat\Sigma_2$ is
$\hat\sigma_{k+1}$,
\[
  \sigma_k\|U_1^*\hat U_2 v\|_2
  \le \|\Sigma_1 U_1^*\hat U_2 v\|_2
  \le \|U_1^*\hat U_2\|_2\,\hat\sigma_{k+1} + \|\hat E\|_2 .
\]
Taking the supremum over $\|v\|_2=1$ gives
$\sigma_k\|U_1^*\hat U_2\|_2 \le \hat\sigma_{k+1}\|U_1^*\hat U_2\|_2 + \|\hat E\|_2$,
i.e. whenever $\sigma_k>\hat\sigma_{k+1}$,
\begin{equation}
\label{eq:pf-bound2}
  \|U_1^*\hat U_2\|_2 \le \frac{\|\hat E\|_2}{\sigma_k - \hat\sigma_{k+1}}.
\end{equation}

\medskip
\textbf{Step 4 (Weyl's inequality and assembly).}
As $\hat C = C+\hat E$ with both matrices Hermitian, Weyl's inequality yields
$|\hat\sigma_i-\sigma_i|\le\|\hat E\|_2$ for all $i$. Using $\sigma_{k+1}=0$,
\[
  \hat\sigma_{k+1}\le\sigma_{k+1}+\|\hat E\|_2 = \|\hat E\|_2,
  \qquad
  \hat\sigma_k\ge\sigma_k-\|\hat E\|_2 .
\]
Assume now $\|\hat E\|_2<\sigma_k$. Then $\hat\sigma_k\ge\sigma_k-\|\hat E\|_2>0$
(so $\hat\Sigma_1$ is indeed invertible and \eqref{eq:pf-bound1} applies), and
$\sigma_k-\hat\sigma_{k+1}\ge\sigma_k-\|\hat E\|_2>0$
(so \eqref{eq:pf-bound2} applies). Both denominators are bounded below by
$\sigma_k-\|\hat E\|_2$, hence combining
\eqref{eq:pf-step1}, \eqref{eq:pf-step2}, \eqref{eq:pf-bound1} and
\eqref{eq:pf-bound2},
\[
  \sup_\mu\bigl|\hat J(\mu)-J(\mu)\bigr|
  \le \sqrt L\,\|\mathcal P_{\hat U_2}-\mathcal P_{U_2}\|_2 \le \;\sqrt L\left(\frac{\|\hat E\|_2}{\sigma_k-\hat\sigma_{k+1}}
                 + \frac{\|\hat E\|_2}{\hat\sigma_k}\right)
  \le \frac{2\sqrt L\,\|\hat E\|_2}{\sigma_k - \|\hat E\|_2},
\]
which is \eqref{eq:deterministic-imaging}. 
\end{proof}

\subsection{Proof of Proposition \ref{prop:local coercivity condition}}

\begin{proof}
Recall that $\mathcal A = [\mathbf a_L(\mu_1),\cdots,\mathbf a_L(\mu_k)] \in \C^{L\times k}$, $U_1 = \spn\{\mathbf a_L(\mu_1),\cdots,\mathbf a_L(\mu_k)\}$ is the column space of $\mathcal A$. Let $\mathcal{P}_{U_2} = I_L - \mathcal{P}_{U_1}$ denote the orthogonal projection onto $U_2$. Then $J(\mu) = \norm{\mathcal{P}_{U_2} \mathbf a_L(\mu)}_2$. 

\textbf{Step 1:} We localize $J$ near a true mean. Let $\mu^* = \mu_j$ and write $\mu = \mu_j + \delta$ with $\delta \in \R^d$. Entrywise, $[\mathbf a_L(\mu_j+\delta)]_l = e^{\iota\innerproduct{\mu_j}{t_l}} e^{\iota\innerproduct{\delta}{t_l}} = e^{\iota\innerproduct{\mu_j}{t_l}}(1 + \iota\innerproduct{\delta}{t_l} + O(\norm{\delta}^2))$, with the remainder uniform in $l$ since $|\innerproduct{\delta}{t_l}| \le \norm{\delta}_2 \norm{t_l}_2$ is bounded. Collecting entries,
    \[
        \mathbf a_L(\mu_j + \delta) = \mathbf a_L(\mu_j) + D_j \delta + O(\norm{\delta}^2), \quad (D_j)_{l,:} = \iota\, e^{\iota\innerproduct{\mu_j}{t_l}} t_l^\mathrm{T},
    \]
where $D_j:=D \mathbf a_L(\mu_j) \in \C^{L\times d}$ is the Jacobi matrix of $\mathbf a_L$ at $\mu_j$. Since $\mathbf a_L(\mu_j) \in U_1$, we have $\mathcal{P}_{U_2}\mathbf a_L(\mu_j) = 0$, hence
    \[
        J(\mu_j + \delta) = \norm{\mathcal{P}_{U_2} D_j \delta}_2 + O(\norm{\delta}^2) = \norm{T_j \delta}_2 + O(\norm{\delta}^2),
    \]
where $T_j := \mathcal{P}_{U_2} D_j : \R^d \to U_2$. If $T_j$ is injective, then $\sigma_{\min}(T_j) > 0$, and for $\norm{\delta}_2$ small enough that the remainder is dominated, $J(\mu_j + \delta) \ge \frac{1}{2}\sigma_{\min}(T_j)\norm{\delta}_2$. It therefore suffices to prove that each $T_j$ is injective; the constant $\kappa = \frac{1}{2}\min_{1\le j\le k}\sigma_{\min}(T_j)$ and a sufficiently small neighborhood $\mathcal{N}$ then yield the claim.

\textbf{Step 2:} We reduce the injectivity of $T_j$ to a rank condition. For $\delta \in \R^d$,
    \[
        T_j \delta = 0 \iff \mathcal{P}_{U_2}(D_j \delta) = 0 \iff D_j \delta \in U_1 = \mathrm{col}(\mathcal A) \iff \exists\, \alpha \in \C^k:\ D_j \delta = \mathcal A \alpha.
    \]
Thus $T_j \delta = 0$ has only trivial solution $\delta = 0$ whenever the $L \times N$ matrix $[D_j \mid \mathcal A]$ has full column rank:
    \[
        \mathrm{rank}\,[D_j \mid \mathcal A] = N. \numberthis \label{eqn: rank condition}
    \]
In particular, \eqref{eqn: rank condition} implies that $T_j$ is injective.

\textbf{Step 3:} We note that the columns of $[D_j \mid \mathcal A]$ are the evaluations, at $t_1,\dots,t_L$, of the $N$ functions
    \[
        g_m(t) = e^{\iota\innerproduct{\mu_m}{t}}\ (m=1,\dots,k), \quad h_p(t) = \iota\, t^{(p)} e^{\iota\innerproduct{\mu_j}{t}}\ (p=1,\dots,d),
    \]
where $t^{(p)}$ is the $p$-th coordinate of $t$.  In view of this fact, we need only to prove $T_j$ is injective for the case $L=N$. 

Each of the $N$ functions is of the form $q(t)e^{\iota\innerproduct{\nu}{t}}$, with $q\equiv 1$ and frequency $\mu_m$ for the $g_m$, and $q(t) = \iota t^{(p)}$ and frequency $\mu_j$ for the $h_p$. Since exponential polynomials with distinct frequencies are linearly independent, any identity
    \[
        \sum_{m\neq j} \alpha_m e^{\iota\innerproduct{\mu_m}{t}} + \left(\alpha_j + \iota\innerproduct{\beta}{t}\right) e^{\iota\innerproduct{\mu_j}{t}} \equiv 0
    \]
forces each modal coefficient to vanish: $\alpha_m = 0$ for $m\neq j$, and $\alpha_j + \iota\innerproduct{\beta}{t} \equiv 0$, whence $\beta = 0$ and $\alpha_j = 0$. Therefore $g_1,\dots,g_k,h_1,\dots,h_d$ are linearly independent as functions. This shows that the rank condition \eqref{eqn: rank condition} holds on the continuous level, i.e., when the sampling points $t_j$'s form a continuum. 

\textbf{Step 4:} In what follows, we focus on the case $L=N$. We show that \eqref{eqn: rank condition} holds at the cost of a measure-zero exceptional set. 
The key is to analyze the following square matrix
    \[
        \Psi(t_1,\dots,t_N) := [D_j \mid \mathcal A] = \bigl[\varphi_i(t_l)\bigr]_{1\le l\le N,\,1\le i\le N} 
        \in \C^{N\times N},
    \]
 where $\varphi_1,\dots,\varphi_N$ denote the functions $g_1,\dots,g_k,h_1,\dots,h_d$ from Step~3. 
Consider the determinant
    \[
        W(s_1,\dots,s_N) := \det\bigl[\varphi_i(s_l)\bigr]_{1\le i,l\le N},
    \]
regarded as a function of $(s_1,\dots,s_N) \in (\R^d)^N$. Since each $\varphi_i$ is entire, 
$W$ is real-analytic on $(\R^d)^N \cong \R^{dN}$.

\textbf{Step 5:}  We show that $W \not\equiv 0$, and argue by induction on $N$. For $N=1$, $\varphi_1 \not\equiv 0$ 
gives some $s_1$ with $\varphi_1(s_1) \neq 0$. For the inductive step, choose 
$s_1,\dots,s_{N-1}$ with $M := \det[\varphi_i(s_l)]_{1\le i,l\le N-1} \neq 0$, and expand 
$W(s_1,\dots,s_{N-1},s)$ along its last column:
    \[
        W(s_1,\dots,s_{N-1},s) = \sum_{i=1}^N (-1)^{i+N} C_i\, \varphi_i(s),
    \]
where the cofactors $C_i$ are independent of $s$ and the coefficient of $\varphi_N$ is 
$C_N = M \neq 0$. This is a nontrivial linear combination of $\varphi_1,\dots,\varphi_N$, so 
by their linear independence (Step~3) it is not identically zero; hence some $s_N$ gives 
$W(s_1,\dots,s_N) \neq 0$, proving $W \not\equiv 0$.

\textbf{Step 6:} Because $\Psi$ is square, $\mathrm{rank}\,\Psi < N$ if and only if $\det \Psi = 0$, so the 
exceptional set is
    \[
        Z_j := \bigl\{(t_1,\dots,t_N) : \mathrm{rank}\,[D_j \mid \mathcal A] < N\bigr\} 
        = \bigl\{(t_1,\dots,t_N) : W(t_1,\dots,t_N) = 0\bigr\} \subset \R^{dN}.
    \]
As the zero set of a real-analytic function that is not identically zero, $Z_j$ is closed, 
has Lebesgue measure zero, and has empty interior. Consequently, 
condition \eqref{eqn: rank condition} holds for every $(t_1,\dots,t_N) \notin Z_j$.

\textbf{Step 7:} Let $Z=\cup_{1\leq j \leq k} Z_j$. Then $Z$ is also closed, of Lebesgue measure zero, and with 
empty interior. For $(t_1,\dots,t_N) \notin Z$, 
every $T_j$ is injective, and $\kappa = \frac{1}{2}\min_{1\le j\le k}\sigma_{\min}(T_j) > 0$. By Step 1, for each $\mu^* = \mu_j$ there is a neighborhood $\mathcal{N}$ on which $J(\mu) \ge \kappa\norm{\mu - \mu^*}_2$. This finishes the proof.
\end{proof}

\subsection{Proof of Theorem \ref{cor:param_error}}

\begin{proof}
Fix a true component mean $\mu^*\in\{\mu_1,\dots,\mu_k\}$. By
Proposition~\ref{prop:local coercivity condition}, there exist a constant
$\kappa>0$ and a neighborhood $\mathcal N$ of $\mu^*$ such that
\begin{equation}\label{eq:coercivity-ball}
    J(\mu)\ge\kappa\|\mu-\mu^*\|_2,
    \qquad \mu\in\mathcal N.
\end{equation}
Choose $r_0>0$ such that
$\overline{B}(\mu^*,r_0)\subset\mathcal N$, and write
\[
    \varepsilon:=
    \sup_{\mu\in\mathbb R^d}\bigl|\hat J(\mu)-J(\mu)\bigr|.
\]

\medskip
\noindent\textbf{Step 1 (Deterministic error bound).}
On the event $\{\|\hat E\|_2<\sigma_k(C)\}$,
Proposition~\ref{prop:imaging-stability} gives
\begin{equation}\label{eq:fluct}
    \varepsilon
    \le
    \frac{2\sqrt L\,\|\hat E\|_2}
         {\sigma_k(C)-\|\hat E\|_2}.
\end{equation}
Since $\hat J$ is continuous and $\overline{B}(\mu^*,r_0)$ is compact,
$\hat J$ attains its minimum over this closed ball. Let
\[
    \hat\mu
    \in
    \operatorname*{argmin}_{\mu\in\overline{B}(\mu^*,r_0)}
    \hat J(\mu).
\]
Because $\mu^*\in\overline{B}(\mu^*,r_0)$ and $J(\mu^*)=0$,
\begin{equation}\label{eq:min-chain}
    \hat J(\hat\mu)
    \le
    \hat J(\mu^*)
    \le
    J(\mu^*)+\varepsilon
    =
    \varepsilon.
\end{equation}
Combining $|\hat J(\hat\mu)-J(\hat\mu)|\le\varepsilon$ with
\eqref{eq:min-chain} yields $J(\hat\mu)\le2\varepsilon$. Since
$\hat\mu\in\overline{B}(\mu^*,r_0)\subset\mathcal N$, the coercivity estimate
\eqref{eq:coercivity-ball} applies:
\begin{equation}\label{eq:core-bound}
    \kappa\|\hat\mu-\mu^*\|_2
    \le
    J(\hat\mu)
    \le
    2\varepsilon,
    \qquad\text{and hence}\qquad
    \|\hat\mu-\mu^*\|_2
    \le
    \frac{2\varepsilon}{\kappa}.
\end{equation}
Substituting \eqref{eq:fluct} into \eqref{eq:core-bound} gives
\begin{equation}\label{eq:E-bound}
    \|\hat\mu-\mu^*\|_2
    \le
    \frac{4\sqrt L\,\|\hat E\|_2}
         {\kappa\bigl(\sigma_k(C)-\|\hat E\|_2\bigr)}.
\end{equation}

\medskip
\noindent\textbf{Step 2 (Local minimality).}
Let $\tau_0$ be as in \eqref{eq:param-tau0}, and consider the event
\[
    \mathcal E_0:=\{\|\hat E\|_2\le\tau_0\}.
\]
On $\mathcal E_0$, the first term in \eqref{eq:param-tau0} gives
\[
    \sigma_k(C)-\|\hat E\|_2
    \ge
    \frac{\sigma_k(C)}{2}.
\]
Consequently, \eqref{eq:E-bound} becomes
\begin{equation}\label{eq:core-rate}
    \|\hat\mu-\mu^*\|_2
    \le
    \frac{8\sqrt L\,\|\hat E\|_2}
         {\kappa\sigma_k(C)}.
\end{equation}
By the second term in \eqref{eq:param-tau0},
\[
    \frac{8\sqrt L\,\|\hat E\|_2}{\kappa\sigma_k(C)}
    \le
    \frac{8\sqrt L\,\tau_0}{\kappa\sigma_k(C)}
    \le
    \frac{r_0}{2}
    <
    r_0.
\]
Thus, $\hat\mu$ lies in the interior of
$\overline{B}(\mu^*,r_0)$. Since it minimizes $\hat J$ over this ball,
$\hat\mu$ is a genuine local minimizer of $\hat J$ on $\mathbb R^d$.

\medskip
\noindent\textbf{Step 3 (High-probability control).}
Let $Q:=L(M+1)$. By Proposition~\ref{prop:covariance-perturbation}, with probability at least
$1-\delta$,
\begin{equation}\label{eq:thm33-concentration}
    \|\hat E\|_2
    \le
    4L\sqrt{\kappa_\Sigma}
    \sqrt{\frac{\ln(4Q/\delta)}{n}}
    +
    4L\kappa_\Sigma
    \frac{\ln(4Q/\delta)}{n}.
\end{equation}
It is sufficient to require each term on the right-hand side to be at most
$\tau_0/2$. This is guaranteed whenever
\[
    n
    \ge
    \max\left\{
        \frac{64L^2\kappa_\Sigma\ln(4Q/\delta)}{\tau_0^2},
        \frac{8L\kappa_\Sigma\ln(4Q/\delta)}{\tau_0}
    \right\}
    =
    n_0,
\]
where $n_0$ is defined in \eqref{eq:n0-explicit}. Therefore, for every
$n\ge n_0$, the event $\mathcal E_0$ holds with probability at least
$1-\delta$.

We also note that $\sigma_k(C)\le\|C\|_2\le L$. Indeed, every latent Fourier
vector satisfies $\|y_m\|_2^2\le L$, and hence
\[
    \|C\|_2
    \le
    \frac{1}{M+1}\sum_{m=0}^M\|y_m\|_2^2
    \le
    L.
\]
It follows that $\tau_0\le L/2$, and therefore
\[
    \frac{
        64L^2\kappa_\Sigma\ln(4Q/\delta)/\tau_0^2
    }{
        8L\kappa_\Sigma\ln(4Q/\delta)/\tau_0
    }
    =
    \frac{8L}{\tau_0}
    \ge
    16.
\]
Thus, the first term in the maximum defining $n_0$ dominates.

\medskip
\noindent\textbf{Step 4 (Parametric convergence rate).}
For every $n\ge n_0$, combining \eqref{eq:core-rate} and
\eqref{eq:thm33-concentration} gives, with probability at least $1-\delta$,
\begin{align*}
    \|\hat\mu-\mu^*\|_2
    \le{}&
    \frac{32L^{3/2}\sqrt{\kappa_\Sigma}}
         {\kappa\sigma_k(C)}
    \sqrt{\frac{\ln(4Q/\delta)}{n}}
    \\
    &+
    \frac{32L^{3/2}\kappa_\Sigma}
         {\kappa\sigma_k(C)}
    \frac{\ln(4Q/\delta)}{n}.
\end{align*}
When the model parameters and the Fourier measurement design are fixed, the
first term is of order $n^{-1/2}$ and the second is of order $n^{-1}$.
Consequently,
\[
    \|\hat\mu-\mu^*\|_2
    =
    O_p(n^{-1/2}).
\]
This completes the proof.
\end{proof}

\subsection{Proof of Lemma \ref{prop:gradient}}
\begin{proof}
For brevity of notation, we write $\mathbf a(\mu):=\mathbf a_L(\mu)$, whose $l$-th entry is $e^{\iota\innerproduct{\mu}{t_l}}$. Since $\hat U_1 \hat U_1^*$ is the orthogonal projection onto the column space of $\hat U_1$, the matrix $I_L - \hat U_1 \hat U_1^*$ is Hermitian, and its $(m,l)$ entry is $\delta_{ml} - r_m r_l^*$, where $\delta_{ml}$ is the Kronecker delta. We compute the objective function as
    \begin{align*}
        f(\mu) 
        &= \norm{\hat U_1 \hat U_1^* \mathbf a(\mu) - \mathbf a(\mu)}_2^2 \\
        &= \left(\hat U_1 \hat U_1^* \mathbf a(\mu) - \mathbf a(\mu)\right)^* \left(\hat U_1 \hat U_1^* \mathbf a(\mu) - \mathbf a(\mu)\right) \\
        &= \mathbf a(\mu)^* (I_L - \hat U_1 \hat U_1^*) \mathbf a(\mu) \\
        &= \sum_{m=1}^L \sum_{l=1}^L e^{-\iota \innerproduct{\mu}{t_m}} (\delta_{ml} - r_m r_l^*) \, e^{\iota \innerproduct{\mu}{t_l}} \\
        &= \sum_{m=1}^L \sum_{l=1}^L (\delta_{ml} - r_m r_l^*) \, e^{\iota \innerproduct{\mu}{t_l - t_m}}.
    \end{align*}
    Taking the gradient with respect to $\mu$, and using $\nabla_\mu e^{\iota\innerproduct{\mu}{t_l - t_m}} = \iota \, e^{\iota\innerproduct{\mu}{t_l - t_m}}(t_l - t_m)$, we obtain
    \begin{align*}
        \nabla_\mu f(\mu) 
        &= \sum_{m=1}^L \sum_{l=1}^L (\delta_{ml} - r_m r_l^*) \, \nabla_\mu e^{\iota \innerproduct{\mu}{t_l - t_m}} \\
        &= \iota \sum_{m=1}^L \sum_{l=1}^L (\delta_{ml} - r_m r_l^*) \, e^{\iota\innerproduct{\mu}{t_l - t_m}} (t_l - t_m).
    \end{align*}
    The diagonal terms $m = l$ vanish since $t_l - t_m = 0$, so
    \[
        \nabla_\mu f(\mu) = -\iota \sum_{m=1}^L \sum_{l=1}^L r_m r_l^* \, e^{\iota\innerproduct{\mu}{t_l - t_m}} (t_l - t_m),
    \]
    which completes the proof.
\end{proof}

\subsection{Proof of Lemma~\ref{lem:local_regularity}}

\begin{proof}
Write $P:=\mathcal P_{U_2}$ and $\hat P:=\mathcal P_{\hat U_2}$ for the
population and empirical orthogonal projectors. 
We denote  
\[
F(\mu)=\mathbf a_L(\mu)^{*}P\,\mathbf a_L(\mu)
\]
and note $F=J^{2}$, $f=\mathbf a_L(\mu)^{*}\hat P\,\mathbf a_L(\mu)=\hat J^{2}$. 
Introduce the \emph{subspace error}
\[
   \eta \;:=\; \bigl\|\hat P-P\bigr\|_2 .
\]

\medskip
\noindent\textbf{Step 1 (smoothness of $F$).}
The feature map $\mathbf a_L$ is $C^\infty$ with derivatives uniformly bounded on
$\bar B(\mu_j,1)$; set
\[
    B\;:=\;\max_{0\le \ell\le 3}\ \sup_{\mu\in \bar B(\mu_j,1)}
        \bigl\|D^{\ell}\mathbf a_L(\mu)\bigr\|\;<\;\infty,
\]
where $D^{\ell}\mathbf a_L(\mu)$ denotes the $\ell$-th order derivative
(Fréchet derivative) of the vector-valued map $\mathbf a_L$ at $\mu$. 
Since $F=\mathbf a_L^{*}P\mathbf a_L$ is a fixed quadratic form in $\mathbf a_L$ with
$\|P\|=1$, the objective $F$ is $C^{3}$ on this ball and its Hessian is
Lipschitz: there is $C_1>0$, depending only on $B$, with
\begin{equation}\label{eq:F-hess-lip}
    \bigl\|\nabla^2 F(\mu)-\nabla^2 F(\mu_j)\bigr\|
    \;\le\; C_1\,\|\mu-\mu_j\|_2 ,
    \qquad \mu\in \bar B(\mu_j,1).
\end{equation}
Let $D_j=D\mathbf a_L(\mu_j)$ and recall that $T_j=PD_j$. Since
$P\mathbf a_L(\mu_j)=0$, Taylor expansion gives, for $h\in\mathbb R^d$,
\[
    P\mathbf a_L(\mu_j+h)=T_jh+O(\|h\|_2^2),
    \qquad
    F(\mu_j+h)=\|T_jh\|_2^2+O(\|h\|_2^3).
\]
Thus, with $H_j:=\nabla^2F(\mu_j)$, where derivatives are taken with respect
to the real parameter $\mu$,
\[
    v^\top H_jv=2\|T_jv\|_2^2,
    \qquad v\in\mathbb R^d.
\]
Equivalently, $H_j=2\operatorname{Re}(T_j^*T_j)$. The singular-value bounds
therefore imply
\begin{equation}\label{eq:population-hessian-bounds}
    m_jI_d\preceq H_j\preceq M_jI_d .
\end{equation}

\medskip
\noindent\textbf{Step 2 (propagation of the subspace error to the Hessian).}
For any $\mu$, subtracting the two quadratic forms gives
\[
    f(\mu)-F(\mu)=\mathbf a_L(\mu)^{*}(\hat P-P)\,\mathbf a_L(\mu).
\]
Differentiating with respect to the real parameter $\mu$ and applying the
product rule, $\nabla^2(f-F)(\mu)$ is a real symmetric matrix whose entries
are finite sums of terms involving $\mathbf a_L$, its first two derivatives,
and the fixed operator $\hat P-P$. Consequently, on $\bar B(\mu_j,1)$,
\begin{equation}\label{eq:hess-perturb}
    \bigl\|\nabla^2 f(\mu)-\nabla^2 F(\mu)\bigr\|
    \;\le\; C_2\,\bigl\|\hat P-P\bigr\|_2 \;=\; C_2\,\eta ,
\end{equation}
with $C_2$ depending only on $B$. In particular $f$ is $C^2$ on
$\bar B(\mu_j,1)$.

By Proposition~\ref{prop:imaging-stability},
on the event $\{\|\hat E\|_2<\sigma_k(C)\}$ the subspace error is controlled by
\begin{equation}\label{eq:eta-bound}
   \eta \;\le\; \frac{c\,\|\hat E\|_2}{\sigma_k(C)-\|\hat E\|_2},
\end{equation}
for an absolute constant $c>0$; and by
Proposition~\ref{prop:covariance-perturbation},
$\|\hat E\|_2=O_p(n^{-1/2})$, whence $\eta=O_p(n^{-1/2})$.

\medskip
\noindent\textbf{Step 3 (Hessian sandwich).}
Set $C_0:=C_1+C_2$. For every $\mu\in B(\mu_j,\rho_{0,j})$ with
$\rho_{0,j}\le 1$, combining \eqref{eq:F-hess-lip} and
\eqref{eq:hess-perturb} through the triangle inequality gives
\[
    \bigl\|\nabla^2 f(\mu)-H_j\bigr\|
    \;\le\;\underbrace{\|\nabla^2 f(\mu)-\nabla^2 F(\mu)\|}_{\le\,C_2\eta}
        +\underbrace{\|\nabla^2 F(\mu)-H_j\|}_{\le\,C_1\|\mu-\mu_j\|_2\le C_1\rho_{0,j}}
    \;\le\; C_0\,(\rho_{0,j}+\eta).
\]
Since $\nabla^2 f(\mu)$ and $H_j$ are real symmetric,
\eqref{eq:population-hessian-bounds} and the preceding operator-norm bound
give
\begin{equation}\label{eq:hess-sandwich}
    \bigl(m_j-C_0(\rho_{0,j}+\eta)\bigr)I_d
    \;\preceq\;\nabla^2 f(\mu)\;\preceq\;
    \bigl(M_j+C_0(\rho_{0,j}+\eta)\bigr)I_d,
    \qquad \mu\in B(\mu_j,\rho_{0,j}) .
\end{equation}

\medskip
\noindent\textbf{Step 4 (choice of $\rho_{0,j}$ and control of the Hessian).}
Choose $\rho_{0,j}\le 1$ and $\eta_0>0$ small enough that
\[
    \overline{B}(\mu_j,\rho_{0,j})\subset\mathcal N_j,
    \qquad
    C_0(\rho_{0,j}+\eta_0)\le\tfrac12 m_j.
\]
By \eqref{eq:eta-bound} and
Proposition~\ref{prop:covariance-perturbation}, there exists
\[
   n_{\mathrm H}
   =n_{\mathrm H}(\delta,\Sigma,\{\mu_1,\dots,\mu_k\},\mathbf a_L)
\]
such that, for all $n\ge n_{\mathrm H}$,
\[
    \mathbb P(\eta\le\eta_0)\ge 1-\frac{\delta}{2}.
\]
On the event $\mathcal E_{\mathrm H}:=\{\eta\le\eta_0\}$,
\eqref{eq:hess-sandwich} yields the lower bound
$\nabla^2 f(\mu)\succeq\tfrac12 m_j\,I_d$ and the upper bound
\[
   \nabla^2 f(\mu)\preceq\bigl(M_j+\tfrac12 m_j\bigr)I_d\preceq 2M_j\,I_d
   \qquad(\text{using }m_j\le M_j),
\]
for all $\mu\in B(\mu_j,\rho_{0,j})$, which is precisely the asserted sandwich. Hence
$f$ is $\tfrac12 m_j$-strongly convex and $2M_j$-smooth on the convex set
$B(\mu_j,\rho_{0,j})$.

\medskip
\noindent\textbf{Step 5 (existence and uniqueness of the minimizer).}
Apply Theorem~\ref{cor:param_error} with confidence level $\delta/2$. After
increasing a deterministic threshold $n_{\mathrm E}$ if necessary, the theorem
provides an event $\mathcal E_{\mathrm E}$ with
\[
    \mathbb P(\mathcal E_{\mathrm E})\ge1-\frac{\delta}{2}
\]
on which $\hat J$, and hence $f=\hat J^2$, has a local minimizer
$\hat\mu_j\in B(\mu_j,\rho_{0,j})$. The same theorem gives
\[
    \|\hat\mu_j-\mu_j\|_2=O_p(n^{-1/2}).
\]
Set $n_0:=\max\{n_{\mathrm H},n_{\mathrm E}\}$. By the union bound,
\[
    \mathbb P(\mathcal E_{\mathrm H}\cap\mathcal E_{\mathrm E})
    \ge 1-\delta .
\]
On this intersection, $\hat\mu_j$ is an interior local minimizer of the
$C^2$ function $f$, and hence $\nabla f(\hat\mu_j)=0$. The strong convexity
established in Step~4 implies that a critical point is the unique minimizer
of $f$ over $B(\mu_j,\rho_{0,j})$. This proves the assertion.
\end{proof}

\subsection{Proof of Lemma~\ref{lem:global_coercivity}}

\begin{proof}
By definition, $J(\mu)=0$ if and only if $\mathbf a_L(\mu)\in U_1$. Suppose some
$\mu\notin\{\mu_1,\dots,\mu_k\}$ satisfied $\mathbf a_L(\mu)\in U_1$. Then
$\mathbf a_L(\mu),\mathbf a_L(\mu_1),\dots,\mathbf a_L(\mu_k)$ would be $k+1$ steering
vectors at distinct parameters that are linearly dependent, contradicting the
nondegeneracy assumption. Hence
\[
  J(\mu)>0\qquad\text{for every }
  \mu\in\Theta\setminus\{\mu_1,\dots,\mu_k\}.
\]

Next, the set
$K:=\Theta\setminus\bigcup_{l=1}^{k}\tilde{\mathcal N}_l$ is closed and bounded, hence compact, and since $\rho_0>0$ we have
$K\cap\{\mu_1,\dots,\mu_k\}=\varnothing$. As $J$ is continuous on the compact
set $K$, it attains its minimum $m:=\min_{\mu\in K}J(\mu)=J(\mu^\ast)$ at some
$\mu^\ast\in K$. Because $\mu^\ast\notin\{\mu_1,\dots,\mu_k\}$,  $m=J(\mu^\ast)>0$, so $J_\star=m/2\in(0,m)$.

Finally, let $\mu\in\Theta$ satisfy $J(\mu)\le J_\star=m/2<m$. If $\mu\in K$,
then $J(\mu)\ge m$, a contradiction; hence $\mu\notin K$, that is,
$\mu\in\bigcup_{l=1}^{k}\tilde{\mathcal N}_l$. This is the asserted inclusion.
\end{proof}

\subsection{Proof of Theorem~\ref{thm:basin}}

\begin{proof}
Throughout, denote 
\[
   \varepsilon:=\|\hat J-J\|_\infty,
   \qquad
   \varepsilon_\star:=\tfrac12\sigma_{\min}\rho_1-\tau>0.
\]
By the construction of the coercivity neighborhoods,
\begin{equation}\label{eqn:quad_lower}
    J(\mu)\ge\tfrac12\sigma_{\min}(T_l)\|\mu-\mu_l\|_2,
    \qquad \mu\in\tilde{\mathcal N}_l,\quad l=1,\dots,k.
\end{equation}

\medskip
\noindent\textbf{Step 0 (sample size $\Rightarrow$ imaging error below the margin).}
Since $J\le\sqrt L$, letting the distance to $\mu_l$ increase to $\rho_0$ in
\eqref{eqn:quad_lower} gives
$\tfrac12\sigma_{\min}\rho_0\le\sqrt L$. Because
$\rho_1\le\rho_0/2$, it follows that
$\varepsilon_\star<\tfrac12\sqrt L$. Define
\[
   Q:=L(M+1),
   \qquad
   \beta_\star:=\frac{\sigma_k(C)\varepsilon_\star}{4\sqrt L},
\]
and
\begin{equation}\label{eqn:basin-n0}
   n_{\rm img}
   :=
   \max\left\{
      \frac{64L^2\kappa_\Sigma}{\beta_\star^2},
      \frac{8L\kappa_\Sigma}{\beta_\star}
   \right\}
   \ln\left(\frac{8Q}{\delta}\right).
\end{equation}
Since $\varepsilon_\star<\tfrac12\sqrt L$, we have
$\beta_\star<\sigma_k(C)/2$. Moreover, $\sigma_k(C)\le\|C\|_2\le L$, so the
first term in the maximum in \eqref{eqn:basin-n0} dominates the second and
\[
   n_{\rm img}
   =
   O\!\left(
      \frac{L^3\kappa_\Sigma\ln(Q/\delta)}
           {\sigma_k(C)^2\varepsilon_\star^2}
   \right).
\]

By Proposition~\ref{prop:imaging-stability}, on the event
$\{\|\hat E\|_2<\sigma_k(C)\}$ the uniform imaging error obeys
\begin{equation}\label{eqn:eps-bound}
   \varepsilon\;=\;\|\hat J-J\|_\infty
   \;\le\;\frac{2\sqrt L\,\|\hat E\|_2}
                    {\sigma_k(C)-\|\hat E\|_2}.
\end{equation}
On the event $\mathcal E_{\mathrm{pert}}:=\{\|\hat E\|_2\le\beta_\star\}$, the definition of
$\beta_\star$ gives
$\sigma_k(C)-\|\hat E\|_2\ge\tfrac12\sigma_k(C)$, so
\eqref{eqn:eps-bound} gives
$\varepsilon\le4\sqrt L\,\|\hat E\|_2/\sigma_k(C)\le\varepsilon_\star$. Hence
\begin{equation}\label{eqn:margin-met}
   \mathcal E_{\mathrm{pert}}\ \Longrightarrow\
   \tau+\varepsilon
   \;\le\;\tau+\varepsilon_\star
   \;=\;\tfrac12\,\sigma_{\min}\rho_1 .
\end{equation}
Apply Proposition~\ref{prop:covariance-perturbation} with confidence level
$\delta/2$. With probability at least $1-\delta/2$,
\[
   \|\hat E\|_2
   \le4L\sqrt{\kappa_\Sigma}
          \sqrt{\frac{\ln(8Q/\delta)}{n}}
      +4L\kappa_\Sigma\frac{\ln(8Q/\delta)}{n}.
\]
The two terms on the right are each at most $\beta_\star/2$ under
\eqref{eqn:basin-n0}. Thus, for
$n\ge n_{\rm img}$,
$\mathbb P(\mathcal E_{\mathrm{pert}})\ge1-\delta/2$. By the definition of
$n_{\rm reg}(\delta/2)$, there is an event $\mathcal R$ on which the
conclusion of Lemma~\ref{lem:local_regularity} holds for every component, with
$\mathbb P(\mathcal R)\ge1-\delta/2$ for
$n\ge n_{\rm reg}(\delta/2)$. The union bound gives
\[
   \mathbb P(\mathcal E_{\mathrm{pert}}\cap\mathcal R)\ge1-\delta .
\]
Taking $n_0:=n_{\rm reg}(\delta/2)\vee n_{\rm img}$ gives the sample-size order
stated in the theorem. We work on $\mathcal E_{\mathrm{pert}}\cap\mathcal R$ for the remainder
of the proof.

\medskip
\noindent\textbf{Step 1 (the score threshold forces $\hat J(x_i)$ small).}
By definition of the score, $s(x_i)=L-\hat J(x_i)^2$, so
\eqref{eqn:score_threshold} gives
\[
    \hat J(x_i)\;=\;\sqrt{L-s(x_i)}\;\le\;\tau .
\]
Combining with $\|\hat J-J\|_\infty=\varepsilon$,
\begin{equation}\label{eqn:J-small}
    J(x_i)\;\le\;\hat J(x_i)+\varepsilon\;\le\;\tau+\varepsilon
    \;\overset{\eqref{eqn:margin-met}}{\le}\;\tfrac12\,\sigma_{\min}\,\rho_1 .
\end{equation}

\medskip
\noindent\textbf{Step 2 ($x_i$ falls inside a coercivity neighborhood).}
Since $x_i\in\Theta$ and
$\rho_1\le 2J_\star/\sigma_{\min}$ we have
$\tfrac12\sigma_{\min}\rho_1\le J_\star$, so \eqref{eqn:J-small} gives
$J(x_i)\le J_\star$. Lemma~\ref{lem:global_coercivity} therefore implies
$x_i\in\tilde{\mathcal N}_l$ for some $l$.

\medskip
\noindent\textbf{Step 3 (locating and bounding the distance).}
Fix the index $l$ from Step~2. Applying the local coercivity bound
\eqref{eqn:quad_lower} on $\tilde{\mathcal N}_l$ and then \eqref{eqn:J-small},
\[
    \tfrac12\,\sigma_{\min}(T_l)\,\norm{x_i-\mu_l}_2
    \;\le\; J(x_i)\;\le\;\tau+\varepsilon,
\]
so, using $\sigma_{\min}(T_l)\ge\sigma_{\min}$ and \eqref{eqn:margin-met},
\begin{equation}\label{eqn:dist-bound}
    \norm{x_i-\mu_l}_2
    \;\le\;\frac{2}{\sigma_{\min}(T_l)}\,(\tau+\varepsilon)
    \;\le\;\frac{2}{\sigma_{\min}}\cdot\tfrac12\,\sigma_{\min}\rho_1
    \;=\;\rho_1 .
\end{equation}

\medskip
\noindent\textbf{Step 4 (uniqueness of the nearest mean).}
By \eqref{eqn:dist-bound}, $x_i\in\mathcal B_l$. Because
$\rho_1\le\Delta/3$, these closed classification balls are pairwise disjoint,
so this index is unique. Denote it by $j$. For any $l\ne j$, the triangle
inequality gives
\[
    \norm{x_i-\mu_l}_2
    \;\ge\;\norm{\mu_j-\mu_l}_2-\norm{x_i-\mu_j}_2
    \;\ge\;\Delta-\rho_1 .
\]
Since $\Delta-\rho_1>\rho_1\ge\|x_i-\mu_j\|_2$, the mean $\mu_j$ is the unique
nearest true mean.

\medskip
\noindent\emph{Conclusion.}
Thus $x_i\in\mathcal B_j\subseteq\tilde{\mathcal N}_j$, and by
Lemma~\ref{lem:local_regularity} (whose regularity event we secured in Step~0)
the empirical objective $f=\hat J^2$ is $\tfrac12 m_j$-strongly convex and
$2M_j$-smooth on $\tilde{\mathcal N}_j$. All bounds hold simultaneously for every candidate $x_i\in\Theta$ crossing the score threshold.
\end{proof}

\subsection{Proof of Lemma~\ref{lem:score_hit}}

\begin{proof}
Since $J(\mu^*)=0$ and $\|\hat J-J\|_\infty\le\tau/2$,
$ \hat J(\mu^*)\le\frac{\tau}{2}$. For every $x\in\mathbb R^d$,
\begin{align*}
    |\hat J(x)-\hat J(\mu^*)|
    &\le \|\mathcal P_{\hat U_2}
          (\mathbf a_L(x)-\mathbf a_L(\mu^*))\|_2 \\
    &\le \|\mathbf a_L(x)-\mathbf a_L(\mu^*)\|_2.
\end{align*}
Moreover,
\begin{align*}
    \|\mathbf a_L(x)-\mathbf a_L(\mu^*)\|_2^2
    &=\sum_{l=1}^L
      \left|e^{\iota\langle t_l,x\rangle}
      -e^{\iota\langle t_l,\mu^*\rangle}\right|^2 \\
    &=4\sum_{l=1}^L
      \sin^2\!\left(\frac{\langle t_l,x-\mu^*\rangle}{2}\right) \\
    &\le \sum_{l=1}^L|\langle t_l,x-\mu^*\rangle|^2 
    \le L_a^2\|x-\mu^*\|_2^2.
\end{align*}
Hence $|\hat J(x)-\hat J(\mu^*)|
\le L_a\|x-\mu^*\|_2$.
Thus, if $\|x-\mu^*\|_2\le r_{\rm cert}$, then
\[
    \hat J(x)\le\frac{\tau}{2}+L_a r_{\rm cert}=\tau,
    \qquad
    s(x)=L-\hat J(x)^2\ge L-\tau^2.
\]
The probability that one observation belongs to the target component and lies
in $B(\mu^*,r_{\rm cert})$ is $w^*p_{r_{\rm cert}}$. Hence
\[
    \Pr\!\left(\min_{1\le i\le n}\|x_i-\mu^*\|_2>r_{\rm cert}\right)
    \le (1-w^*p_{r_{\rm cert}})^n
    \le e^{-nw^*p_{r_{\rm cert}}}
    \le\delta,
\]
which proves the claim.
\end{proof}

\subsection{Proof of Theorem~\ref{thm:gd_convergence}}

\begin{proof}
Let
\[
    R_j:=\norm{\mu^{(0)}-\hat\mu_j}_2,
    \qquad
    \mathcal D_j:=\overline B(\hat\mu_j,R_j).
\]
For every $\mu\in\mathcal D_j$, Theorem~\ref{thm:basin}, the triangle
inequality, and \eqref{eqn:gd_interior} give
\[
    \norm{\mu-\mu_j}_2
    \le R_j+\norm{\hat\mu_j-\mu_j}_2
    \le \rho_1+2\norm{\hat\mu_j-\mu_j}_2
    <\rho_0.
\]
Hence $\mathcal D_j\subset\tilde{\mathcal N}_j$. On this convex set,
Lemma~\ref{lem:local_regularity} yields
\[
    \tfrac12m_jI_d\preceq\nabla^2f(\mu)\preceq2M_jI_d .
\]
Since $\hat\mu_j$ is an interior minimizer, $\nabla f(\hat\mu_j)=0$. For
$\mu\in\mathcal D_j$, define
\[
    H_\mu
    :=\int_0^1\nabla^2f\bigl(\hat\mu_j
            +s(\mu-\hat\mu_j)\bigr)\,ds .
\]
The segment lies in $\mathcal D_j$, and therefore
\[
    \tfrac12m_jI_d\preceq H_\mu\preceq2M_jI_d,
    \qquad
    \nabla f(\mu)=H_\mu(\mu-\hat\mu_j).
\]
For $0<\gamma\le1/(2M_j)$, all eigenvalues of
$I_d-\gamma H_\mu$ lie in
$[0,1-\tfrac12\gamma m_j]$. Thus
\[
    \norm{\mu-\gamma\nabla f(\mu)-\hat\mu_j}_2
    \le
    \Bigl(1-\tfrac12\gamma m_j\Bigr)
    \norm{\mu-\hat\mu_j}_2 .
\]
The gradient map consequently sends $\mathcal D_j$ into itself. Since
$\mu^{(0)}\in\mathcal D_j$, induction shows that every iterate lies in
$\mathcal D_j\subset\tilde{\mathcal N}_j$, and iterating the preceding inequality
proves \eqref{eqn:linear_rate}.
\end{proof}

\subsection{Proof of Proposition \ref{prop: pca error}}
\begin{proof}
Set
\[
    A:=\sum_{i=1}^k w_i\mu_i\mu_i^\mathrm T,
    \qquad
    M:=\E\!\left[\frac1nX^\mathrm TX\right]
      =A+\sigma^2I_d,
    \qquad
    \hat M:=\frac1nX^\mathrm TX.
\]
Let $\hat P:=\hat V\hat V^\mathrm T$, and let $P$ be a rank-$k$
orthogonal projector satisfying $PA=A$. Since $\hat P$ is the leading
rank-$k$ spectral projector of $\hat M$,
\[
    \operatorname{tr}(\hat P\hat M)\ge \operatorname{tr}(P\hat M).
\]
Moreover,
\begin{align}
    \sum_{i=1}^k w_i\norm{\mu_i-\hat P\mu_i}_2^2
    &=\operatorname{tr}\bigl((I_d-\hat P)A\bigr) \nonumber\\
    &=\operatorname{tr}(PM)-\operatorname{tr}(\hat PM) \nonumber\\
    &\le \operatorname{tr}\bigl(P(M-\hat M)\bigr)
       +\operatorname{tr}\bigl(\hat P(\hat M-M)\bigr) \nonumber\\
    &\le 2k\norm{\hat M-M}_2.
    \label{eqn:pca-variational-bound}
\end{align}

For a generic observation $x$ and every $u\in\R^d$,
\[
    \norm{\innerproduct{u}{x}}_{\psi_2}
    \le C_0(R+\sigma)\norm{u}_2,
    \qquad
    \norm{\innerproduct{u}{x}}_{L_2}
    =(u^\mathrm TMu)^{1/2}
    \ge \sigma\norm{u}_2.
\]
Here
\[
    \norm{Y}_{\psi_2}
    :=\inf\left\{s>0:\E\exp\!\left(\frac{Y^2}{s^2}\right)\le2\right\},
    \qquad
    \norm{Y}_{L_2}:=(\E Y^2)^{1/2}.
\]
Consequently,
\[
    \norm{\innerproduct{u}{x}}_{\psi_2}
    \le K\norm{\innerproduct{u}{x}}_{L_2},
    \qquad
    K:=\frac{C_0(R+\sigma)}{\sigma}.
\]
Applying Theorem~4.7.1 and Remark~4.7.3 of
\cite{vershynin2018high}, with probability at least $1-\delta$,
\begin{align*}
    \norm{\hat M-M}_2
    &\le C_1K^2\norm{M}_2
    \left(
        \sqrt{\frac{d+\ln(2/\delta)}{n}}
        +\frac{d+\ln(2/\delta)}{n}
    \right)\\
    &\le C_2\frac{(R+\sigma)^2(R^2+\sigma^2)}{\sigma^2}
    \left(
        \sqrt{\frac{d+\ln(2/\delta)}{n}}
        +\frac{d+\ln(2/\delta)}{n}
    \right),
\end{align*}
where $\norm{M}_2\le R^2+\sigma^2$. Absorbing $\ln 2$ into the
universal constant and combining this estimate with
\eqref{eqn:pca-variational-bound} proves
\eqref{eqn:pca-error-bound}.
\end{proof}

\section{Auxiliary Lemmas}

\begin{lemma} (Theorem 1 in \cite{gautschi1962inverses}).
\label{lemma:van}
Let $x_i \neq x_j$ for $i\neq j$. 
Let 
    \[
        V_k = \begin{bmatrix}
        		1&1&\cdots&1 \\
                    x_1&x_2&\cdots&x_k \\
                    \vdots&\vdots&\ddots&\vdots\\
                    x_1^{k-1}&x_2^{k-1}&\cdots&x_k^{k-1}
        	\end{bmatrix} 
    \]
be a Vandermonde matrix. 
Then
    \begin{equation}
    \label{eqn:2normVan}
        ||V_k^{-1}||_\infty \leq \max_{1\leq j\leq k} \prod_{i=1,i\neq j}^k \frac{1+|x_i|}{|x_i-x_j|}.
    \end{equation}
\end{lemma}

\end{document}